\documentclass{article} 
\usepackage{iclr_arxiv,times}

\usepackage{enumitem}
\usepackage{amsmath,amsfonts,bm}
\usepackage{mathtools}
\usepackage{amsthm}
\usepackage{algorithm}
\usepackage{algpseudocode}
\usepackage{cleveref}       
\usepackage{etoolbox}
\usepackage{mathrsfs}

\newcommand*{\tran}{^{\mkern-1.5mu\mathsf{T}}} 

\newcommand{\compconj}[1]{%
  \overline{#1}
}

\newtheorem{theorem}{Theorem}
\newtheorem{lemma}{Lemma}

\newtheorem{definition}{Definition}
\theoremstyle{remark}
\newtheorem{remark}{Remark}
\theoremstyle{definition}

\newtheorem{assumption}{Assumption}
\crefname{assumption}{assumption}{assumptions}

\crefname{thmenumi}{theorem}{theorems}
\AtBeginEnvironment{theorem}{
    \crefalias{enumi}{thmenumi}
    \setlist[enumerate,1]{
        label={\arabic*.},
        ref={\thetheorem.\arabic*}
    }
}

\crefname{lemenumi}{lemma}{lemmas}
\AtBeginEnvironment{lemma}{
    \crefalias{enumi}{lemenumi}
    \setlist[enumerate,1]{
        label={\arabic*.},
        ref={\thelemma.\arabic*}
    }
}

\crefname{corenumi}{corollary}{corollaries}
\AtBeginEnvironment{corollary}{
    \crefalias{enumi}{corenumi}
    \setlist[enumerate,1]{
        label={\arabic*.},
        ref={\thecorollary.\arabic*}
    }
}

\crefname{assenumi}{assumption}{assumptions}
\AtBeginEnvironment{assumption}{
    \crefalias{enumi}{assenumi}
    \setlist[enumerate,1]{
        label={\theassumption.\arabic*.},
        ref={\theassumption.\arabic*}
    }%
}

\def\1{\bm{1}}

\def\vtheta{{\bm{\theta}}}
\def\va{{\bm{a}}}
\def\vb{{\bm{b}}}
\def\vc{{\bm{c}}}

\def\vf{{\bm{f}}}

\def\vp{{\bm{p}}}
\def\vq{{\bm{q}}}

\def\vu{{\bm{u}}}
\def\vv{{\bm{v}}}
\def\vw{{\bm{w}}}
\def\vx{{\bm{x}}}

\def\vz{{\bm{z}}}

\def\mA{{\bm{A}}}

\def\mC{{\bm{C}}}

\def\mI{{\bm{I}}}

\def\mL{{\bm{L}}}

\def\mU{{\bm{U}}}
\def\mV{{\bm{V}}}
\def\mW{{\bm{W}}}
\def\mX{{\bm{X}}}

\def\mPhi{{\bm{\Phi}}}

\def\mSigma{{\bm{\Sigma}}}

\DeclareMathAlphabet{\mathsfit}{\encodingdefault}{\sfdefault}{m}{sl}
\SetMathAlphabet{\mathsfit}{bold}{\encodingdefault}{\sfdefault}{bx}{n}

\def\gB{{\mathcal{B}}}

\def\gJ{{\mathcal{J}}}

\def\gX{{\mathcal{X}}}

\def\sA{{\mathbb{A}}}
\def\sB{{\mathbb{B}}}
\def\sC{{\mathbb{C}}}
\def\sD{{\mathbb{D}}}

\def\sH{{\mathbb{H}}}

\def\sK{{\mathbb{K}}}

\def\sN{{\mathbb{N}}}

\def\sP{{\mathbb{P}}}

\newcommand{\R}{\mathbb{R}}

\DeclareMathOperator*{\argmax}{arg\,max}

\newcommand\norm[1]{\left\lVert#1\right\rVert}

\newcommand{\ko}[1]{\mathcal{K}^{#1}}
\newcommand{\kot}{\mathcal{K}^t} 
\newcommand{\agen}{\hat{\mathcal{L}}_{NM_N'}} 
\newcommand{\pgen}{\agen}
\newcommand{\ptgen}{\mathcal{L}_{N}}  
\newcommand{\tgen}{\mathcal{L}} 
\newcommand{\dom}[1]{{\mathcal{D}_{#1}}} 
\newcommand{\fss}[1]{{\mathcal{F}_{#1}}} 
\newcommand{\fso}[1]{L^2\left( #1, \mu \right)} 
\newcommand{\sob}{W^{1,2}(\mathcal{X}^\circ, \mu)} 
\newcommand{\sobs}{{W^{1,2}}} 
\newcommand{\id}{\text{Id}} 
\newcommand{\flow}[1]{\varphi_{#1}}

\newcommand{\oswim}[1]{\phi_{#1}} 
\newcommand{\dswim}[1]{\Phi_{#1}} 
\newcommand{\coord}[1]{e_{#1}} 
\newcommand{\projf}[1]{\mathcal{P}_{\fss{#1}}} 
\newcommand{\projd}[1]{\mathcal{P}_{\dom{#1}}} 

\newcommand{\qde}{quasi-dissipativite}
\newcommand{\qdy}{quasi-dissipativity}
\newcommand{\dm}[1]{\mathrm{d}#1} 
\newcommand{\ind}[2]{\mathbf{1}_{\left\{#1\right\}}\left(#2\right)}
\newcommand{\inds}[2]{\mathbf{1}_{#1}\left(#2\right)}
\newcommand{\abs}[1]{\left|#1\right|}
\newcommand{\swimx}{\vx^{(1)}_{0,i}}
\newcommand{\dz}{{d_z}} 
\newcommand{\set}[1]{\left\{#1\right\}}
\newcommand{\ip}[1]{\left\langle#1\right\rangle} 

\newcommand{\uat}{\cite[Theorem 1]{bolager2023sampling}} 
\newcommand{\lp}{\cite[Theorem 3.15]{tk_theorem}} 
\newcommand{\tk}{\cite[Theorem 4.8]{tk_theorem}} 
\newcommand{\closedgen}{\cite[Theorem 1.4]{tk_theorem}}

\newcommand{\assref}[1]{Assumption~\ref{ass:#1}}
\newcommand{\thmref}[1]{Theorem~\ref{thm:#1}}
\newcommand{\lemref}[1]{Lemma~\ref{lemm:#1}}

\newcommand{\theoremtext}{Default text}

\newenvironment{proofsketch}
  {\begin{proof}[Proof sketch]\vspace{-\topsep}}
  {\end{proof}\vspace{-0.3cm}}

\usepackage{url}
\usepackage{amsmath}
\usepackage{thmtools, thm-restate}
\usepackage{amsfonts}
\usepackage{amssymb}
\usepackage{booktabs}
\usepackage[table]{xcolor}
\usepackage{graphicx} 
\usepackage{subcaption}
\usepackage{authblk}
\usepackage{siunitx}
\usepackage{svg}
\usepackage{natbib}
\usepackage{tikz}
\usetikzlibrary{arrows.meta}
\usepackage{xfrac}
\usepackage{wrapfig,lipsum,booktabs}

\title{Data-driven Random-feature \\Approximation of Koopman Generators \\of Nonlinear Continuous-time Systems}

\author{\textbf{Ana~Čukarska~\(^{1,2,\ddagger, *}\),  Maximilian~Gollwitzer~\(^{1,2,\ddagger}\), Viola~Muning~Li~\(^{1,2}\), Vladyslav~Fediukov~\(^{1,2}\), Zahra~Monfared~\(^{3,4}\), Felix~Dietrich~\(^{1,2}\)} 
        \\
        \(^1\)School of Computation, Information and Technology, Technical University of Munich, Germany \\ 
        \(^2\)Munich Data Science Institute\\
        \(^3\)Interdisciplinary Center for Scientific Computing (IWR)\\
        \(^4\)Department of Mathematics and Computer Science, Heidelberg University, Germany\\        
        \(^{\ddagger}\) These authors contributed equally. \\
        \(^*\) Corresponding author: \texttt{ana.cukarska@tum.de}
        }
\iclrarxivcopy 
\begin{document}

\maketitle

\begin{abstract}
Learning continuous-time dynamical systems from data is an active area of research. The Koopman formalism is often used here, because it is a flexible approach that is informative about the system's fundamental properties. The formalism represents dynamical systems through linear operators acting on a suitable function space. A practical difficulty arising in the numerical approximation of Koopman operators is the choice of a ``dictionary'', meaning an effective finite-dimensional representation of the function space. In this paper, we focus on learning dynamical systems in continuous-time, and so we consider numerical approximations of the Koopman generator instead of the Koopman operator.  Our main contribution is a data-driven random-feature method to obtain a parsimonious and effective function representation for Galerkin approximation methods such as gEDMD. The representation we construct is data-informed, which allows us to improve upon data-agnostic methods. We empirically demonstrate the suitability of our method for forecasting tasks and system analysis. The numerical experiments show that our approach is accurate, and faster than iterative optimization. Further, we prove convergence of our approximation of the vector field and the resulting trajectories. In addition to the numerical results, we also make theoretical contributions, including a general convergence criterion for the semigroup generated by the gEDMD approximation and a counterexample showing that gEDMD does not generally preserve exponential boundedness. With our contribution, we advance the frontier of efficiently learning continuous-time dynamics from data, towards interpretable and mathematically robust modelling based on the Koopman framework.



\end{abstract}

\section{Introduction}
Identifying and inferring fundamental properties of complex nonlinear dynamical systems with an unknown analytical expression using only sparse or irregularly sampled observations is an important but difficult problem in scientific machine learning. One viable framework for such analyses is Koopman theory~\citep{koopman1931hamiltonian}, which has received much attention in recent years due to benefiting greatly from advancements in data-driven methods~\citep{mezic2005spectral, budivsic2012applied}. The Koopman framework shifts the perspective from the evolution of states in physical coordinates to the evolution of functions of the state, called observables~\citep{budivsic2012applied}. This enables the study of the underlying system through linear operators acting on the space of observables rather than directly handling the nonlinear dynamics in state space. The central operators of study are the semigroup of Koopman operators and its associated generator, which capture most properties of the semigroup of flow maps and vector field. The operators are linear on the space of observables of the system, which makes them amenable to powerful tools from spectral theory and which can help to reveal global system behaviour, such as stability, attractors and constants of motion~\citep{mezic2004comparison}. For general dynamical systems, this linearisation is only possible when the space of observables is chosen to be sufficiently large, typically infinite-dimensional, thus turning the problem of approximating an unknown finite-dimensional nonlinear map into approximating an infinite-dimensional but linear map. Given a dataset of trajectories, common methods for approximating these operators are the EDMD~\citep{williams2015data} and gEDMD~\citep{klus2020data} algorithms, which learn the evolution of the observables projected onto a finite-dimensional subspace. This subspace is spanned by a chosen set of functions, often called a ``dictionary''.
Rather than targeting a single member of the operator semigroup, it is useful to learn its generator, which represents the instantaneous change rather than a fixed time increment and makes the whole semigroup accessible via its spectrum. The performance of EDMD and gEDMD depends heavily on selecting a dictionary capable of capturing the dynamics of the underlying system, a task that is often difficult without sufficient prior knowledge of the system. A solution to this problem is to construct the dictionary from the available data using data-driven methods, a task which is commonly referred to as dictionary learning. For instance, \citet{li2017extended} proposed training a neural network as a dictionary for EDMD by minimising a suitable loss function. This has later been extended to an analogous algorithm for learning the Koopman generator using gEDMD~\citep{buzhardt2025relationship,gEDMD-DL25}. 
As alternatives to neural networks, researchers have also investigated using random features as dictionary functions~\citep{tabish-2025, sprink2026accurate}, which are computationally efficient by avoiding the iterative optimisation of training a neural network. In this work, \textbf{we employ data-driven random-features to learn parsimonious and effective dictionaries for gEDMD}. The random feature distribution is constructed using the SWIM algorithm~\citep{bolager2023sampling}, an approach in which training data prescribes the choice of random features. This approach allows us to combine the performance of a learned dictionary with the efficiency of random-feature methods to approximate the Koopman generator from continuous-time data. 

\begin{figure}[h]
    \centering
\includegraphics[width=0.95\linewidth]{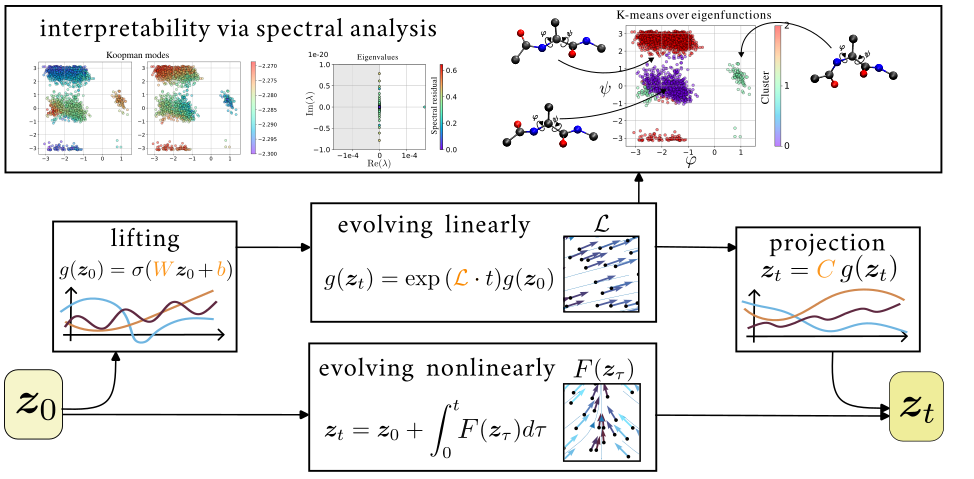}
    \caption{A schematic representation of our method. We construct all terms in orange from the training data.}
    \label{fig:swimnode_viz}
\end{figure}
Chaotic dynamical systems present a particular challenge to data-driven methods, because prediction is only possible for a limited time range. In the paper, we discuss this challenge with respect to dictionary learning for gEDMD, and contrast it to adjoint-based methods such as Neural ODEs.

Our key contributions in this paper are the following.
\begin{itemize}
    \item We propose a method for learning linear representations of continuous-time dynamical systems
    (see \Cref{fig:swimnode_viz}, \Cref{sec:methodology}).
    \item We provide theoretical guarantees for the convergence of our method given sufficient data and dictionary functions (\Cref{thm:vector_field}, \Cref{thm:trajectories}).
    \item We extend the convergence theory of gEDMD (\Cref{thm:trajectories}, \Cref{thm:counterexample}).
    \item We investigate challenges of learning chaotic continuous-time systems from data and contrast it to Neural ODEs (\Cref{thm:param_grad}).
    \item We demonstrate the effectiveness of our approach in forecasting and characterizing the long-term behavior of dynamical systems (\Cref{sec:experiments}).
    \item We provide an implementation of this work publicly accessible via \url{https://gitlab.com/fd-research/dragons}.
\end{itemize}

We include a comprehensive review of related works in~\Cref{sec:related_works}.

\section{Methodology}\label{sec:methodology}

We briefly explain the mathematical background for Koopman operators in \Cref{sec:koopman} and for random feature methods in \Cref{sec:swim}.  We provide a detailed explanation of our method in \Cref{sec:dragons}.

\subsection{Koopman operator \& generator}\label{sec:koopman}

Consider a dynamical system generated by the vector field \(d\vz_t/dt = F(\vz_t)\), with flow \(\flow{t}:\mathcal X\mapsto\mathcal X\) satisfying \(\vz_t=\flow{t}(\vz_0)\), \(\flow{0}=\operatorname{Id}_{\mathcal X}\), and \(\flow{t_1}\circ\flow{t_2}=\flow{t_1+t_2}\). The system is defined by the triplet \(\left(\mathcal X,\mathcal T,\flow{\cdot}\right)\), where \(\mathcal X\subseteq\mathbb R^d\) is a compact state space equipped with a measure \(\mu\), and \(\mathcal T\subseteq\mathbb R\) is the time domain.
Koopman theory allows us to define a linear operator \( \kot :\mathcal{F} \mapsto \mathcal{F} \), $\forall$ \( g \in \mathcal{F} \), \( [\kot g] (\vz_0) = g \circ \varphi_t (\vz_0) \), where \( \mathcal{F} \) is a function(al) space, typically a Hilbert space. In the literature, \(g\) is often called an \textit{observable}. Due to \( \{ \flow{t} \}_{t\geq 0} \) forming a semigroup, \( \{\kot \}_{t\geq 0} \) also forms a semigroup called the \textit{Koopman semigroup}. 
For a strongly continuous semigroup  \( \{\kot \}_{t\geq 0} \) one can define an operator \( \mathcal{L} : \mathcal{D} \mapsto \mathcal{F}\) as \(\mathcal{L} g := \lim_{t\rightarrow0} \sfrac{(\kot g - g)}{t}, \forall g \in \mathcal{D}.\)
\( \mathcal{L} \) is called a \textit{generator} of the Koopman semigroup and is defined on the domain \( \mathcal{D} \) consisting of all functions \(g\) for which the limit is defined.
One can use \( \mathcal{L} \) to express an ODE on the space \( \mathcal{F} \) as \( \dot{g} = \mathcal{L} g,\) where a solution is given with \( g (\vz_t) = \exp{(\mathcal{L}\cdot t)} g(\vz_0)  = \kot g(\vz_0) \). The transformation to a functional representation of the system is illustrated in \Cref{fig:vf_idea}, where the finite-dimensional nonlinear vector field is the usual representation of a process, and using the Koopman framework, this can be transformed to a linear vector field, determined by the Koopman generator  \( \mathcal{L} \). A more detailed introduction to Koopman theory can be found in \Cref{sec:appendix_koopman}.

\begin{figure}[h]
    \centering
    \includegraphics[width=0.98\linewidth]{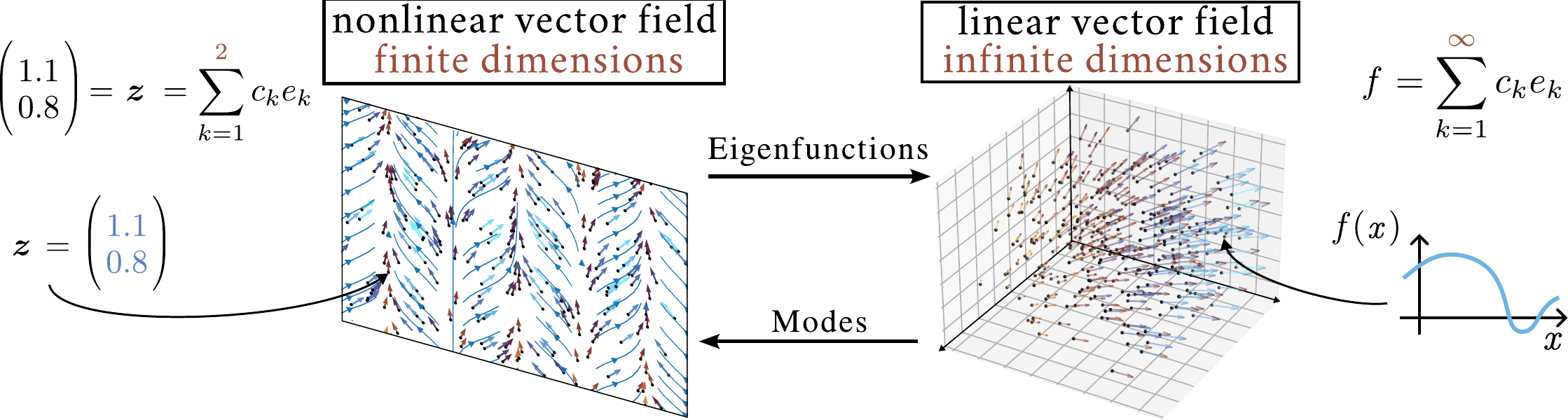}
    \caption{Illustration of the Koopman framework. Focusing on the Koopman generator trades the finite-dimensional nonlinear vector field for an infinite-dimensional linear one.}
    \label{fig:vf_idea}
\end{figure}
\citet{nuske2023finite} derive finite data error bounds for approximating Koopman generators. Recently, \citet{llamazares2024data} extended and generalised existing results for a large class of operators and algorithms, including EDMD and gEDMD as special cases.

\subsection{Random features method: Sampling weights}\label{sec:swim}

Our main algorithm relieas on the construction of a dictionary via the SWIM method introduced by~\cite{bolager2023sampling}, which we discuss in detail in~\Cref{sec:related_works}. To summarize, let \(\sP_{\gX}\) be a probability distribution over the data domain \(\mathcal{X}\times\mathcal{X}\) with the density
\begin{equation}
\label{eq:swim_distribution}
    p_{\gX}(\vx^{(1)},\vx^{(2)})\propto \sfrac{\lVert f(\vx^{(2)}) - f(\vx^{(1)})\rVert}{\lVert \vx^{(2)} - \vx^{(1)}\rVert},
\end{equation}
which is high at steep gradients.
For each neuron in the hidden layer of \(F_{\mathcal{X}}\), sample \((\vx^{(1)}, \vx^{(2)}) \sim \mathbb{P}_{\mathcal{X}}\) and set the weight \(\vw\) and the bias \(b\) of said neuron to
\begin{align}\label{eq:sampled_weights}
    \vw = s_1 \frac{\vx^{(2)} - \vx^{(1)}}{\lVert \vx^{(2)} - \vx^{(1)}\rVert^2}, \quad b = - \langle \vw, \vx^{(1)} \rangle + s_2,
\end{align}
where \(\lVert \cdot \rVert\) and \(\langle \cdot, \cdot \rangle\) denote the Euclidean norm and inner product, and $s_1, s_2 \in \mathbb{R}$ are scalars that are fixed by the choice of activation function~\citep{bolager2023sampling}. 

\subsection{DRAGONS}
\label{sec:dragons}

We now introduce the main approach of this paper: a \textbf{D}ata-driven \textbf{R}andom-feature \textbf{A}pproximation of Koopman \textbf{G}enerators \textbf{o}f \textbf{N}onlinear continuous-time \textbf{S}ystems. We thus refer to our approach as DRAGONS.
The general setting for our approach is as follows. Let \(\mX=\{\vz_{m}, \flow{t_m}(\vz_m) \}_{m=1}^{M}\) denote a dataset sampled from a dynamical system, consisting of \(M \in \mathbb{N}\) snapshots, with \(\vz_{m} \in \mathcal{X} \subset \mathbb{R}^d \). We assume the flow map  \( \{ \flow{t} \}_{t\geq 0} \) is not known, and we also will not attempt to approximate it directly; instead, we will approximate the Koopman semigroup  \( \{\kot \}_{t\geq 0} \) of the system. 

First, a set of observables needs to be chosen. Motivated by the success of data-driven random features, we set the observables to be neurons of a hidden layer with parameters sampled according to SWIM using \Cref{eq:sampled_weights}, with pairs of trajectory snapshots \((\vz^{(1)}, \vz^{(2)}) \sim \sP_{\gX} \). 
Since there are multiple plausible choices for the sampling target $f$ in~\Cref{eq:swim_distribution}, we investigate three different options for $\sP_{\gX}$, which we summarise in~\Cref{tab:dragons_zoo}, while providing a detailed discussion in~\Cref{app:sampling}. We evaluate constructing the dictionary as a basis for regressing against the vector field or flow map, or constructing the neurons based on how much information they reveal about the generator. In our experiments we found that the generator-driven sampling consistently performs the best.
Denote a SWIM hidden layer with \(N \in \mathbb{N}\) neurons as $\dswim{N}: \mathcal{X} \rightarrow \mathbb{R}^N $, with each neuron taking the role of one observable, denoted as $\{ \oswim{i} \}_{i=1}^N$. In the Koopman literature, embedding the finite-dimensional states  \(\vz_{m} \) into the observable space by evaluating $\dswim{N}(z_m)$ is called \textit{lifting}.

 \begin{table}[ht]
\centering
\setlength{\tabcolsep}{20pt}
\caption{DRAGONS Zoo. Different densities $p_\gX(\vz^{(1)},\vz^{(2)})$ for sampling distribution $\sP_{\gX}$.}
\label{tab:dragons_zoo}
\begin{tabular}{ccc}

\textbf{Flow-driven: } & \textbf{Vector field-driven: } & \textbf{Generator-driven: }
\\
$\frac{\norm{ \flow{\tau_2}(\vz^{(2)}) - \flow{\tau_1}(\vz^{(1)})}}{\norm{\vz^{(2)} - \vz^{(1)}}}$
&
$         \frac{\norm{ F(\vz^{(2)}) - F(\vz^{(1)})}}{\norm{\vz^{(2)} - \vz^{(1)}}}$
&
$\frac{\abs{(\vz^{(2)} - \vz^{(1)})^\top (F(\vz^{(2)}) - F(\vz^{(1)}))}}{\norm{\vz^{(2)} - \vz^{(1)}}^3}$
\end{tabular}

\end{table}

Under simple assumptions (see \Cref{lemm:ass}), a SWIM-sampled dictionary is suitable for  approximating a Koopman generator  \(\hat{\tgen}_{NM} \) as the matrix \(\hat{L}_{NM} \) with gEDMD such that the operator it represents converges pointwise, \(\hat{\tgen}_{NM} \overset{\text{a.s.}}{\rightarrow} \mathcal{L} \). Thus, we use gEDMD for the approximation of the generator on the span of SWIM-neuron observables using a standard Galerkin approximation, i.e.,
\begin{equation}
\begin{aligned}
    \label{eq:edmd_methodology}
    [\hat{\mathbf{C}}_{NM}]_{i,j} &= \frac{1}{M} \sum_{m=1}^{M} \compconj{\phi_j(\vz_{m})}  \mathcal{L} \phi_i(\vz_{m}) &[\hat{\mathbf{G}}_{NM}]_{i,j} &= \frac{1}{M} \sum_{m=1}^{M} \compconj{\phi_j(\vz_{m})} \phi_i(\vz_{m}) \\ 
    &
    & \hat{L}_{NM}^T = \hat{\mathbf{C}}_{NM} \hat{\mathbf{G}}_{NM}^{+}. & &
\end{aligned}
\end{equation}

This enables us to evolve observables $\phi_i$ in time: we have \( \dot{\phi}_i(\vz_m) = \mathcal{L} \phi_i(\vz_m)\), which we can integrate to obtain a linear map \(\kot = \exp{(\mathcal{L}t)} \) to compute \( \phi_i(\flow{t}( \vz_m)) = \kot \phi_i(\vz_0)\).

For many applications we are interested in forecasting the state of the dynamical system. Thus, we need to invert the lifting function which maps finite-dimensional states into observable space. In practice, and for a wide range of Koopman dictionaries, a common approach is to approximate the inverse of \(\Phi_N\) with a linear map \(C= \mX \Phi_N(\mX)^{+}\)
; we refer to this mapping as the \textit{projection} step. The matrix $C$ contains the so-called Koopman modes~\citep{budivsic2012applied}.

\begin{minipage}{0.43\textwidth}
    \begin{algorithm}[H]
    \caption{Training DRAGONS.}\label{alg:training_dragons}
    \begin{algorithmic}
      \Procedure{train}{$\mX, \dot{\mX}$}
      \State $\mW, \vb \leftarrow \text{SAMPLE-Layer}(\mX, \mathbb{P}_X)$
      \State $\dswim{N}(\cdot)  \leftarrow  \sigma(\mW \, \cdot + \vb)$
      \State $\mL^T = (\partial_t\dswim{N}(\mX)) (\dswim{N}(\mX))^+$
      \State $\mC = \mX\, \dswim{N}(\mX)^+$
      \State Return $\mC, \mL, \dswim{N}$
      \EndProcedure
  \end{algorithmic}
\end{algorithm}
\end{minipage}
\hfill
\begin{minipage}{0.52\textwidth}
\begin{algorithm}[H]
    \caption{Forecasting with DRAGONS.}\label{alg:forecasting_dragons}
    \begin{algorithmic}
        \Procedure{predict}{$\mX, \dot{\mX},(t_i)_{i=1}^K,\vx_{t_0}$}
        \State $\mC, \mL, \dswim{N} \leftarrow \text{train}(\mX, \dot{\mX})$
        \For{$i=1,2,\dots,K$}
            \State $\vx_{t_i} \leftarrow \mC\exp(\mL^T(t_i-t_{i-1}))\dswim{N}(\vx_{t_{i-1}})$
        \EndFor
        \State Return $(\vx_{t_i})_{i=1}^K$
        \EndProcedure
    \end{algorithmic}
\end{algorithm}
\end{minipage}

Our approach is summarized in \Cref{alg:training_dragons}. It provides a discretized approximation of the Koopman generator, which can be used to predict system states through \Cref{alg:forecasting_dragons}. If they are not given, we approximate the time derivatives of $X$ using a two-point finite difference scheme. In the worst case, the computational time for \Cref{alg:training_dragons} scales cubically with the number of datapoints. This can be reduced drastically, even to linear scaling, by using efficient, iterative linear solvers and by appropriate subsampling of the full dataset.

\section{Theoretical results}
In this section, we present our theoretical results. Here we only state the main theorems and provide proof sketches, while the full proofs can be found in the Appendices \ref{app:convergence_theory} and \ref{sec:adjoing_ev_proof}.
\subsection{Convergence Results}
\label{sec:convergence_main}
A major benefit of connecting our method to Koopman theory is that it facilitates theoretical analysis by making existing results on the convergence of gEDMD and from the general theory of $C_0$-Semigroups available for use. We start by stating the required assumptions, for which we provide practically oriented sufficient conditions through \lemref{c0}.

\renewcommand{\theoremtext}{
For the results presented in \Cref{sec:convergence_main}, we assume the following.
    \begin{itemize}
        \item The measure $\mu$ has an everywhere positive and bounded density $\rho$ w.r.t. the Lebesgue measure $\lambda$ on $\gX$, where $0<\alpha\leq\rho(\vx)\leq\beta<\infty, \forall \vx\in\gX$.
        \item The data is sampled i.i.d. from $\mu$.
        \item The Koopman operators $\kot$ form a $C_0$-semigroup.
    \end{itemize}
}
\begin{restatable}{assumption}{MainAssumption}
\label{ass:main_ass}
\theoremtext
\end{restatable}

The first two points are not very restrictive as they include a large class of measures. Furthermore, as \lemref{c0} shows, the third point is readily attained for Lipschitz vector fields, which is the natural setting to work in due to the Picard-Lindelöf Theorem. For SWIM-sampled networks, we additionally require the activation function to be ReLU to enable approximation of Sobolev functions as required by \citet{llamazares2024data}. Our results naturally extend to other dictionaries that fulfil a Sobolev approximation property analogous to \lemref{swim_sob}, including classically trained neural networks with a broad class of activation functions \citep{HORNIK1990551,HORNIK1991251}. Additionally, we also require $\gX$ to follow Definition \ref{def:input_space} in Appendix \ref{app:theory_ass}, which allows the use of existing universal approximation theory from \citet{bolager2023sampling}. We can now state our first convergence result.
\renewcommand{\theoremtext}{
Let the assumptions be as in \assref{main_ass}. Then $\forall \epsilon>0$, there are constants $N,M_N\in\sN$, a SWIM sampled hidden layer $\dswim{N}$ and outer weights $\vc^i$ such that for any $M_N'\geq M_N$ almost surely
    \begin{equation}
        \norm{\hat{F_i}-F_i}_{\fss{}}=\norm{\sum_{j=1}^Nc^i_j\agen \oswim{j}-F_i}_{\fss{}}<\epsilon,\quad \forall i\in\{1,2,\dots,d\}.
    \end{equation}
}
\begin{restatable}[Convergence of the approximate vector field]{theorem}{VectorField}
\label{thm:vector_field}
\theoremtext
\end{restatable}
\begin{proofsketch}
\citet{llamazares2024data} provide a thorough analysis and sufficient conditions on the dictionary for pointwise convergence of the gEDMD approximations to $\tgen$. We show that our dictionary fulfills their assumptions in \lemref{ass}. The claim then follows from their results.
\end{proofsketch}
As an auxiliary result, we also extend the existing convergence theory for sampled networks to weighted Sobolev spaces through \lemref{swim_sob}.
Furthermore, we extend the existing convergence theory for gEDMD by providing, to the best of the authors' knowledge, the first proof for pointwise convergence of the semigroup generated by the gEDMD approximation under suitable conditions.
\renewcommand{\theoremtext}{
Let the assumptions be as in \assref{main_ass} and assume the true Koopman operators to form a quasi-contractive semigroup, i.e. $\norm{\kot}\leq e^{\omega t}$ for some $\omega_*\in\R$ and all $t\in[0, T]$. Then $\forall  \epsilon>0$ and $t\in[0,T]$, there are constants $N,M_N\in\sN$, a SWIM sampled hidden layer $\dswim{N}$ and outer weights $\vc^i$ such that for any $M_N'\geq M_N$ almost surely
    \begin{equation}
        \norm{(\varphi_t)_i-(\hat{\varphi_t})_i}_{\fss{}}=\norm{
            \exp\left(\tgen t\right)\coord{i} - \sum_{j=1}^N c^i_j \exp\left(\pgen t\right) \oswim{j}
        }_{\fss{}}<\epsilon,\quad \forall i\in\{1,2,\dots,d\}.
    \end{equation}
}
\begin{restatable}[Convergence of the approximate trajectories]{theorem}{Trajectories}
\label{thm:trajectories}
\theoremtext
\end{restatable}
\begin{proofsketch}
     We show the statement for arbitrary observables $g\in\dom{}$. The main idea is to apply the First Trotter-Kato Approximation Theorem \tk, which guarantees the convergence of the approximate semigroups given exponential boundedness and the convergence of the generators. The exponential boundedness is obtained in our setting for \qde\ systems as we show via \lemref{qd_error} that gEDMD preserves this property.
\end{proofsketch}
This result is not specific to our dictionary, as it does not use any specific properties beyond universal approximation and thus transfers to any dictionaries in the general setting of \citet{llamazares2024data}.
We empirically found the standard practice of lifting and projecting in between evaluations, as written in \Cref{alg:forecasting_dragons}, to perform better than a single projection, which is why this is done in the numerical experiments.
Regarding the assumption of \qdy, we have two results worth highlighting. Firstly, \lemref{c0} shows that this is fulfilled for Lipschitz vector fields with $\mu$ being the normalised Lebesgue measure. Secondly, we provide an explicit counterexample (\thmref{counterexample}) showing that without imposing further restrictions, the semigroups generated by gEDMD do not generally inherit the exponential boundedness from the true semigroup, which we formalise in the following theorem.
\renewcommand{\theoremtext}{
There are systems for which the true semigroup $\kot$ is exponentially bounded and the sequence of gEDMD approximants $\agen$ converges pointwise on $\dom{}$ to $\tgen$, while $\nexists W,\omega\in\R$ such that
    \begin{equation}
        \norm{\exp{\left(\agen t\right)}}\leq We^{\omega t},\quad\forall N\in\sN,
    \end{equation}
    i.e. the sequence of approximate semigroups can not be bounded exponentially by uniform constants.
}
\begin{restatable}[gEDMD does not preserve exponential boundedness]{theorem}{Counterexample}
\label{thm:counterexample}
\theoremtext
\end{restatable}
\begin{proofsketch}
    We show that the scalar Lipschitz system $\dot{x}=-x$ with a dictionary in the setting of \lemref{c0} and \citet{llamazares2024data} meeting the stated convergence and boundedness requirements can yield a sequence $\agen$ with eigenvalues lower bounded by $N$ for each $N\in\sN$. The main mechanism is choosing our dictionaries so that one function concentrates on a cusp pathology of the density. The details can be found in \thmref{counterexample} in Appendix \ref{app:theory_conv}.
\end{proofsketch}
Since the enabling mechanism is the pathology of the density, this motivates the restriction to the Lebesgue measure as a sufficient condition for \qdy. A more thorough investigation into when gEDMD preserves exponential boundedness is left for future work.

\subsection{Limitation: Spectral Convergence}
\label{sec:limit_spectral_convergence}
In this section, we briefly highlight what we can and cannot theoretically guarantee about how the spectrum of our approximant relates to that of the true generator. In general, spectral convergence is a challenging problem within the Koopman community~\citep{colbrook-2024a}, that has as of yet not been fully resolved. \citet[Theorem 6.3]{llamazares2024data} show the almost sure convergence of eigenpairs of $\agen$ to some (nontrivial) eigenpairs of $\tgen$ when taking an appropriate subsequence of $(N,M_N)$ pairs, where the convergence of the eigenfunctions is weak. By \Cref{lemm:ass}, this result also applies to our method. However, it is worth highlighting the limitations of this result. Firstly, the result holding only for a subsequence means that increasing the dictionary size and training data can potentially increase the error on the spectral approximation, even when reducing the pointwise error for the generator over $\fss{}$. Secondly, the result explicitly guarantees convergence only to a subset of the true spectrum. This is an inherent limitation of gEDMD, as it cannot resolve continuous spectra, which commonly arise in sufficiently complex dynamical systems~\citep{continuous1,continuous2}.
To account for this limitation, in addition to the computed spectra, we also approximate the associated residual according to \citet{colbrook-2024a} and report it in all of our experiments. Chaotic systems are particularly susceptible to this limitation. Next to posing difficulties for spectral approximation, these systems also pose more general challenges for learning them from data using neural ODEs, which we highlight in the next section. 

\subsection{Related Limitations of Neural ODEs: chaotic systems}
\label{sec:limitations_of_neural_ODE}
Neural ODEs~\citep{chen2018neural} are a popular approach for training a neural network to approximate a vector field from data.
The main idea is to use the adjoint equations to propagate the error through time. We now describe the main issue of this approach when applied to data from a chaotic system, and provide a more detailed discussion in \Cref{sec:adjoing_ev_proof}. Consider training a neural network to parameterise a vector field $\dot{\vz}_\vtheta=\vf(\vz;\vtheta)$ by minimising a suitable scalar loss $\gJ(\vtheta)=\gJ(\texttt{ODESolve}(\vz_0,\vf,T,\vtheta))$, i.e., between the predicted and ground-truth future states for a given set of initial conditions and evaluation time points. Here, $\vtheta$ denotes the trainable parameters and $T$ the prediction horizon. The idea of Neural ODEs is to compute the loss gradient $\nabla_\vtheta \gJ={\dm{\gJ}}/{\dm{\vtheta}}$ by employing the adjoint sensitivity method~\citep{pontryagin2018mathematical, cea1986conception}. In this method, the sensitivity $\va(t)=\dm{\gJ}/\dm{\vz(t)}$ is used to define $\nabla_\vtheta \gJ=-\int_0^Ta(t)^\top{\partial \vf}/{\partial\vtheta}(\vz_\vtheta(t);\vtheta)\dm{t}$. This error is integrated backwards from a predicted state at time $T$ using a numerical solver to obtain the parameter update. We now state a simplified version of our main result.
It establishes that Neural ODEs trained using the adjoint method are prone to exploding and vanishing gradients on certain systems, regardless of the numerical integrator used for training. 
The full statement (\Cref{thm:param_grad_unified}) and proof can be found in \Cref{sec:adjoing_ev_proof}.
\begin{theorem}[Simplified: Exponential growth of the parameter gradient of a neural ODE]\label{thm:param_grad}
Consider the Neural ODE, $\dot{\vz_\vtheta}(t) = \vf(\vz_\vtheta(t), \vtheta)$ trained with the adjoint method using the loss function $\gJ$ defined above, on time series data from a chaotic system over a time horizon $T$. Since the data are chaotic, the largest Lyapunov Exponent (LE) is positive, $\lambda_{\max}>0$, and for sufficiently long $[t,T]$, the largest Finite-Time Lyapunov Exponent (FTLE) is also positive, $\lambda_{\max}^{FT}>0$, with $\lambda_{\max}^{FT} \to\lambda_{\max}$ as $T\to\infty$. Then in the finite-time regime,
$\|\nabla_\theta \mathcal{J}\|
=
\Omega\!\left(
e^{\lambda_{\max}^{FT}T}
\right)$,
and asymptotically as \(T\to\infty\), 
$ \|\nabla_\theta \mathcal{J}\|
=
\Omega\!\left(
e^{(\lambda_{\max}-\epsilon)T}
\right)$ for every \(\epsilon>0\). Here $\Omega(\cdot)$ denotes a lower bound up to a constant factor.

\end{theorem}

\begin{proofsketch}
We first establish analogous bounds on the adjoint state in \Cref{thm:adjoint_state:FTLE,thm:adjoint_state}. The main idea is to express the solution of the adjoint equation in terms of the tangent linear flow via its fundamental matrix.
For the FTLE bounds, the upper bound follows from the definition of $\lambda_{\max}^{FT}$, while the lower bound follows from the singular value decomposition of the fundamental matrix and the projection of the terminal adjoint onto its maximal left singular vector (\Cref{thm:adjoint_state:FTLE}).
For the LE bounds, the upper bound follows from the definition of $\lambda_{\max}$, and the lower bound follows from Oseledec's multiplicative ergodic theorem applied to the dual Oseledec decomposition, where the top dual Oseledec subspace determines the backward growth rate (\Cref{thm:adjoint_state}).
These estimates are then transferred to the loss gradient via norm estimates on the integral expression and the stated boundedness assumption.
\end{proofsketch}

The asymptotic (LE) results show that training on chaotic systems is prone to exploding gradients when sufficiently long trajectories are used.
The finite-time (FTLE) results give exact, computable bounds at the actual training horizon, without requiring ergodicity or asymptotic limits, and thus provide a practical diagnostic for anticipating gradient explosion or vanishing on finite intervals.
We validate this numerically on the example of the Lorenz 63 system in~\Cref{app:node_exploding_adjoint}.

\section{Computational experiments}
\label{sec:experiments}

\subsection{Forecasting}

To assess the predictive performance of DRAGONS, we choose three test problems: a Van der Pol oscillator, the chaotic Lorenz 63 system, and a mean-field model of cylinder flow. Details on the datasets, experimental setup and hyperparameter search are given in \Cref{app:experiments}. We compare our approach to state-of-the-art dictionary learning methods using an Adam optimizer (gEDMD-DL), and against data-agnostic random features (Gaussian) and common dictionary choices such as radial basis functions (RBF-Dict) and polynomials (Poly-Dict). For the chosen systems, a polynomial dictionary is sufficiently expressive and invariant, thus the Poly-Dict sets the standard for a good dictionary for all of these problems. 

The results are given in \Cref{tab:main_results}, where we report the test error, dictionary size and training time. DRAGONS shows excellent performance, either on par with or better than existing methods in terms of error. The fit time is reasonably low when considering the dictionary size, indicating that the computation is dominated by the linear least squares solver. For chaotic systems in particular we achieve a lower error with DRAGONS compared to the Gaussian and the Adam-trained neural network dictionary. This is an noteworthy observation because chaotic systems have continuous spectra, which pose a challenge for gEDMD (see \Cref{sec:limit_spectral_convergence}). DRAGONS is thus a robust choice for forecasting tasks on irregular time intervals.
Further details and insights from these experiments can also be found in \Cref{app:experiments}.

\begin{center}
\begin{table}[h]
\scriptsize
\caption{Results are evaluated on a test set and averaged over 5 random seeds.}
\label{tab:main_results}
\begin{tabular}{p{1.65cm} p{2.1cm}  p{3.9cm} p{0.8cm}  p{3.3cm}  }
    \toprule
    Example & Model & Error: avg (min, max) & \# dict & Time [s]: avg (min, max) \\
    \midrule
    Van der Pol &  DRAGONS & 6.23e-7 (2.84e-7, 8.08e-7) & 130 & 0.79 (0.76, 0.89) \\  
    & Gaussian & 3.18e-6 (1.31e-6, 6.23e-6)  & 130 & 0.69 (0.67, 0.70) \\ 
    & gEDMD-DL & 1.36e-5 (3.73e-6, 2.73e-5) & 130 & 90.61 (55.92, 223.04)  \\ 
     &  RBF-Dict & 1.58e-2 (1.26e-2, 1.81e-2) & 202 & 0.60 (0.56, 0.68) \\ 
     &  Poly-Dict &  2.29e-5 & 9 & 0.37  \\ 
     \midrule
     Lorenz 63 &  DRAGONS & 3.75e-4 (1.87e-4, 5.69e-4) & 131 & 2.72 (2.68, 2.76) \\
     & Gaussian & 3.68e-1 (1.18e-1, 8.36e-1) & 67 & 3.19 (2.01, 5.02)\\
     & gEDMD-DL & 1.48e-1 (2.39e-2, 4.14e-1) & 19 & 11.83 (11.41, 12.14) \\ 
      &  RBF-Dict & 1.58e-1 (9.76e-2, 2.47e-1) & 303 & 2.53 (2.48, 2.58)  \\
      &  Poly-Dict &  3.85e-4  & 34 & 1.87 \\
      \midrule 
     Cylinder &  DRAGONS & 3.34e-7 (3.28e-7, 3.42e-7) & 515 & 5.50 (4.69, 8.06) \\ 
      &  Gaussian & 3.24e-7 (3.13e-7, 3.41e-7) & 131 & 2.35 (1.84, 3.35) \\
      &  gEDMD-DL & 3.24e-7 (3.23e-7, 3.28e-7) & 35 & 14.67 (14.14, 15.50) \\
      &  RBF-Dict &  1.64e-3 (1.61e-3, 1.65e-3)   & 103 & 1.37 (1.15, 1.74) \\
      &  Poly-Dict & 3.13e-7 & 34 & 0.93 \\
  \hline
\end{tabular}
\end{table}
\vspace{-0.3cm}
\end{center}
\begin{figure}[h]
    \centering
    \begin{subfigure}{0.23\textwidth}
        \includegraphics[width=\linewidth]{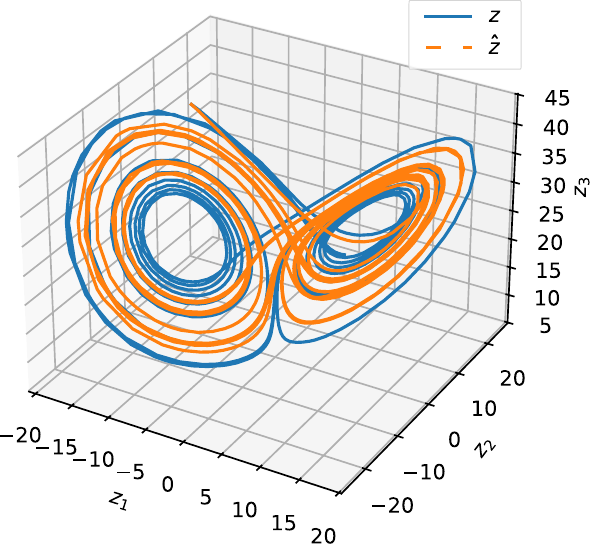}
        \caption{Phase space}
        \label{fig:DRAGONS_L63_phase}
    \end{subfigure}
    \hspace{7pt}
    \begin{subfigure}{0.72\textwidth}
        \includegraphics[width=\linewidth]{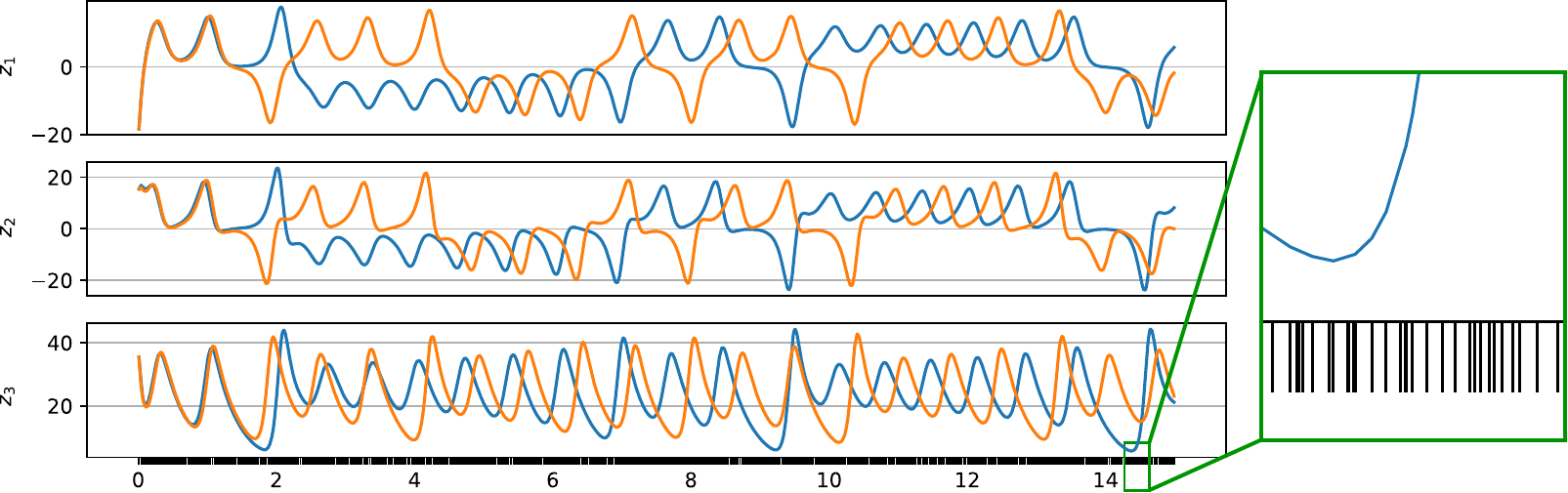}
        \caption{Trajectory}
        \label{fig:DRAGONS_L63_traj}
    \end{subfigure}
    \caption{Comparison of numerically calculated versus predicted trajectories for Lorenz experiment. Ticks on the x-axis show the evaluation points in time, which are irregularly spaced.}
    \label{fig:DRAGONS_L63}
\end{figure}

\subsection{Spectral Analysis}

\paragraph{Koopman Mode Analysis: Flow around a cylinder}
In this experiment, we demonstrate that spectral properties of a system's Koopman generator can be extracted accurately with DRAGONS. Specifically, it allows us to interpret the underlying system dynamics on a global scale for the widely used model of the flow around a cylinder. \Cref{fig:cyl_eigenvalues} shows the eigenvalues of the generator and their corresponding spectral residuals. Nine eigenvalues with large negative real parts are omitted for visualization. The real part of an eigenvalue determines the growth or decay rate of the corresponding mode, while the imaginary part determines its oscillation frequency. The two modes of a complex-conjugate pair typically represent one oscillatory pattern. Representative spatial Koopman modes corresponding to three selected eigenvalues are shown in \Cref{fig:cyl_non_oscillatory_modes} and \Cref{fig:cyl_oscillatory_mode}. 

\begin{figure}[h]
    \centering
    \begin{subfigure}{0.35\textwidth}
        \includegraphics[width=\linewidth]{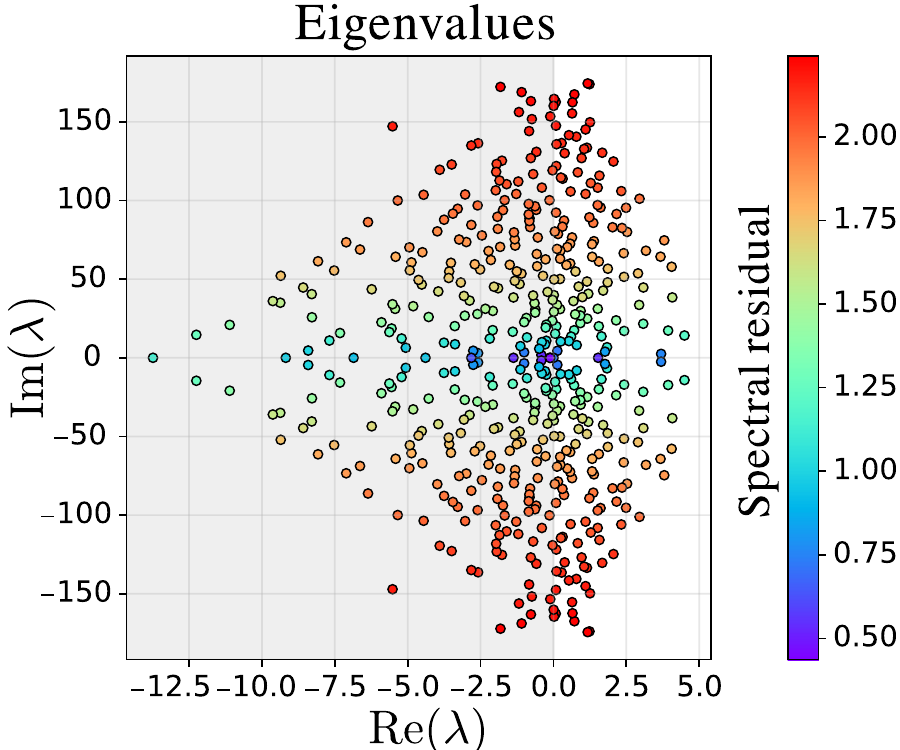}
    \caption{Eigenvalues of \(\hat{L}\).}
        \label{fig:cyl_eigenvalues}
    \end{subfigure}
    \vspace{0pt}
    \begin{subfigure}{0.3\textwidth}
\includegraphics[width=\linewidth]{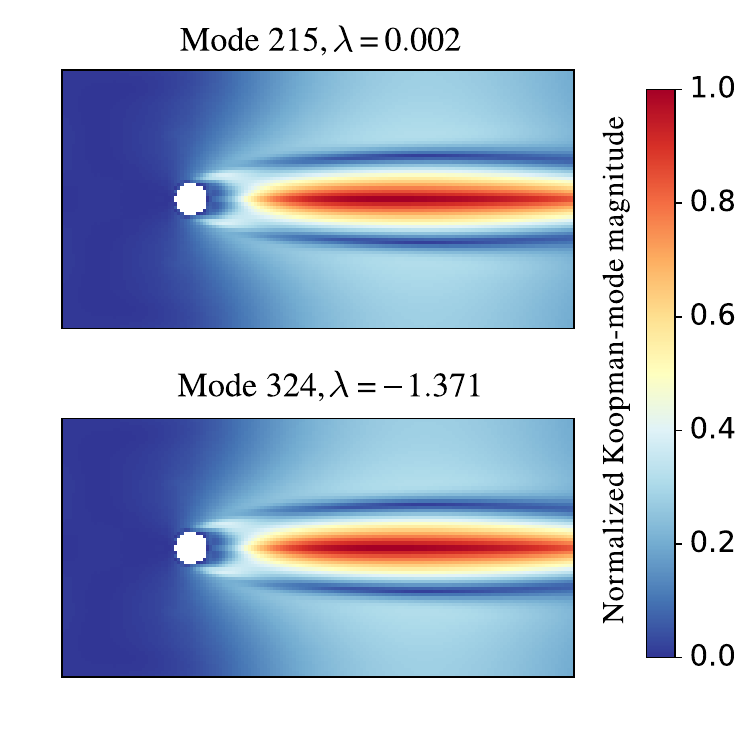}
    \caption{Non-oscillatory modes.}
\label{fig:cyl_non_oscillatory_modes}
    \end{subfigure}
\vspace{0pt}
    \begin{subfigure}{0.3\textwidth}
\includegraphics[width=\linewidth]{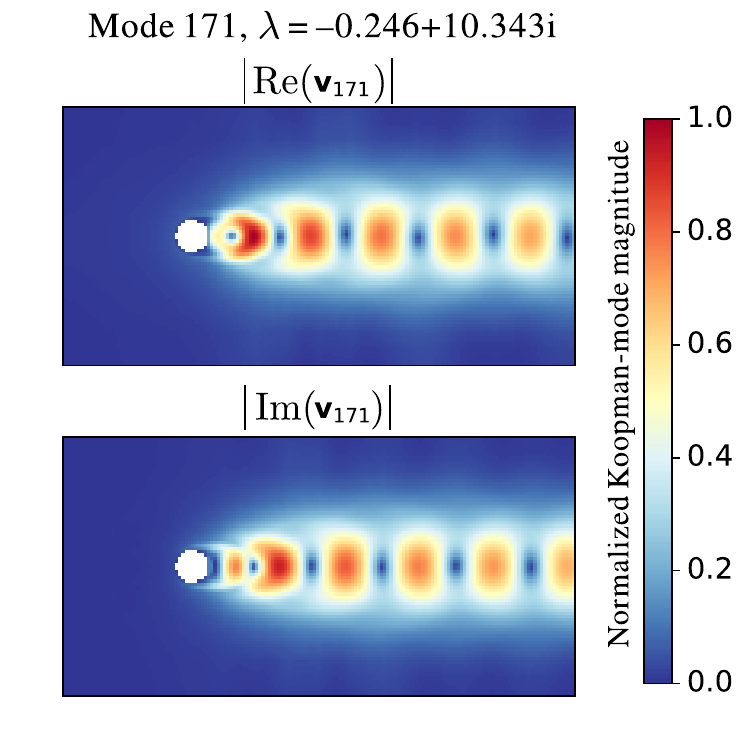}
    \caption{Oscillatory mode.}
    \label{fig:cyl_oscillatory_mode}
    \end{subfigure}
\caption{Generator spectrum and representative spatial modes for the flow around a cylinder.}
\label{fig:cyl_experiment}
\end{figure}

Several modes with positive real parts are present in the learned spectrum, which may reflect the transient dynamics away from the unstable steady solution. Although such modes are associated with growth, we do not observe unstable predictions near the limit cycle, probably because their net contribution to the reconstructed dynamics remains small in this region. In addition, the relatively large residuals of some growing modes may indicate that they are only locally accurate near the unstable state rather than globally valid along the entire trajectory.

\paragraph{Molecular Conformation Dynamics: Alanine Dipeptide}
In this final experiment, we demonstrate that our approach is applicable to high-dimensional data sets. Specifically, we show that the generator approximation can be used to extract distinct molecular conformations and metastable states of alanine dipeptide purely from long and high-dimensional trajectory data.

For molecular dynamics, an evolution in finite time does not suffice for an understanding of the system's long term behaviour; typically high-dimensional systems exhibit a few metastable states determined by a potential energy landscape. 
The global conformation pattern of alanine dipeptide can be represented using two dihedral angles at the central backbone  denoted as \(\varphi\) and \(\psi\) \citep{ramachandran-1963}. 
We use trajectory data from \citet{nuske-2017} and perform a molecular conformation analysis similar to \citet{klus-2018a} and \citet{herwig-2025}. We find that our model identifies the metastable states present in the dataset, due to the eigenfunctions assigning different values to datapoints in different metastable regions, as can be seen in \Cref{fig:AD_eigenfunctions}; these can be clustered using k-means to make the distinctions between regimes clearer as shown in \Cref{fig:AD_kmeans}. In addition, we take a look at the eigenvalues of the generator and their spectral residuals calculated according to \Cref{sec:gEDMD_spectral_pollution}. Eigenvalues with an almost zero magnitude have the smallest spectral residual, which is in line with our expectations, and most eigenvalues are along the imaginary axis.

\begin{figure}[h]
    \centering
    \begin{subfigure}{0.225\textwidth}
        \includegraphics[width=\linewidth]{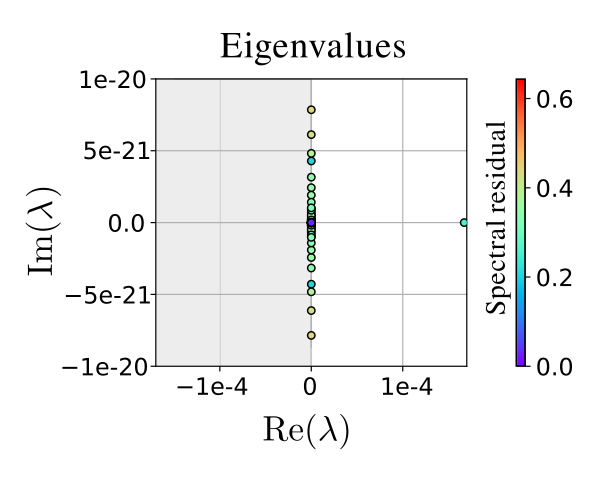}
    \caption{Eigenvalues of \(\hat{L}\).}
        \label{fig:AD_eigenvalues}
    \end{subfigure}
    \vspace{0pt}
    \begin{subfigure}{0.38\textwidth}
\includegraphics[width=\linewidth]{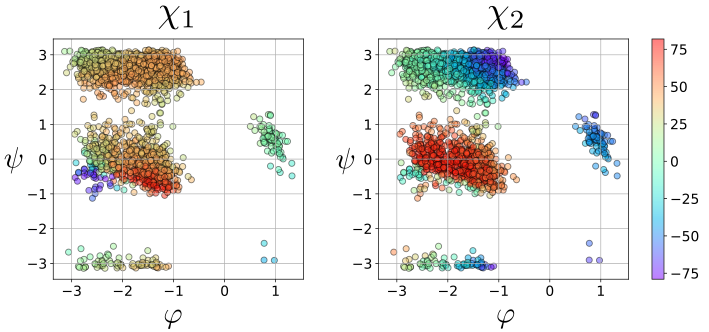}
    \caption{Two eigenfunctions of \(\hat{L}\).}
\label{fig:AD_eigenfunctions}
    \end{subfigure}
\vspace{0pt}
    \begin{subfigure}{0.38\textwidth}
\includegraphics[width=\linewidth]{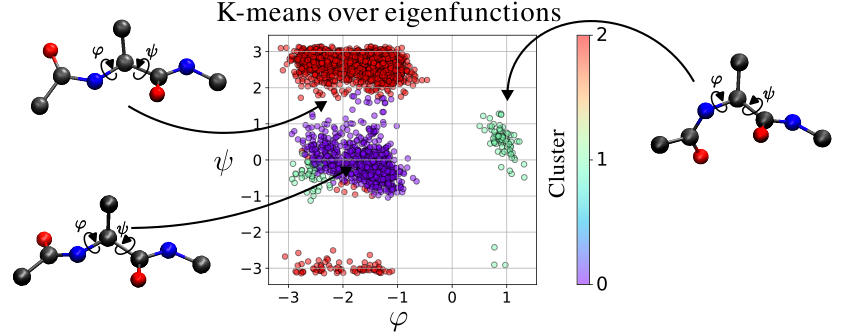}
        \caption{K-means on five eigenfunctions.}\label{fig:AD_kmeans}

    \end{subfigure}
\caption{Spectral insights into the alanine dipeptide dataset using the approximated Koopman generator $\hat{L}$. The datapoints are projected from their 30-dim. state to its corresponding dihedral angles for visualization purposes. With DRAGONS, we are able to directly extract the important metastable conformations of the molecule (\Cref{fig:AD_kmeans}) from the high-dimensional trajectory data.}
\label{fig:AD_experiment}
\end{figure}

\section{Conclusion}

We have shown that data-driven random features provide a robust dictionary for the approximation of the Koopman generator. Our method is supported by empirical results on forecasting problems, demonstrating high accuracy without the high computational costs typically associated with iterative dictionary learning. We also demonstrate how our method can be used to gain insight into the system's spectral properties. Our theoretical results affirm the validity of our method and extend the existing convergence theory for gEDMD. We also extend known limitations of iteratively trained neural networks from discrete-time to a continuous-time setting.

\textbf{Several limitations remain for future work.} Firstly, our convergence results are existence statements without an expression relating the error to the required number of data points or dictionary functions. This could be alleviated by combining results from \citet{llamazares2024data} and \citet{rudi-2017}. Secondly, our theoretical guarantees for spectral convergence are bound to the limitations of gEDMD, which is why future work should look into combining our dictionary with other generator estimation algorithms with improved guarantees. 
Finally, since our algorithm involves solving a potentially large linear system, in the case of sizeable datasets the user needs to carefully choose an efficient linear solver or subsampling the dataset.

\subsection*{AI use statement}


Generative AI tools have been used to assist in the writing of proofs, implement methods, assist with translation and interpret results. 

However, we did not use generative AI tools to help develop theoretical models or conceptual frameworks, generate synthetic data sets, formulate mathematical claims, provide critical ingredients for proving mathematical claims, propose or refine hypotheses, design or provide feedback on research  methodology or experiments, clean and reformat dataset, support qualitative and thematic data analysis. Activities such as formulating questions for surveys or interviews, creation of artifacts and transcribing recordings of research material are not applicable to this work. 

Additionally, we used generative AI tools to summarize or analyse existing literature, source information, edit a research paper to improve readability and identify relevant literature. 

The majority of work done to produce this manuscript was done without the assistance of AI. Key methodological and theoretical ideas came directly from the authors or were inspired by other research which is clearly cited in the manuscript.
Every task performed with the assistance of AI was cross reviewed by authors. AI tools aided the writing of the proof however not in the sense of bridging important gaps, but rather to aid clarity. The core code base is written without any usage of AI tools; AI-generated code was used for auxiliary coding tasks and was verified by at least two authors.
We take full responsibility for the completed work, both for the tasks performed with the assistance of AI as well as the ones without.




\subsection*{Reproducibility statement}

We support our methodology with theoretical results as well as empirical evidence. For the theoretical results we clearly state the necessary assumptions and provide the proofs with sufficient details in the appendix. For the numerical experiments we provide details of the dataset as well as model hyperparameters in the appendix. Our experiments are performed with different random seeds in order to insure robustness to random effects. We provide an anonymized source code and upon acceptance we will make our code repository public so that others can reproduce this work.


\subsubsection*{Acknowledgments}
The authors are grateful to Erik Lien Bolager, Iryna Burak and Chinmay Datar for the helpful discussions and early feedback. Our appreciation also extends to Daniel Lehmberg for the support with the \texttt{datafold} library and to Roman Podieiablonskii for providing the steady-state cylinder flow simulation data. 
 
 V.M.L., A.C. and F.D. are supported by the TUM Georg Nemetschek Institute - Artificial Intelligence for the Built World.
 Z.M. is grateful to the Bundesministerium für Bildung und Forschung (BMBF, Federal Ministry of Education and Research) for funding through project OIDLITDSM, No. 01IS24061.

\newpage
\bibliography{iclr2027_conference}
\bibliographystyle{iclr2027_conference}

\appendix
\section{Related Works}
\label{sec:related_works}
Related works in this direction have primarily focused on learning effective dictionaries for Koopman operator approximation. Ideally, the observables span a compact approximately invariant subspace, but constructing such representations often requires strong prior knowledge or restrictive assumptions~\citep{folkestad2020extended}. Data-driven approaches therefore learn the mapping from the state space to observables directly from data~\citep{dietrich2020koopman}. Neural-network dictionaries are a common choice~\citep{li2017extended}, including methods that learn Koopman eigenfunctions~\citep{lusch2018deep}, jointly optimize the dictionary and operator~\citep{yeung2019learning}, or construct more structured representations using time delays, invertible mappings, or residual information~\citep{valabek2025deep,jin2024extended,coote2026residual}. Although expressive, these methods generally require iterative gradient-based optimization and can therefore be substantially more expensive than the final linear Koopman fit.

A computationally attractive alternative is to construct nonlinear features without iteratively training the hidden representation. Random-feature methods fix or directly sample the nonlinear mapping and reduce learning to a linear estimation problem~\citep{rahimi2007random,rudi-2017}. In Koopman approximation, this idea has been explored using random Fourier features~\citep{bak2022reachability,nuske2023efficient,sprink2026accurate}, randomized approximations of Koopman eigenfunctions~\citep{liao2026approximation}, and reservoir computing~\citep{gulina2021two}.

We discuss these directions in detail below.

\textbf{Explicitly finding a governing equation.}
Learning governing dynamics directly from data has long been an important goal in data-driven modeling, particularly for systems whose analytical equations are unknown or only partially available. One approach is symbolic regression, which searches over candidate analytical expressions to identify compact representations of the underlying dynamics~\citep{bongard2007automated,schmidt2009distilling}. Sparse Identification of Nonlinear Dynamics (SINDy) instead assumes that the governing equations admit a sparse representation in a predefined library of candidate functions and identifies the active terms through sparse regression~\citep{brunton2016discovering}. This sparse-regression framework has also been extended to the discovery of partial differential equations through PDE-FIND~\citep{rudy2017data}.

\textbf{Koopman operator approximation.}
The Koopman operator provides a linear, although generally infinite-dimensional, representation of nonlinear dynamical systems by describing the evolution of observable functions rather than the state itself. Data-driven methods approximate this operator on a finite-dimensional subspace spanned by a dictionary of observables. In particular, Extended Dynamic Mode Decomposition (EDMD) constructs a finite-dimensional approximation from snapshot pairs once such a dictionary has been specified~\citep{williams2015data}. The quality of the resulting approximation, however, depends strongly on this choice. Simple predefined dictionaries, such as monomials, radial basis functions, or other problem-specific features, can be effective for low-dimensional systems but may grow prohibitively large as the state dimension or required approximation order increases.

\textbf{The dictionary-selection problem.}
Selecting an effective dictionary is therefore a central challenge in Koopman operator approximation. Ideally, the observables would span a compact Koopman-invariant, or approximately invariant, subspace and thus provide coordinates in which the nonlinear dynamics evolve linearly. Koopman eigenfunctions constitute a natural choice for such coordinates, but their construction generally requires significant prior knowledge or restrictive assumptions on the underlying dynamical system~\citep{folkestad2020extended}. Conversely, choosing an excessively rich generic dictionary increases both the computational cost and the amount of data required to identify the projected operator. Data-informed dictionaries seek to mitigate this trade-off by constructing representations adapted to the observed dynamics and thereby reducing the number of features required for an accurate approximation~\citep{dietrich2020koopman}.

\textbf{Kernel-based observable spaces.}
Kernel methods provide one route toward richer observable spaces without explicitly enumerating a large dictionary. Kernel DMD represents observables implicitly through the feature map associated with a prescribed kernel, making it possible to operate in very high- or infinite-dimensional feature spaces while carrying out the main computations in terms of kernel evaluations~\citep{williams2015kernel}. This alleviates the explicit feature-growth problem of conventional EDMD, particularly for high-dimensional systems. Nevertheless, the kernel type and its hyperparameters remain design choices, and the computational cost typically scales unfavorably with the number of snapshots. More recently, kernel selection itself has therefore also been formulated as a dictionary-learning problem, in which combinations of kernels and their parameters are optimized from data~\citep{bolager2026dictionary}. Thus, although kernel methods reduce the need to manually specify individual observables, the broader problem of identifying a representation appropriate for the dynamics remains.

\textbf{Neural dictionary learning.}
A widely studied alternative is to parametrize the dictionary directly using neural networks. Early EDMD dictionary-learning approaches jointly adapt a neural representation and the associated Koopman approximation to the observed dynamics~\citep{li2017extended}. DeepKoopman similarly learns approximate Koopman eigenfunctions from data, demonstrating that nonlinear coordinate transformations can be constructed without prescribing the observable functions beforehand~\citep{lusch2018deep}. Related architectures combine an encoder, approximately linear latent dynamics, and a decoder to learn low-dimensional representations in which the evolution is close to linear~\citep{otto2019linearly}. Other formulations directly optimize both the neural-network dictionary and the projected Koopman operator~\citep{yeung2019learning}, use time-delay embeddings to incorporate temporal information~\citep{valabek2025deep}, or employ invertible neural networks to obtain an explicit inverse transformation between the state and observable spaces~\citep{jin2024extended}. These approaches substantially increase the flexibility of the dictionary and can yield compact representations for systems for which suitable analytical observables are not known.
 
\textbf{Neural ODE-based dictionaries and joint Koopman optimization.}
Several recent works extend this basic neural dictionary-learning paradigm. Neural ODEs have been used to construct parameter-efficient EDMD dictionaries by interpreting the feature transformation as the flow of an auxiliary continuous-depth network~\citep{terao2021extended}. Connections between Koopman dictionary learning and neural state-space models have subsequently been studied in more detail. In particular, repeatedly projecting the linearly evolved lifted state back to the physical state space, as commonly done in practical EDMD implementations~\citep{van2023reprojection,junker2022data}, reintroduces a nonlinear state-space evolution and creates a close relationship between projected EDMD models and neural ODE or neural flow-map representations~\citep{buzhardt2025relationship}. Automatic differentiation through the EDMD solution has likewise been used to optimize the dictionary and Koopman approximation simultaneously rather than alternating between the two problems~\citep{constanteamores2024enhancing}.

\textbf{Structure- and residual-aware dictionaries.}
The representation can also be adapted to properties beyond short-term prediction accuracy. Trainable dictionaries have, for example, been combined with Residual DMD by incorporating spectral residuals of approximate Koopman eigenvalues and eigenfunctions into the learning objective~\citep{coote2026residual}. This explicitly targets the spectral consistency of the learned representation rather than relying solely on reconstruction or trajectory errors. Related work constructs approximately invariant observable subspaces through optimization on the Grassmann manifold, allowing part of the observable space to be constrained by prior knowledge while the remaining subspace is adapted to the data~\citep{schurig2025shaping}. For parametrized dynamical systems, both the dictionary and a parameter-dependent family of projected Koopman operators can be represented by neural networks and learned jointly, permitting the operator itself to vary with static or time-dependent parameters~\citep{guo2025learning}. \citet{bai-2025} use the Koopman framework to evolve the parameters of a Neural ODE, although they do not perform dictionary learning, making them an exception to this type of learning.

\textbf{Neural Network approximation of generators for continuous spectra.} Neural approaches have also been developed directly at the generator level. KoopGen, for instance, represents structured state-dependent Koopman generators and is designed to accommodate systems with broadband or continuous spectral components without assuming that the relevant dynamics admit a simple global finite-dimensional eigendecomposition~\citep{su2026koopgen}, however since it does not learn dictionaries, the approach is not directly comparable to our method.

\textbf{Computational cost of trainable dictionaries.}
These developments highlight the expressive power of neural dictionary learning, but this flexibility comes at a computational cost. The parameters defining the observables must typically be determined through iterative gradient-based optimization, often jointly with operator, decoder, or reconstruction parameters. Consequently, dictionary construction can require many passes over the training data and can become substantially more expensive than the final linear least-squares problem underlying classical EDMD. The result may additionally depend on network architecture, initialization, learning rates, regularization, and other optimization hyperparameters. This distinction is important: once a fixed dictionary has been constructed, the corresponding Koopman approximation is essentially a linear regression problem, whereas learning the dictionary can turn the overall procedure into a large non-convex optimization problem.

\textbf{Convergence and spectral reliability.}
Care is also required when discussing convergence of such approximations. For a fixed sequence of sufficiently rich approximation spaces and appropriate sampling assumptions, convergence results are available for EDMD; in particular, finite-dimensional projected operators converge to the Koopman operator in suitable operator senses as the number of samples and dictionary functions increases~\citep{korda2018convergence}. This does not, however, imply that every eigenvalue of a finite EDMD matrix provides a reliable approximation of the spectrum of the infinite-dimensional Koopman operator. Finite-dimensional truncations can exhibit spectral pollution, particularly when the underlying operator has continuous or otherwise complicated spectral components. Residual DMD addresses this distinction by computing operator-level residuals that permit a posteriori verification of candidate spectral quantities and, under suitable assumptions, provide rigorous convergence guarantees for spectra, pseudospectra, and spectral measures~\citep{colbrook-2024a}. Recent residual-guided dictionary-learning methods build directly on this observation by incorporating such residuals into the representation-learning objective~\citep{coote2026residual}.

\textbf{Random-feature alternatives.}
A conceptually different approach is to retain expressive nonlinear representations while avoiding iterative optimization of the hidden mapping altogether. Random-feature methods follow precisely this strategy: rather than learning all parameters of a nonlinear feature map through gradient descent, the hidden parameters are sampled or otherwise constructed once, after which only a linear output model is fitted~\citep{rahimi2007random,rahimi2008weighted}. Such models possess strong approximation properties for broad classes of functions, including Barron-type functions~\citep{barron1993universal} and functions in reproducing kernel Hilbert spaces~\citep{rahimi2008uniform}. Conventional random-feature methods typically draw their parameters from a fixed, data-independent distribution; random Fourier features, for example, sample frequencies according to the spectral distribution associated with a shift-invariant kernel~\citep{rahimi2007random}. Despite their simple training procedure, random features provide favorable trade-offs between approximation accuracy, feature dimension, and computational complexity~\citep{rudi-2017}. From the perspective of Koopman approximation, this separation is particularly natural: the nonlinear features define a candidate observable space, while the action of the Koopman operator on this space can subsequently be estimated through linear regression.

\textbf{Data-dependent sampling with SWIM.}
Rather than sampling features from a prescribed, data-independent distribution, \citet{bolager2023sampling} introduced \textit{Sampling Where It Matters} (SWIM), a data-driven procedure for constructing the weights and biases of shallow and deep neural networks without gradient-based optimization. SWIM samples pairs of training points and constructs individual neurons from their relative positions, assigning higher sampling probability to pairs exhibiting large changes in the target function relative to their distance in feature space. In this way, features are preferentially placed in regions where the target function varies strongly. After the hidden layers have been constructed, only the final linear layer is fitted. The resulting sampled networks retain universal approximation properties and, in the experiments of \citet{bolager2023sampling}, achieve accuracy comparable to conventionally trained networks while requiring substantially less training time.

A closely related combination of SWIM and Koopman operator approximation was recently proposed by \citet{bolager2026koopmaninformed} for constructing \textit{Koopman-informed recurrent neural networks}. There, SWIM is used to generate a data-adapted nonlinear lifting without backpropagation, while EDMD determines a finite-dimensional Koopman operator governing the evolution in the lifted space. A recurrent update is then obtained by repeatedly lifting the current state, evolving it linearly through the learned Koopman operator, and projecting it back to the state space. Consequently, both the nonlinear feature representation and the recurrent dynamics can be constructed without gradient-based training. This work demonstrates that SWIM-generated observables can be successfully coupled with Koopman approximations for dynamical-system learning and provides a particularly close precursor to the dictionary-construction strategy considered here.

\textbf{Random Fourier features for Koopman learning.}
Randomized feature construction has consequently appeared in Koopman methods in several forms. Random Fourier features approximate translation-invariant kernels with an explicit finite-dimensional feature map, avoiding the construction and factorization of the full kernel matrix. Importantly, their use in EDMD predates recent neural dictionary-learning approaches: random Fourier features and Nyström approximations were proposed as scalable approximations of kernel EDMD specifically to reduce the computational cost associated with large datasets~\citep{degennaro2019scalable}. Random Fourier dictionaries have since been applied to Koopman approximation in different settings~\citep{bak2022reachability,nuske2023efficient, tabish-2025, sprink2026accurate}, and randomized shallow feature models have also been used directly to approximate Koopman eigenfunctions~\citep{liao2026approximation}. These approaches replace the nonlinear optimization of a neural dictionary by the generation of a fixed collection of nonlinear observables followed by the usual linear operator fit.

\textbf{Reservoir-based dictionaries.}
Reservoir computing follows a closely related philosophy. A randomly initialized recurrent network maps the input into a high-dimensional nonlinear dynamical representation, while its internal parameters remain fixed and only a linear readout is trained~\citep{bollt2021explaining,gauthier2021next}. The resulting optimization is therefore typically reduced to a least-squares problem. In a Koopman context, reservoir states can themselves serve as observables, allowing the reservoir to provide a nonlinear dictionary without backpropagating through its internal weights. Reservoir-based methods have explicitly exploited this idea for Koopman operator approximation and dictionary construction using linear convex optimization~\citep{gulina2021two}. The recurrent structure additionally provides an implicit representation of temporal history, distinguishing reservoir features from purely state-dependent random mappings.

\section{Koopman Theory}\label{sec:appendix_koopman}
\subsection{Koopman operator}

The Koopman operator provides a linear representation of potentially nonlinear dynamical systems by describing the evolution of observable functions rather than the state itself~\citep{budivsic2012applied}.

Consider a dynamical system on a state space $\mathcal{X}\subseteq\mathbb{R}^{d}$ equipped with a measure $\mu$, where $\vz_t\in\mathcal{X}$ denotes the state at time $t$. Its evolution is governed by
\begin{equation}
    \frac{d}{dt}\vz_t = F(\vz_t),
\end{equation}
where $F:\mathcal{X}\mapsto\mathbb{R}^{d}$ is the vector field describing the instantaneous velocity at each state. Assuming that each initial state in $\mathcal{X}$ admits a unique
solution that remains in $\mathcal{X}$ for all $t\geq0$, the flow map $\flow{t}:\mathcal{X}\mapsto\mathcal{X}$ maps an initial state to its state at time $t$:
\begin{equation}
    \vz_t = \flow{t}(\vz_0).
\end{equation}

The system can therefore be described by the triplet
$(\mathcal{X},\mathcal{T},\flow{\cdot})$, where
$\mathcal{T}\subseteq\mathbb{R}$ is the time domain. For
forward-time evolution, the flow maps $\{\flow{t}\}_{t\geq0}$
form a semigroup, satisfying
\begin{equation}
    \flow{0}=\operatorname{Id}_{\mathcal{X}},
    \qquad
    \flow{t_1}\circ\flow{t_2}=\flow{t_1+t_2}.
\end{equation}

Let $\mathcal{F}$, typically a Hilbert space, be a space of scalar-valued functions
$g:\mathcal{X}\mapsto\mathbb{C}$ called \textit{observables}, chosen such that $g\circ\flow{t}\in\mathcal{F}$ for every
$g\in\mathcal{F}$ and $t\geq0$. The Koopman operator
$\kot:\mathcal{F}\mapsto\mathcal{F}$ is defined by
\begin{equation}
    \kot g = g\circ\flow{t},
    \qquad
    (\kot g)(\vz_0)
    =
    g(\flow{t}(\vz_0))
    =
    g(\vz_t).
\end{equation}
Thus, evolving the state for a time $t$ and then evaluating $g$
gives the same result as applying $\kot$ to $g$ and evaluating
the resulting observable at the initial state. Although the
underlying flow map may be nonlinear, the Koopman operator is
linear in the observable:
\begin{equation}
    \kot(\alpha g_1+\beta g_2)
    =
    \alpha\kot g_1+\beta\kot g_2,
    \qquad
    g_1,g_2\in\mathcal{F},
    \quad
    \alpha,\beta\in\mathbb{C}.
\end{equation}
The flow's semigroup property is inherited by
$\{\kot\}_{t\geq0}$, called the \textit{Koopman semigroup}.

This framework trades nonlinear evolution in a finite-dimensional state space for linear evolution in a typically infinite-dimensional function space, enabling the use of linear spectral methods to analyze nonlinear dynamics. A nonzero observable $\chi\in\mathcal{F}$ is an eigenfunction
of $\kot$ with eigenvalue $\lambda$ if
\begin{equation}
    (\kot\chi)(\vz_0)
    =
    \chi(\flow{t}(\vz_0))
    =
    \lambda\chi(\vz_0),
    \qquad \forall \vz_0\in\mathcal{X}.
\end{equation}
When a vector-valued observable admits an eigenfunction
expansion, it can be expressed as a linear combination of
eigenfunctions with vector coefficients called Koopman modes.
These modes specify fixed directions in the observable's
output space, while the eigenfunctions determine their
state-dependent amplitudes.

For numerical approximation, a finite dictionary of observables is selected to span a finite-dimensional subspace of $\mathcal{F}$. Such a subspace is generally not invariant under the Koopman operator: evolving an observable in the subspace may produce a function outside it. Representing the evolution within the chosen subspace therefore introduces an approximation.

\subsection{gEDMD algorithm}

Generator extended dynamic mode decomposition (gEDMD) approximates
the Koopman generator on a finite-dimensional space of observables
using derivative information~\citep{klus2020data}.

For a strongly continuous Koopman semigroup $\{\kot\}_{t\geq0}$,
the generator $\mathcal{L}:\mathcal{D}\rightarrow\mathcal{F}$
is defined by
\begin{equation}
    \label{eq:appendix_generator_definition}
    \mathcal{L}g
    :=
    \lim_{t\downarrow0}
    \frac{\kot g-g}{t},
    \qquad g\in\mathcal{D}.
\end{equation}
It describes the instantaneous rate of change of observables.
The domain $\mathcal{D}\subseteq\mathcal{F}$ consists of those
observables for which this limit exists in the norm of
$\mathcal{F}$.

For the system $\dot{\vz}_t=F(\vz_t)$ and sufficiently
differentiable $g\in\mathcal{D}$, the chain rule yields
\begin{equation}
    \label{eq:appendix_generator_as_time_der}
    (\mathcal{L}g)(\vz)
    =
    \nabla g(\vz)^\top F(\vz).
\end{equation}
Along a trajectory, this gives
\begin{equation}
    \frac{d}{dt}g(\vz_t)
    =
    (\mathcal{L}g)(\vz_t).
\end{equation}

A nonzero observable $\chi\in\mathcal{D}$ is a generator
eigenfunction with eigenvalue $\lambda$ if
\begin{equation}
    \mathcal{L}\chi=\lambda\chi.
\end{equation}
It therefore evolves exponentially along trajectories:
\begin{equation}
    \chi(\vz_t)
    =
    (\kot\chi)(\vz_0)
    =
    e^{\lambda t}\chi(\vz_0).
\end{equation}
Thus, $\chi$ is also an eigenfunction of $\kot$, with eigenvalue
$e^{\lambda t}$. The real and imaginary parts of the generator
eigenvalue $\lambda$ describe its growth or decay rate and
angular oscillation frequency, respectively.

To approximate the generator we gEDMD, choose a dictionary of sufficiently
differentiable functions $\{\phi_i\}_{i=1}^{N}\subseteq\mathcal{D}$
spanning a finite-dimensional subspace of $\mathcal{F}$ and
collect its values in
\begin{equation}
    \Phi_N(\vz)
    =
    \begin{bmatrix}
        \phi_1(\vz) & \cdots & \phi_N(\vz)
    \end{bmatrix}^{T}.
\end{equation}
Given $M$ sampled states $\vz_m\in\mathcal{X}$ and their
corresponding velocities $\dot{\vz}_m=F(\vz_m)$,
Eq.~\ref{eq:appendix_generator_as_time_der} yields
\begin{equation}
    \dot{\phi}_i(\vz_m)
    =
    (\mathcal{L}\phi_i)(\vz_m)
    =
    \nabla\phi_i(\vz_m)^\top\dot{\vz}_m.
\end{equation}
Here, $\dot{\phi}_i(\vz):=(\mathcal{L}\phi_i)(\vz)$ denotes
the derivative of $\phi_i$ along the dynamics.
The velocities must be available from measurements or the
vector field, or estimated from trajectory data. Define
\begin{equation}
    \dot{\Phi}_N(\vz)
    =
    \begin{bmatrix}
        \dot{\phi}_1(\vz) & \cdots & \dot{\phi}_N(\vz)
    \end{bmatrix}^{T},
\end{equation}
and assemble the data matrices
\begin{equation}
    \begin{aligned}
        \Phi_X
        &=
        \begin{bmatrix}
            \Phi_N(\vz_1) & \cdots & \Phi_N(\vz_M)
        \end{bmatrix},
        \\
        \dot{\Phi}_X
        &=
        \begin{bmatrix}
            \dot{\Phi}_N(\vz_1) & \cdots & \dot{\Phi}_N(\vz_M)
        \end{bmatrix}.
    \end{aligned}
\end{equation}
Both matrices have dimensions $N\times M$.

The generator matrix approximation is obtained by solving
a least-squares problem:
\begin{equation}
    \hat{L}_{NM}^{T}
    \in
    \operatorname*{arg\,min}_{B\in\mathbb{C}^{N\times N}}
    \left\|\dot{\Phi}_X-B\Phi_X\right\|_F^2,
\end{equation}
where $\|\cdot\|_F$ denotes the Frobenius norm.
Its minimum Frobenius-norm solution is
\begin{equation}
    \label{eq:appendix_gedmd_least_squares}
    \hat{L}_{NM}^{T}
    =
    \dot{\Phi}_X\Phi_X^{+},
\end{equation}
where $+$ denotes the Moore--Penrose pseudoinverse.

The transpose arises because dictionary values are arranged
as column vectors.

Equivalently, define the empirical Galerkin matrices
\begin{equation}
    \label{eq:appendix_gedmd_galerkin}
    \hat{\mathbf{C}}_{NM}
    =
    \frac{1}{M}\dot{\Phi}_X\Phi_X^{*},
    \qquad
    \hat{\mathbf{G}}_{NM}
    =
    \frac{1}{M}\Phi_X\Phi_X^{*},
\end{equation}
where $*$ denotes the conjugate transpose. Their entries are
\begin{equation}
    \begin{aligned}
        [\hat{\mathbf{C}}_{NM}]_{i,j}
        &=
        \frac{1}{M}\sum_{m=1}^{M}
        \compconj{\phi_j(\vz_m)}
        \mathcal{L}\phi_i(\vz_m),
        \\
        [\hat{\mathbf{G}}_{NM}]_{i,j}
        &=
        \frac{1}{M}\sum_{m=1}^{M}
        \compconj{\phi_j(\vz_m)}
        \phi_i(\vz_m).
    \end{aligned}
\end{equation}
The generator approximation can then be written as
\begin{equation}
    \hat{L}_{NM}^{T}
    =
    \hat{\mathbf{C}}_{NM}\hat{\mathbf{G}}_{NM}^{+},
\end{equation}
which is equivalent to
Eq.~\ref{eq:appendix_gedmd_least_squares}.

The resulting matrix enables spectral analysis and prediction of dictionary values through
\begin{equation}
    \Phi_N(\vz_t)
    \approx
    \exp\!\left(t\hat{L}_{NM}^{T}\right)\Phi_N(\vz_0).
\end{equation}
This finite-dimensional evolution is generally approximate because the dictionary span need not be invariant under the generator. Its accuracy also depends on the sampled states and the quality of the derivative information. Predicting states additionally requires a reconstruction map from dictionary values to state coordinates.
\section{DRAGONS sampling distribution}
\label{app:sampling}
By choosing the target function $f$ from which we construct $\sP_{\gX}$, we can determine how the data informs the dictionary.
Here, we provide further details on and the intuition behind the different sampling strategies we propose and evaluate for DRAGONS.
\paragraph{Flow-driven} This is a natural choice since it is a direct application of the SWIM algorithm to our given dataset. Further, this has been thoroughly evaluated and empirically shown to perform well across various tasks in the Koopman-Informed Recurrent Neural Networks architecture proposed in ~\citet{bolager2026koopmaninformed}.
    Since our dataset consists of pairs $(\vz_i, \flow{\tau_i}({\vz_i}))$, their strategy can immediately be used to place the neurons in regions of the state space where the flow map exhibits complex behaviour, indicated by large changes over small regions. These considerations yield the density
    \begin{equation}
         p_{\gX}(\vz^{(1)},\vz^{(2)})\propto \frac{\norm{ \flow{\tau_2}(\vz^{(2)}) - \flow{\tau_1}(\vz^{(1)})}}{\norm{\vz^{(2)} - \vz^{(1)}}}.
         \label{eq:flow-driven sampling}
    \end{equation}
    Notably, however, the observed future states are in general not advanced by the same time step in our setting, i.e. $\tau_1\neq \tau_2$. This means that in the limit of $z^{(2)}\rightarrow z^{(1)}$, the above score does not converge to a directional derivative since the temporal offset introduces additional variation that is not explained by the spatial offset between the input points. This merits investigating alternative constructions. 
    
    \paragraph{Vector field-driven} Since the network ultimately parameterises the vector field $F$ rather than the flow map $\flow{}$, we also consider using this as a sampling target. This is further motivated by the close relation of the Koopman generator to the vector field, as evidenced by the relationship $\tgen g = \ip{\nabla g, F}$. This suggests sampling according to
        \begin{equation}
            p_{\gX}(\vz^{(1)},\vz^{(2)})\propto \frac{\norm{ F(\vz^{(2)}) - F(\vz^{(1)})}}{\norm{\vz^{(2)} - \vz^{(1)}}}.
        \end{equation}
    If the vector field is unavailable, it can be approximated from the flow map using a suitable finite-difference method.
    \paragraph{Generator-driven} Another route, motivated by dictionary learning, is to choose observables that faithfully represent the generator's action on the dictionary, analogous to how \citet{li2017extended} construct the training loss by targeting invariant dictionaries. Using the explicit expression of the SWIM-neurons in terms of the sampled points, we can already compute the sampling distribution from how a candidate neuron would behave if sampled without actually having to construct and evaluate it. We now describe how to turn this into a sampling strategy targeting the generator, since this is the main object the neurons need to represent.\\
    Define $\oswim{kl}$ as the SWIM-neuron constructed from the data points, $\vz^{(k)}$ and $\vz^{(l)}$, concretely $\oswim{kl}(\vz)=\sigma(\vw_{kl}^\top\vz+b_{kl})$, where
    \begin{align}
         \vw_{kl}&= s_1\frac{\vz^{(l)}-\vz^{(k)}}{\norm{\vz^{(l)}-\vz^{(k)}}^2},\\
         b_{kl}&=-\ip{\vw_{kl}, \vz^{(k)}}+s_2.
    \end{align}
    An important consequence of the above is that $\oswim{kl}(\vz^{(k)})=\sigma(s_2)$ and $\oswim{kl}(\vz^{(l)})=\sigma(s_1 + s_2)$, which means that we can precompute $\nabla\oswim{kl}$ at the sampled points. Even further, for ReLU and Tanh as activation functions with the default values for $s_1,s_2$, $\sigma'(\vw_{kl}^\top\vz^{(k)}+b_{kl})=\sigma'(\vw_{kl}^\top\vz^{(l)}+b_{kl})$, using the convention that $\text{ReLU}'(x)=\ind{x\geq0}{x}$.
    
    With these considerations in mind, we can define a sampling distribution with $f=\tgen\oswim{12}$ through the density
    \begin{align}
            p_{\mathcal{X}}(\vz^{(1)},\vz^{(2)})&\propto \frac{\norm{ \tgen\oswim{12}(\vz^{(2)}) - \tgen\oswim{12}(\vz^{(1)})}}{\norm{\vz^{(2)} - \vz^{(1)}}}\\
            &\propto\frac{\abs{ \ip{\nabla\oswim{12}(\vz^{(2)}),F(\vz^{(2)})} - \ip{\nabla\oswim{12}(\vz^{(1)}),F(\vz^{(1)})}}}{\norm{\vz^{(2)} - \vz^{(1)}}} \\
            &\propto\frac{\abs{ \sigma'(\vw_{12}^\top\vz^{(2)}+b_{12})\ip{\vw_{12},F(\vz^{(2)})} - \sigma'(\vw_{12}^\top\vz^{(1)}+b_{12})\ip{\vw_{12},F(\vz^{(1)})}}}{\norm{\vz^{(2)} - \vz^{(1)}}} \\
            &\propto\frac{\abs{\ip{\vw_{12},F(\vz^{(2)})} - \ip{\vw_{12},F(\vz^{(1)})}}}{\norm{\vz^{(2)} - \vz^{(1)}}} \\
        &\propto\frac{\abs{\ip{s_1\frac{\vz^{(2)}-\vz^{(1)}}{\norm{\vz^{(2)}-\vz^{(1)}}^2},F(\vz^{(2)})-F(\vz^{(1)})} }}{\norm{\vz^{(2)} - \vz^{(1)}}} \\
            &\propto\frac{\abs{(\vz^{(2)} - \vz^{(1)})^\top (F(\vz^{(2)}) - F(\vz^{(1)}))}}{\norm{\vz^{(2)} - \vz^{(1)}}^3},
    \end{align}
    where we absorbed the $\sigma'$ term and $s_1$ into the normalisation constant.  This can be intuitively understood as using the neuron to test how $\tgen$ varies over the space of observables.

\section{Convergence of vector field approximation and resulting trajectories}
\label{app:convergence_theory}
In this section, we give the necessary assumptions and proofs for the theoretical results about DRAGONS and gEDMD in the main paper.
\subsection{Assumptions \& Notation}
\label{app:theory_ass}
 Following \citet{bolager2026koopmaninformed}, we define the state space $\gX$ as in \citet{bolager2023sampling}. We set
\begin{align*}
    d(\vz, A) = \inf\{d(\vz,\va) \colon \va \in A\},
\end{align*}
where \(d\) is the canonical Euclidean distance in the space \(\mathbb R^\dz\). The medial axis is defined as
\begin{align*}
    \text{Med}(A) = \{ \vx \in \mathbb{R}^{d} \colon \exists \vp\neq \vq \in A, \lVert \vp - \vx \rVert = \lVert \vq - \vx \rVert = d(\vx, A)\}
\end{align*}
and the reach is the scalar
\begin{align*}
    \tau_A = \inf_{\va\in A} d (\va, \text{Med}(A)),
\end{align*}
i.e., the point in \(A\) that is closest to the projection of points in \(A^c\).
\begin{definition}\label{def:input_space}
    Let \(\widetilde{\gX}\) be a nonempty compact subset of \(\mathbb R^{d}\) with reach \(\tau_{\widetilde{\gX}} > 0\). The input space \(\gX\) is defined as 
    \begin{align*}
        \gX = \{\vx \in \mathbb R^{d} \colon d(\vx,\widetilde{\gX}) \leq \epsilon_\gX \},
    \end{align*}
    where \(0 < \epsilon_\gX < \min\{\tau_{\widetilde{\gX}},1\}\).
\end{definition}
\begin{remark}
    This restriction to the type of state and input spaces we consider is sufficient to construct all neural networks of interest by choosing pairs of points from the space in question, and construct the weight and bias as in \Cref{eq:sampled_weights}. It is also argued in \citet{bolager2023sampling} that most interesting real-world applications will contain some noise and make sure that the state and input spaces are approximately in the form given in the definition.
    \end{remark}

Furthermore, the convergence results from \citet{llamazares2024data}, on which our convergence proofs are based, require the following assumptions, which we restate here in full for convenience before showing that sampled networks meet the requirements on the dictionary and giving a simple sufficient condition on the dynamical system for the assumptions.
\begin{assumption}
\label{ass:koop1}
We assume the following:
\begin{enumerate}
    \item\label{ass:koop1a} The dictionary functions $\dswim{N} = \{\oswim{1}, \dots, \oswim{N}\} \subset \dom{}$ are linearly independent.
    \item\label{ass:koop1b} The functions $\{\oswim{n}, \tgen \oswim{n}\}_{n=1}^N$ are continuous $\mu$ almost everywhere.
    \item\label{ass:koop1c} The points $\{\vx_m\}_{m=1}^M \subset \gX$ are i.i.d.\ samples from $\mu$.
\end{enumerate}
\end{assumption}

\begin{assumption}
\label{ass:koop2}
    The dictionary functions $\dswim{N} = \{\oswim{n}\}_{n=1}^N$ satisfy $\oswim{j} \tgen \oswim{i} \in \fss{}$ as well as $\oswim{i}^2 \in \fss{}$ for all $i, j \in \{1, \dots, N\}$.
\end{assumption}
\begin{assumption}
\label{ass:koop3}
\[
\lim_{N \to \infty} \| \projf{N} f - f \|_\fss{} = 0, \quad \forall f \in \fss{},
\]
where $\projf{N}$ is the projection of $\fss{}$ onto $\fss{N}$ using the inner product on $\fss{}$.
\end{assumption}

\begin{assumption}
\label{ass:koop4}
$\tgen$ is a closed operator, and
\[
\lim_{N \to \infty} \| \projd{N} f - f \|_\dom{} = 0, \quad \forall f \in \dom{}.
\]
Here, $\projd{N}$ is the projection of $\dom{}$ onto $\fss{N}$ using the inner product on $\dom{}$.
\end{assumption}

Finally, we add the following standard assumption on the probability measure $\mu$, which leads to simple sufficient conditions for the boundedness and strong continuity of the Koopman Operator semigroup.

\begin{assumption}
    \label{ass:measure}
    The measure $\mu$ has an everywhere positive and bounded density $\rho$ w.r.t. the Lebesgue measure $\lambda$ on $\gX$, where $0<\alpha\leq\rho(\vx)\leq\beta<\infty, \forall \vx\in\gX$.
\end{assumption}

\paragraph{Notations.} We further define the following notations to be used throughout this section.
\begin{itemize}
    \item Space of observables $\fss{}=\fso{\gX}$, equipped with its standard inner product $\ip{\cdot,\cdot}_\fss{}$ and norm $\norm{\cdot}_\fss{}$.
    \item Operator norm induced by the $\fss{}$ norm $\norm{\kot}_{\gB(\fss{})}$.
    \item Span of the dictionary $\fss{N}=\text{span}(\dswim{N})$.
    \item Indicator functions $\mathbf{1}_{\sA}(x)=\begin{cases}
        1, x\in\sA,\\
        0, x\notin\sA.
    \end{cases}$
    \item The Euclidean inner product and norm, $\ip{\cdot,\cdot}$ and $\norm{\cdot}$, are denoted without subscript.
\end{itemize}

\subsection{Validity of SWIM-sampled Dictionary}
Next, we will show the validity of our dictionary and state simple, practically oriented, sufficient conditions for abstract properties that we require on the dynamical system.

Firstly, we require exponential boundedness of the Koopman operator $\kot$ and closedness of the generator $\tgen$. Since both of these properties are fulfilled when the Koopman operators form a $C_0$-Semigroup, we provide the following lemma showing that this is, for example, attained in the case of Lipschitz Vector Fields. This is also the natural setting in which to work, since only then are we guaranteed the existence of a unique solution to the ODE by the Picard-Lindelöf Theorem.

\begin{lemma}[Lipschitz Vector Fields have strongly continuous Koopman operators]
\label{lemm:c0}
    Let $F(z)=\left(F_1(z),F_2(z),\dots,F_d(z)\right)^\top$ be a Lipschitz continuous vector field with Lipschitz constant $L$, and $\gX$ be forward invariant under the flow map $\flow{t}$, i.e. $\flow{t}(\gX)\subset\gX$. Then the Koopman operators form a strongly continuous semigroup. In specific, the Koopman generator $\tgen$ is a closed operator and $\kot$ fulfills the exponential bound $\norm{\kot}_{\gB(\fss{})}\leq \sqrt{\frac{\beta}{\alpha}}e^{\omega t}$ with $\omega=\frac{d\cdot L}{2}$.
\end{lemma}
\begin{proof}
    We will first show exponential boundedness and then use it to prove strong continuity, from which the closedness of the generator immediately follows \closedgen. Let $f\in\fss{}=\fso{\gX}$. Then
    \begin{align}
        \norm{\kot f}_{\fss{}}^2&=\int_{\gX}f\left(\flow{t}(\vz_0)\right)^2\dm{\mu(\vz_0)}&&\\
        &=\int_{\gX}f\left(\flow{t}(\vz_0)\right)^2\rho(\vz_0)\dm{\lambda(\vz_0)}&&\\
        &=\int_{\flow{t}(\gX)}f\left(\vz_t\right)^2\left|\det \nabla \flow{-t}(\vz_t)\right|\frac{\rho\left(\flow{-t}(\vz_t)\right)}{\rho(\vz_t)}\rho(\vz_t)\dm{\lambda(\vz_t)}&&\text{($\vz_t=\flow{t}(\vz_0)$)}\\
        &\leq\frac{\beta}{\alpha}\int_{\flow{t}(\gX)}f\left(\vz_t\right)^2\left|\det \nabla \flow{-t}(\vz_t)\right|\rho(\vz_t)\dm{\lambda(\vz_t)}.&&
    \end{align}
    Next, for $z_t\in\flow{t}(\gX)$, we estimate the Jacobian determinant by
    \begin{align}
        \label{c0:det}
        \left|\det \nabla \flow{-t}(\vz_t)\right| &= 
        \prod_{i=1}^d \sigma_i(\nabla \flow{-t}(\vz_t)) \leq \sigma_{\max}(\nabla \flow{-t}(\vz_t))^d=\norm{\nabla \flow{-t}(\vz_t)}^d_{\gB(\ell^2)},
    \end{align}
    where the norm is the matrix operator norm induced by the $\ell^2$ norm, i.e. the spectral norm. Since the operator norm of the Jacobian is bounded by the Lipschitz constant, we next show that $\flow{-t}$ is Lipschitz continuous and estimate its constant. For $z_1,z_2\in\flow{t}(\gX)$ and $\tau\leq t$ (so that the backwards flow map is defined), we have
    \begin{align}
        \frac{\dm{}}{\dm{\tau}} \norm{\flow{-\tau} (z_1)-\flow{-\tau} (z_2)}^2
        &=\frac{\dm{}}{\dm{\tau}} \ip{\flow{-\tau} (z_1)-\flow{-\tau} (z_2),\;\flow{-\tau} (z_1)-\flow{-\tau} (z_2)}\\
        &=2\ip{\frac{\dm{}}{\dm{\tau}}\flow{-\tau} (z_1)-\frac{\dm{}}{\dm{\tau}}\flow{-\tau} (z_2),\;\flow{-\tau} (z_1)-\flow{-\tau} (z_2)}\\
        &=-2\ip{F(\flow{-\tau} (z_1))-F(\flow{-\tau} (z_2)),\;\flow{-\tau} (z_1)-\flow{-\tau} (z_2)}\\
        &\leq2\norm{F(\flow{-\tau} (z_1))-F(\flow{-\tau} (z_2))}\norm{\flow{-\tau} (z_1)-\flow{-\tau} (z_2)}\\
        &\leq2 L\norm{\flow{-\tau} (z_1)-\flow{-\tau} (z_2)}^2,
    \end{align}
    where we used Cauchy-Schwartz and the Lipschitzness of the vector field in the last two steps.
    Applying Grönwall's inequality to the above now yields
    \begin{equation}
        \norm{\flow{-\tau} (z_1)-\flow{-\tau} (z_2)}^2\leq \exp(2L\tau) \norm{\flow{0} (z_1)-\flow{0} (z_2)}^2=\exp(2L\tau)\norm{z_1-z_2}^2,
    \end{equation}
    which, after taking the square root, shows that $\flow{-t}$ is $ \exp(Lt)$-Lipschitz. The same argument shows that $\flow{t}$ is also $ \exp(Lt)$-Lipschitz.
    Bounding the (matrix) operator norm and combining with equation~\eqref{c0:det}, we get
    \begin{equation}
        \left|\det \nabla \flow{-t}(\vz_t)\right|\leq\exp(Ldt),
    \end{equation}
    and together with the previous calculation
    \begin{align}
        \norm{\kot f}_\fss{}^2&\leq\frac{\beta}{\alpha}\exp(Ld t)\int_{\flow{t}(\gX)}f\left(\vz_t\right)^2\rho(\vz_t)\dm{\lambda(\vz_t)}\\
        &\leq\frac{\beta}{\alpha}\exp(Ld t)\int_{\gX}f\left(\vz_t\right)^2\dm{\mu(\vz_t)}
        = \frac{\beta}{\alpha}\exp(Ld t)\norm{f}_\fss{}^2.
    \end{align}
    Since $f$ was arbitrary, taking the square root gives the exponential bound $\norm{\kot}_{\gB(\fss{})}\leq\sqrt{\frac{\beta}{\alpha}} e^{\omega t}$, where $\omega=\frac{d\cdot L}{2}$ as claimed.\\
    Next, we show that the stated assumptions on $\mu$ and $F$ lead to a $C_0$-Semigroup of Koopman operators. The identity and composition properties follow immediately from the definition of $\kot$. Thus, we only need to show that it is continuous in $t$ in the strong operator topology induced by the $\fss{}$-norm, that is
    \begin{equation}
        \forall f\in\fss{}: \lim_{t\downarrow0} \norm{\kot f - f}_\fss{}=0, \;\text{ i.e. }\; \forall f\in\fss{},\forall \epsilon>0, \exists t>0 : \norm{\kot f - f}_\fss{}<\epsilon.
    \end{equation}
    Let $f\in\fss{}$ and $\epsilon>0$ be fixed. From the above we have that $\norm{\kot}_{\gB(\fss{})}<Me^{\omega t}$ where $\omega\in\R$ and w.l.o.g. $M\geq1$.
    Further, because $C^0(\gX)$ is dense in $\fss{}$ w.r.t. the $L^2$-norm and by the assumptions on the measure also w.r.t. the $\fss{}$-norm, we can choose $g\in C^0$ s.t. $\norm{f-g}_\fss{}<\frac{\epsilon}{3M(1+\epsilon)}<\frac{\epsilon}{3}$.
    We then estimate
    \begin{align}
        \norm{\kot f - f}_\fss{}&\leq \norm{\kot (f - g)}_\fss{}+\norm{\kot g-g}_\fss{} + \norm{g-f}_\fss{}\\
        &\leq\norm{\kot}_{\gB(\fss{})}\frac{\epsilon}{3M(1+\epsilon)}+\norm{g\circ\flow{t}-g}_\fss{} + \frac{\epsilon}{3}.
    \end{align}
    Next, because of the boundedness of $\kot$, we can choose a $t_0$ s.t. for $t_0'\leq t_0$ $\norm{\ko{t_0'}}_{\gB(\fss{})}<M(1+\epsilon)$. Further, we estimate 
    \begin{equation}
        \norm{g\circ \flow{t}- g}_\fss{}\leq \sqrt{\mu(\gX)}\norm{g\circ \flow{t}- g}_\infty=\max_{z\in\gX}\abs{g(\flow{t}(z))- g(z)},
    \end{equation}
    where the supremum is taken as a maximum because $\abs{g(\flow{t}(z))- g(z)}$ is continuous in $z$ as the composition of continuous functions (as shown above) and $\gX$ is compact. Letting now $z^*=\argmax_{z\in\gX}\abs{g(\flow{t}(z))- g(z)}$, by continuity $\exists\delta>0$ such that 
    \begin{equation}
        \norm{\flow{t}(z^*)-z^*}<\delta\implies \abs{g(\flow{t}(z^*))-g(z^*)}<\frac{\epsilon}{3}.
    \end{equation}
    Further, because $\flow{t}(z^*)$ is continuous in $t$, we can choose such a $t_1$ for which for all $t_1'\leq t_1$ $\norm{\flow{t_1'}(z^*)-z^*}<\delta$.
    Finally, setting $t=\min(t_0,t_1)$ yields
    \begin{align}
        \norm{\kot f - f}_\fss{}
        &<M(1+\epsilon)\frac{\epsilon}{3M(1+\epsilon)}+\frac{\epsilon}{3} + \frac{\epsilon}{3}\\
        &= \frac{\epsilon}{3}+\frac{\epsilon}{3}+\frac{\epsilon}{3}=\epsilon,
    \end{align}
    completing the proof.
\end{proof}
\begin{remark}
    Two special cases of the above are worth noting. Firstly, for finite time intervals of interest, we get a constant bound on the operator norm. Secondly, if $\mu=\frac{1}{\lambda(\gX)}\lambda$ is the normalised Lebesgue measure, then the operator fulfils the stronger bound $\norm{\kot}_{\gB(\fss{})}\leq e^{\omega t}$.
\end{remark}

Next, we also need the following extension of existing results on the approximation capabilities of SWIM-sampled neural networks from \citet{bolager2023sampling}. 
\begin{lemma}[Universal approximation of Sobolev functions]
\label{lemm:swim_sob}
    Let $g\in\sob$. Then, for any $\epsilon>0$, there exists a SWIM-sampled network $\dswim{}$ with ReLU activation function $\oswim{}$, such that
    \begin{equation}
        \norm{g-\Phi}_\sobs<\epsilon.
    \end{equation}
\end{lemma}
\begin{proof}
    By the Meyers-Serrin Theorem \citep{meyers-serrin}, we can pick a $f\in C^1(\gX^\circ)$ such that $\norm{f-g}_{\sobs}<\frac{\epsilon}{3}$. Further, by \citet[Theorem 1]{sob_nets}, there exists a shallow network $\tilde{\dswim{}}$ with $N$ neurons and ReLU activation function with $\norm{f-\tilde{\dswim{}}}_{\sobs}<\frac{\epsilon}{3}$. Thus, the only non-trivial step is extending the universal approximation of ReLU networks from \citet[Lemma 3]{bolager2023sampling} to the $\sobs$ norm. We again show this for unit-sampled networks before extending the result to SWIM by invoking \citet[Lemma 4]{bolager2023sampling}. In the setting of \citet[Lemma 3]{bolager2023sampling}, we can construct a network $\hat{\dswim{}}$ with $N_1\leq6N$ neurons and biases of the form $\hat{b}_{1,i}=\langle\hat{\vw}_{1,i},\vx_i\rangle$ with $\vx_i\in\gX$ such that $\hat{\dswim{}}=\tilde{\dswim{}}$ on $\gX$. Next, we construct the $N_1$ weights and biases of a unit sampled network $\dswim{}^U$ exactly as in the proof of the Lemma and investigate the per-neuron error on the weak gradients. Let
    \begin{align}
        \hat{n}_i(\vx)&=\hat{w}_{2,i}\phi(\langle \hat{\vw}_{1,i},\vx-\vx_i\rangle),\\
        n_i(\vx)&=\hat{w}_{2,i}\phi(\langle \hat{\vw}_{1,i},\vx-\swimx \rangle)
    \end{align} respectively denote $i$-th neuron in $\hat{\Phi}$ and $\dswim{}^U$, where $\swimx\in\gX^\circ$ was chosen such that $\norm{\vx_i-\swimx}<\delta$ with a $\delta$ s.t. by continuity $\left|\phi(\langle \hat{\vw}_{1,i},\vx-\swimx\rangle)-\phi(\langle \hat{\vw}_{1,i},\vx-\vx_i\rangle)\right|<\frac{\epsilon}{6 N_1 w_2\mu(\gX)},$ $\forall \vx\in\gX$ and $w_2=\max_{1\leq i\leq N_1} |\hat{w}_{2,i}|$. Their weak gradients are thus 
    \begin{align}
        \nabla\hat{n}_i(\vx)&=\hat{w}_{2,i}\hat{\vw}_{1,i}\ind{\langle \hat{\vw}_{1,i},\vx-\vx_i\rangle\geq0}{\vx},\\
        \nabla n_i(\vx)&=\hat{w}_{2,i}\hat{\vw}_{1,i}\ind{\langle \hat{\vw}_{1,i},\vx-\swimx\rangle\geq0}{\vx},
    \end{align}
    with error (the dependence on $\delta$ will be shown in what follows) 
    \begin{align}
        \varepsilon^2_i(\delta)&=\norm{\nabla\hat{n}_i(\vx)-\nabla n_i(\vx)}_{L^2}^2\\
        &=\norm{\hat{\vw}_{1,i}}^2\abs{\hat{w}_{2,i}}^2\int_{\gX}\abs{\ind{\langle \hat{\vw}_{1,i},\vx-\vx_i\rangle\geq0}{\vx}-\ind{\langle \hat{\vw}_{1,i},\vx-\swimx\rangle\geq0}{\vx}}^2\dm{\mu(\vx)}\\
        &\leq\norm{\hat{\vw}_{1,i}}^2\abs{\hat{w}_{2,i}}^2\beta\int_{\gX}\abs{\ind{\langle \hat{\vw}_{1,i},\vx-\vx_i\rangle\geq0}{\vx}-\ind{\langle \hat{\vw}_{1,i},\vx-\swimx\rangle\geq0}{\vx}}^2\dm{\lambda(\vx)}.
    \end{align}
    Inspection of the integrand reveals that it evaluates to $1$ exactly on the region enclosed between two parallel hyperplanes and $0$ otherwise, and can thus be written as the indicator function for this region. Concretely, let
    \begin{align}
        \hat{\sH}_i &= \set{\vx\in\gX \;|\; \langle \hat{\vw}_{1,i},\vx-\vx_i\rangle=0 },\\
        \sH_i &= \set{\vx\in\gX \;|\; \langle \hat{\vw}_{1,i},\vx-\swimx\rangle=0 },\\
        \sB_i(\delta)&=\set{\vx\in\gX\;|\;\min\left(\langle \hat{\vw}_{1,i},\vx_i\rangle, \langle \hat{\vw}_{1,i},\swimx\rangle\right)\leq\langle \hat{\vw}_{1,i},\vx\rangle\leq\max\left(\langle \hat{\vw}_{1,i},\vx_i\rangle, \langle \hat{\vw}_{1,i},\swimx\rangle\right) }.
    \end{align}
    Informally, $\sB_i(\delta)$ denotes exactly the region inside $\gX$ between $\hat{\sH}_i$ and $\sH_i$. Furthermore, the perpendicular distance $d(\hat{\sH}_i,\sH_i)$ along the shared normal $\hat{\vw}_{1,i}$ is given by 
    \begin{equation}
        d(\hat{\sH}_i,\sH_i)=\frac{\abs{\langle\hat{\vw}_{1,i},\vx_i-\swimx\rangle}}{\hat{\vw}_{1,i}}\leq\norm{\vx_i-\swimx}\leq\delta.
    \end{equation}
    By the properties of the Lebesgue measure, we thus have $\lambda(\sB_i(\delta))<\delta \lambda^{\dz-1}\left(\sB_i(\delta)\cap \set{\hat{\vw}_{1,i}}^\perp\right)$, i.e. $\lambda(\sB_i(\delta))<\delta C_i$ with $C_i$ finite by compactness of $\gX$ and independent of $\delta$. After these considerations, we can now bound
    \begin{align}
        \varepsilon^2_i(\delta)&\leq\norm{\hat{\vw}_{1,i}}^2\abs{\hat{w}_{2,i}}^2\beta\int_{\gX}\inds{\sB_i(\delta)}{\vx}\dm{\lambda(\vx)} \\
        &=\norm{\hat{\vw}_{1,i}}^2\abs{\hat{w}_{2,i}}^2\beta\lambda(\sB_i(\delta))\\
        &<\norm{\hat{\vw}_{1,i}}^2\abs{\hat{w}_{2,i}}^2\beta\delta C_i,
    \end{align}
    and respectively
    \begin{equation}
        \varepsilon_i(\delta)\leq w_1w_2\sqrt{\beta C \delta}.
    \end{equation}
    Defining $C=\max_{1\leq i\leq N_1} C_i$ and $w_1=\max_{1\leq i\leq N_1} \norm{\hat{\vw}_{1,i}}$, we now choose
    \begin{equation}
        \delta'=\min\left(\delta, \frac{\epsilon^2}{36N_1^2w_1^2w_2^2C\beta}\right),   
    \end{equation}
    giving
    \begin{equation}
        \sqrt{\delta'}<\frac{\epsilon}{6N_1w_1w_2\sqrt{\beta C}},
    \end{equation}
    and further
    \begin{align}
        \norm{\tilde{\Phi}-\Phi^U}_\sobs&=\norm{\hat{\Phi}-\Phi^U}_\sobs
        =\norm{\sum_{i=1}^{N_1}\hat{n}_i-\sum_{i=1}^{N_1}n_i}_\sobs\\
        &\leq\sum_{i=1}^{N_1}\norm{\hat{n}_i-n_i}_\sobs=\sum_{i=1}^{N_1}\norm{\hat{n}_i-n_i}_{L^2}+\norm{\nabla\hat{n}_i-\nabla n_i}_{L^2}\\ 
        &=\sum_{i=1}^{N_1}\abs{\hat{w}_{2,i}}\sqrt{\int_\gX\abs{\phi(\langle \hat{\vw}_{1,i},\vx-\swimx\rangle)-\phi(\langle \hat{\vw}_{1,i},\vx-\vx_i\rangle)}^2\dm{\mu(\vx)}}+\varepsilon_i(\delta')\\
        &<\sum_{i=1}^{N_1}w_2\frac{\epsilon}{6N_1w_2\mu(\gX)}\mu(\gX)+w_1w_2\sqrt{C\beta\delta'}<\sum_{i=1}^{N_1}\frac{\epsilon}{6N_1}+\frac{\epsilon}{6N_1}=\frac{\epsilon}{3}.
    \end{align}
    Finally, by \citet[Lemma 4]{bolager2023sampling}, we choose a SWIM sampled $\Phi$ s.t. $\Phi^U=\Phi$ on $\gX$ and thus $\norm{\tilde{\Phi}-\Phi}_\sobs<\frac{\epsilon}{3}$. With this, we get
    \begin{equation}
         \norm{g-\Phi}_\sobs\leq\norm{g-f}_\sobs+\norm{f-\tilde{\Phi}}_\sobs+\norm{\tilde{\Phi}-\Phi}_\sobs<\epsilon,
    \end{equation}
    completing the proof.
\end{proof}

Now we are ready to show that our SWIM-sampled dictionary fulfils the above assumptions, meaning that we may invoke the results from \citet{llamazares2024data}, which we summarise by the following lemma.
\begin{lemma}[Validity of SWIM-sampled Dictionary]
    \label{lemm:ass}
    Let $\kot$ be a $C_0$ semigroup, $\mu$ be as in \assref{measure} with the data being i.i.d. samples from $\mu$. Then, there exists a sequence of $(N,M_N)$ pairs such that for $M_N'\geq M$ almost surely $\agen\projd{N}\rightarrow\tgen$ pointwise in $\dom{}$ when using a dictionary of SWIM-sampled ReLU neurons.
\end{lemma}
\begin{proof}
    Since obviously $\mu\ll\lambda$, we have (see \cite{llamazares2024data} for details) that $\dom{}=\sob$. For
    \begin{equation}
        \oswim{i}(\vx)=\max(0,\vw_i^\top \vx-b_i),
    \end{equation} we have the representation in the sense of weak derivatives
    \begin{equation}
        \label{eq:lphi}
        \tgen \oswim{i} (\vx) = \langle \nabla \oswim{i}(\vx), F(\vx)\rangle=\ind{\langle\vw_i, \vx\rangle\geq b_i}{\vx}\sum_{k=1}^\dz \vw_{ik}F_k(\vx).
    \end{equation}
    Regarding \assref{koop1a},
        given enough data, SWIM can produce neurons with weight vectors of arbitrary orientation and arbitrary biases inside the domain. Therefore, for any $N$, SWIM can produce a dictionary of $N$ linearly independent functions using $M\leq 2N$ data points.
    Furthermore, all $\oswim{i}$ are in $\dom{}$ as they are continuous and thus bounded and therefore square integrable on the compact set $\gX$, while their weak derivatives are bounded and therefore again square integrable. Regarding \assref{koop1b}, $\oswim{i}$ are continuous everywhere on $\gX$, while $\tgen \oswim{i}$ are only discontinous on the $\lambda$-null set $\set{\vx\in\gX\;|\; \langle\vw_i, \vx\rangle= b_i}$ which by absolute continuity is also a $\mu$-null set. 

    Regarding \assref{koop2}, $\oswim{i}^2\in \fss{}$ again follows from continuity. Further, again by continuity of $\oswim{}$ and boundedness of $\inds{}{\vx}$, we have that
    \begin{equation}
        \oswim{j}(\vx)\tgen \oswim{i} (\vx) = \oswim{j}(\vx)\ind{\langle\vw_i, \vx\rangle\geq b_i}{\vx}\sum_{k=1}^\dz \vw_{ik}F_k(\vx)
    \end{equation}
    is square integrable by being the product of a bounded and a square integrable function.

    Regarding \assref{koop3} and \assref{koop4}, the convergence requirements follow immediately from the density of SWIM neural networks in $\fss{}$ \uat\  and $\sobs$ (\lemref{swim_sob}). The requirement that $\tgen$ is a closed operator is fulfilled by assumption.

    The claim now follows from Theorem 5.1 in \citet{llamazares2024data}.
\end{proof}
\begin{remark}
    This result and thus all the following naturally extend to other dictionaries, training methods or activation functions as long as they fulfil an analogue of \lemref{swim_sob}.
\end{remark}

\subsection{Convergence Results}
\label{app:theory_conv}
We are now ready to show our main results, namely -- informally -- that, the predicted vector field $\hat{F}$ converges to the true vector field $F$ and the resulting trajectories $\hat{z}$ to the true trajectories $z$.

\renewcommand{\theoremtext}{
Let the assumptions be as in \lemref{ass}. Then $\forall \epsilon>0$, there are constants $N,M_N\in\sN$, a SWIM sampled hidden layer $\dswim{N}$ and outer weights $\vc^i$ such that for any $M_N'\geq M_N$ almost surely
    \begin{equation}
        \norm{\hat{F_i}-F_i}_{\fss{}}=\norm{\sum_{j=1}^Nc^i_j\agen \oswim{j}-F_i}_{\fss{}}<\epsilon,\quad \forall i\in\{1,2,\dots,d\}.
    \end{equation}
}

\VectorField*
\begin{proof}
    \begin{align}
        \norm{\hat{F_i}-F_i}_{\fss{}}&=\norm{\sum_{j=1}^Nc^i_j\agen \oswim{j}-F_i}_{\fss{}}\\
        &= \norm{\agen\sum_{j=1}^Nc^i_j \oswim{j}-\tgen \coord{i}}_{\fss{}}\\
        &\leq\norm{\agen\sum_{j=1}^Nc^i_j \oswim{j}-\agen\projd{N} \coord{i}}_{\fss{}}
        + 
        \norm{\agen\projd{N} \coord{i}-\tgen \coord{i}}_{\fss{}}\\
        &<\epsilon,
    \end{align}
    where in the last step the first term is identically zero and the second term can be controlled by \lemref{ass} because the coordinate functions $\coord{}$ and dictionary functions $\oswim{}$ are in $\dom{}$ if the components of $F$ are in $L^2$ as per assumption.
\end{proof}

For the convergence of the trajectories, we need the approximate semigroups to be exponentially bounded so that they can not arbitrarily amplify errors in finite time. The following example shows that this property is, in general, not preserved by Galerkin methods such as g-EDMD even when the true semigroup is, and the generators converge.

\renewcommand{\theoremtext}{
There are systems for which the true semigroup $\kot$ is exponentially bounded and the sequence of g-EDMD approximants $\agen$ converges pointwise on $\dom{}$ to $\tgen$, while $\nexists W,\omega\in\R$ such that
    \begin{equation}
        \norm{\exp{\left(\agen t\right)}}_{\gB(\fss{N})}\leq We^{\omega t},\quad\forall N\in\sN,
    \end{equation}
    i.e. the sequence of approximate semigroups can not be bounded exponentially by uniform constants.
}
\Counterexample*
\begin{proof}
    We construct a system with a Lipschitz vector field and a sequence of dictionaries in the setting of \lemref{c0} fulfilling Assumptions \ref{ass:koop1}-\ref{ass:measure}, such that we know from \citet{llamazares2024data} and \lemref{c0} that the stated convergence and boundedness are fulfilled while showing that the matrix representing $\agen$ has an eigenvalue lower bounded by $N$ for each $N\in\sN$, preventing the generated semigroups from sharing a uniform growth-bound $\omega$.
    Consider the system
    \begin{align}
        \gX = [-1,1],\quad\dot{x}=-x
    \end{align}
    with the observables $\fss{}=\fso{\gX}$, where $\mu$ has the density $\rho(x)=1+\sqrt{\abs{x-\frac{1}{2}}}$ with respect to the Lebesgue measure $\lambda$. Clearly, we are in the setting of \lemref{c0}, so the Koopman operators form an exponentially bounded $C_0$-Semigroup. The associated Koopman generator is of the form
    \begin{equation}
        \tgen f=-x \frac{\dm{f}}{\dm{x}}.
    \end{equation}
    We will now show that a valid but poorly chosen dictionary can cause the stiffness term $\ip{\tgen u, u}$ and thus an eigenvalue of the approximations $\kot_N=\exp(\ptgen t)$ to grow unboundedly.
    Away from the singularity at $x_0=\frac{1}{2}$, the derivative of $\rho$ is continuously differentiable and given by
    \begin{equation}
        \rho'(x)=\frac{\text{sign}(x-x_0)}{2\sqrt{\abs{x-x_0}}},
    \end{equation}
    however, $\rho'(x)\rightarrow\infty$ as $x\downarrow x_0$. We set up the sequence of dictionaries as $\Psi_N=\set{e_1,\ldots,e_{N-1},e_N}$, where $\set{e_1,\ldots,e_{N-1}}$ are a system of orthonormal polynomials fulfilling the density requirements in $\fss{}$ and $\dom{}$, while we choose $e_N$ to be the orthonormalisation of a standard bump function moving towards the singularity of $\rho'$ as $N$ increases. Concretely, with $\varphi\in\mathcal{C}_C^\infty(-1,1)$ being a standard bump, let
    \begin{equation}
        \varphi_N(x)=\varphi\left(\frac{x-x_N}{\delta_N}\right), \quad \psi_N=\varphi_N/\norm{\varphi_N}_\fss{},
    \end{equation} and set
    \begin{equation}
        e_N=\frac{1}{z_N}\left(\psi_N-\sum_{k=1}^{N-1}\ip{e_i,\psi_N}_\fss{}e_i\right):=\frac{1}{z_N}(\psi_N-r_N),
    \end{equation}
    where the normalisation constant $z_N$ is chosen so that $\norm{e_N}_\fss{}=1$, and $x_N>x_0,\delta_N\in(0,\frac{x_n-x_0}{2}]$ will be specified below. It is easy to verify that $\Psi_N$ fulfils all the assumptions for Theorem 5.1 from \citet{llamazares2024data}, meaning that the approximate generators will indeed converge pointwise to the true generator. Next, we observe that $S_N:=\text{supp}(\psi_N)=\text{supp}(\psi_N')=[x_N-\delta_N,x_N+\delta_N]$ and because the functions in $\Psi_N$ are mutually orthonormal, the population generator matrix simplifies to $L_N=C_NG_N^{-1}=C_N$, so we only need to estimate the eigenvalues of $C_N$ in what follows. Now follows a sequence of estimates that will finally serve to bound the eigenvalues by invoking the Gershgorin Theorem. Firstly,
    \begin{align}
        \ip{\tgen \psi_N,\psi_N}_\fss{}&=\int_{S_N}-x\psi'_N(x)\psi_N(x)\rho(x)\dm{x}&&\\
        &=-\frac{1}{2}\int_{S_N}x\frac{\dm{}}{\dm{x}}\left[\psi_N(x)^2\right]\rho(x)\dm{x}&&\\
        &=\frac{1}{2}\int_{S_N}\psi_N(x)^2(\rho(x)+x\rho'(x))\dm{x}&&\\
        &=\frac{1}{2}\int_{S_N}\psi_N(x)^2\underbrace{\left(1+\frac{x\rho'(x)}{\rho(x)}\right)}_{:=g(x)}\dm{\mu}(x)&&\\
        &\geq\frac{1}{2}\min_{x\in S_N}g(x)\int_{S_N}\psi_N(x)^2\dm{\mu}(x)\geq\frac{1}{4}g(x_N):=g_N,&&
    \end{align}
    for $\delta_N$ small enough, where it was used that $\psi_N$ has unit norm and $g$ is continuous on $S_N$. Next, for $k<N$
    \begin{align}
        \ip{\tgen \psi_N,e_k}_\fss{}&=-\int_{S_N}x\psi_N'(x) e_k(x)\rho(x)\; \dm{x}\\
        &=\int_{S_N}\psi_N(x) \left[x e_k(x)\rho(x)\right]'\dm{x}\\
        &=\int_{S_N}\psi_N(x)\left(xe_k'(x)+e_k(x)+e_k(x)\frac{x\rho'(x)}{\rho(x)}\right)\dm{\mu}(x)\\
        &=\int_{S_N}\psi_N(x)\left(xe_k'(x)+e_k(x)g(x)\right)\dm{\mu}(x),\\
        \abs{\ip{\tgen \psi_N,e_k}_\fss{}}&\leq\int_{S_N}\abs{\psi_N(x)}(\abs{xe'_k(x)}+\abs{e_k(x)g(x)})\dm{\mu}(x)\\
        &\leq E_k(1+ \kappa_{N} g_N)\int_{S_N}\abs{\psi_N(x)}\cdot 1\;\dm{\mu}(x)\\
        &\leq E_k(1+\kappa_{N} g_N) \norm{\psi_N}_\fss{}\sqrt{\mu(S_N)}=E_k(1+\kappa_{N} g_N)\sqrt{\mu(S_N)},
    \end{align}
    where we defined $E_k:=\norm{e_k}_{\mathcal{C}^1}=\norm{e_k}_\infty+\norm{e_k'}_\infty$ and $\kappa_N=\sup_{x\in S_N}\frac{g(x)}{g(x_N)}$, which exists by continuity and can be bounded independent of $\delta_N$, because $\delta_N$ is bounded above. Similarly, we have
    \begin{equation}
       \abs{\ip{\tgen e_k, \psi_N}_\fss{}}=\abs{\int_{S_N}xe_k'(x)\psi_N(x)\; \dm{\mu}(x)}\leq E_k\int_{S_N}\abs{\psi_N(x)}\; \dm{\mu}(x)\leq E_k\sqrt{\mu(S_N)},
    \end{equation}
    and
    \begin{align}
    \label{ce:pnek}
        \abs{\ip{\psi_N, e_k}_\fss{}}\leq \int_{S_N}\abs{\psi_N}\abs{e_k}\dm{\mu}\leq E_k\sqrt{\mu(S_N)}.
    \end{align}
    The previous estimates pertain to the behaviour of $\psi_N$, which we now need to adapt for $e_N=\frac{1}{z_N}(\psi_N-r_N)$. From \eqref{ce:pnek} and using orthonormality of the dictionary, we get
    \begin{equation}
        \norm{r_N}_\fss{}^2=\ip{r_N,r_N}_\fss{}=\sum_{k=1}^{N-1}\ip{e_k,\psi_N}_\fss{}^2\leq NE^2\mu(S_N),
    \end{equation}
    with $E:=\max_{k<N}E_k$, which implies that
    \begin{equation}
        -\sqrt{N}E\sqrt{\mu(S_N)}\leq -\norm{r_N}_\fss{}\leq \norm{r_N}_\fss{}\leq\sqrt{N}E\sqrt{\mu(S_N)}.
    \end{equation}
    Using Bessel's and the reverse triangle inequality, we can bound $z_N$ by
    \begin{equation}
        1\geq z_N=\norm{\psi_N-r_N}_\fss{}\geq\abs{\norm{\psi_N}_\fss{}-\norm{r_N}_\fss{}}\geq 1-\sqrt{N}E\sqrt{\mu(S_N)}.
    \end{equation}
    Further observing that $\ip{e_k, r_N}_\fss{}=\ip{e_k,\psi_N}_\fss{}$, we get $\abs{\ip{e_k,r_N}_\fss{}}\leq E_k\sqrt{\mu(S_N}$. Based on the previous estimations and defining $E'=\max_{k,j<N}\abs{\ip{\tgen e_k,e_j}_\fss{}}$, straightforward calculations yield
    \begin{align}
        \abs{\ip{\tgen r_N, e_k}_\fss{}}&=\abs{\sum_{j=1}^{N-1}\ip{e_j,\psi_N}_\fss{}\ip{\tgen e_j,e_k}_\fss{}}\leq N EE'\sqrt{\mu(S_N)},\\
        \abs{\ip{\tgen e_k,r_N }_\fss{}}&=\abs{\sum_{j=1}^{N-1}\ip{e_j,\psi_N}_\fss{}\ip{\tgen e_k,e_j}_\fss{}}\leq N EE'\sqrt{\mu(S_N)},\\
        \abs{\ip{\tgen \psi_N, r_N}_\fss{}}&=\abs{\sum_{j=1}^{N-1}\ip{e_j,\psi_N}_\fss{}\ip{\tgen\psi_N,e_j}_\fss{}}\leq E\sqrt{\mu(S_N)}\abs{\sum_{j=1}^{N-1}\ip{\tgen \psi_N,e_j}_\fss{}}\\&\leq NE^2\sqrt{\mu(S_N)}(1+\kappa_Ng_N),\\
        \abs{\ip{\tgen r_N, \psi_N}_\fss{}}&=\abs{\sum_{j=1}^{N-1}\ip{e_j,\psi_N}_\fss{}\ip{\tgen e_j,\psi_N}_\fss{}}\leq NE^2\sqrt{\mu(S_N)},\\
        \abs{\ip{\tgen r_N,r_N}_\fss{}}&=\abs{\sum_{j=1}^{N-1}\ip{e_j,\psi_N}_\fss{}\ip{\tgen e_j,r_N}_\fss{}}\leq N^2E^2E'\sqrt{\mu(S_N)},
    \end{align}
    which again implies
    \begin{equation}
        -N^2E^2E'\sqrt{\mu(S_N)}\leq \ip{Lr_N,r_N}_\fss{}\leq N^2E^2E'\sqrt{\mu(S_N)}.
    \end{equation}
    We can now use these estimates to bound the entries of the matrix $C_N$ for $k,j<N$ by
    \begin{align}
        \abs{(C_N)_{kj}}&=\abs{\ip{\tgen e_k,e_j}_\fss{}}\leq E',\\
        \abs{(C_N)_{Nk}}&=\abs{\ip{\tgen e_N,e_k}_\fss{}}\leq\frac{1}{z_N}\left(\abs{\ip{\tgen \psi_N, e_k}_\fss{}}+\abs{\ip{\tgen r_N, e_k}_\fss{}}\right)\\
        &\leq\frac{1}{z_N}\left(E(1+\kappa_Ng_N)\sqrt{\mu(S_N)}+NE'E\sqrt{\mu(S_N)}\right)\\
        &=\frac{1}{z_N}E\sqrt{\mu(S_N)}\left(1+\kappa_Ng_N+NE'\right),\\
        \abs{(C_N)_{kN}}&=\abs{\ip{\tgen e_k,e_N}_\fss{}}\leq\frac{1}{z_N}\left(\abs{\ip{\tgen e_k, \psi_N}_\fss{}}+\abs{\ip{\tgen e_k, r_N}_\fss{}}\right)\\
        &\leq\frac{1}{z_N}E\sqrt{\mu(S_N)}\left(1+NE'\right),\\
        (C_N)_{NN}&=\frac{1}{z_N^2}\ip{\tgen\psi_N-\tgen r_N,\psi_N-r_N}_\fss{}\\
        &=\frac{1}{z_N^2}\left(
        \ip{\tgen\psi_N,\psi_N}_\fss{}
        -\ip{\tgen\psi_N,r_N}_\fss{}
        -\ip{\tgen r_N,\psi_N}_\fss{}
        +\ip{\tgen r_N,r_N}_\fss{}
        \right)\\ 
        &\geq\frac{1}{z_N^2}\left(
        g_N-NE^2\mu(S_N)(1+\kappa_Ng_N)-NE^2\mu(S_N)-N^2E^2E'\sqrt{\mu(S_N)}
        \right).
    \end{align}
    The concluding step will be to invoke the Gershgorin Circle Theorem to show the existence of an eigenvalue with growing real part by choosing a sequence of dictionaries parameterised by $x_N,\delta_N$ such that one Gershgorin Circle is disjoint from the others and "moves towards infinity". Concretely, consider the Gershgorin Disks resulting from row estimates
    \begin{align}
        \sD_i&=\sK((C_N)_{ii},R_i):=\set{z\in\sC\;|\;\abs{z-(C_N)_{ii}}\leq R_i},\\
        R_i&=\sum_{j\neq i}\abs{(C_N)_{ij}}.
    \end{align}
    For $k<N$, we get the radii
    \begin{align}
        R_k&\leq (N-2)E'+\frac{1}{z_N}E\sqrt{\mu(S_N)}(1+NE'),\\
        &\leq R^*:=R_k+2E'=NE'+\frac{1}{z_N}E\sqrt{\mu(S_N)}(1+NE'),
    \end{align}
    which, together with the bound $(C_N)_{kk}\leq E'$ gives
    \begin{align}
        \sD_k\subseteq \sD^*:=\sK(E',R^*),\quad k=1,\ldots,N-1, 
    \end{align}
    since the centres can be at most $2E'$ apart. Analogously,
    \begin{equation}
        R_N=\sum_{k=1}^{N-1}\abs{(C_N)_{Nk}}\leq(N-1)\frac{1}{z_N}E\sqrt{\mu(S_N)}\left(1+\kappa_Ng_N+NE'\right),
    \end{equation}
    which together with the estimate for $(C_N)_{NN}$ defines the remaining disk. Finally, to enforce the sought behaviour for the remaining disk and thus the last eigenvalue, we can, because of the singularity of $\rho'(x)$ and respectively $g(x)$, choose $x_N$ such that
    \begin{align}
        g_N\geq N+E'+NE'+1,
    \end{align}
    and afterwards $\delta_N$ such that
    \begin{align}
        & S_N\subsetneq [-1,1],\\
        &\land\frac{1}{2}\min_{x\in S_N} g(x)\geq\frac{1}{4}g(x_N),\\
        &\land-\sqrt{N}E\sqrt{\mu(S_N)}\geq-\frac{1}{2},\\
        &\land2E\sqrt{\mu(S_N)}(1+NE')\leq\frac{1}{2},\\
        &\land -4\left(NE^2\mu(S_N)(1+\kappa_Ng_N)-NE^2\mu(S_N)-N^2E^2E'\sqrt{\mu(S_N)}\right)\geq-\frac{1}{2},\\
        &\land 2(N-1)E\sqrt{\mu(S_N)}(1+\kappa_N g_N+NE')\leq\frac{1}{2}.
    \end{align}
    We do not spell out the details here, as it is immediately apparent that all constraints on $\delta_N$ can be fulfilled by shrinking $\delta_N$ and thus $\mu(S_N)$, which is always possible as we never required a positive lower bound on $\delta_N$. Observing that now
    \begin{equation}
        1\geq z_N\geq \frac{1}{2} \implies 1\leq \frac{1}{z_N}\leq2,
    \end{equation}
    and $R^*\leq NE'+1$, while
    \begin{equation}
        (C_N)_{NN}\geq g_N-\frac{1}{2}\geq N+E'+R^*+\frac{1}{2}\geq N+ E'+R^*+R_N,
    \end{equation}
    we have $\sD_N\cap\sD^*=\emptyset$ for $N\geq1$ and further by the Gershgorin Circle Theorem,
    \begin{equation}
        Re(\lambda_N)\geq\min_{x\in\sD_N}Re(x)\geq N.
    \end{equation}
    The claim now immediately follows from
    \begin{equation}
        \norm{\exp{(\ptgen t)}}_{\gB(\fss{N})}\geq e^{\lambda_N t}.
    \end{equation}
\end{proof}
\begin{remark}
    The example used a single bump function that moved as the dictionary size increased. This was merely used to simplify the presentation; however, the underlying effect of a dictionary function concentrating on a density pathology naturally occurs in the same way when refining a bump-like basis discretisation, such as Radial Basis Functions or Nodal Bases.
\end{remark}

As \thmref{counterexample} shows, the root cause for losing exponential boundedness is a pathology of the density $\rho$. We conjecture that a large class of well-behaved densities could be shown to preserve exponential boundedness, but leave a thorough characterisation of such "safe" densities for future work. In what follows, we will thus restrict ourselves to a special form of exponential boundedness, namely quasi-contractivity, which corresponds to a \qde\ generator. The following Lemma shows that g-EDMD in fact preserves this property. As per \lemref{c0}, a sufficient condition for quasi-contractivity of the semigroup is restricting $\mu$ to the Lebesgue measure, i.e. $\alpha=\beta=\frac{1}{\lambda(\gX)}$.\\
\begin{lemma}[g-EDMD preserves \qdy]
    \label{lemm:qd_error}
    Let assumptions \ref{ass:koop1} and \ref{ass:koop2} hold and assume the true Koopman operators to form a quasi-contractive semigroup, i.e. $\norm{\kot}_{\gB(\fss{})}\leq e^{\omega t}$ for some $\omega\in\R$ and all $t\in[0, T]$. Then $\forall  \epsilon>0$, there are constants $N,M_N\in\sN$ such that for any $M_N'\geq M_N$ the g-EDMD approximation $\agen$ is almost surely \qde\ with $\omega' = \omega+\epsilon$. In particular, there exists a sequence $\set{(N,M_N)}_{N=1}^\infty$ such that the sequence of approximations is almost surely \qde\ with a uniform \qdy-constant $\omega'$.
\end{lemma}
\begin{proof}
    By the Lumer-Phillips Theorem \lp, $\tgen$ is \qde, i.e. \begin{equation}
        \text{Re}\ip{ \tgen f,f}_\fss{}\leq \omega\left\langle f,f\right\rangle_\fss{}, \quad \forall f\in\dom{}.
    \end{equation}
    Because $\fss{N}\subset \dom{}$ and projectors onto finite subspaces are self-adjoint
    \begin{equation}
        \langle\ptgen f, f\rangle_\fss{} = \langle\projf{N} \tgen f, f\rangle_\fss{}=\langle\tgen f, \projf{N} f\rangle_\fss{}=\langle\tgen f, f\rangle_\fss{}, \quad\forall f\in\fss{N},
    \end{equation}
    which yields that $\ptgen$ is also \qde\ with the same $\omega$ on $\fss{N}$. Further, by \citet{llamazares2024data}[Corollary 3.7], we can for every $N$ find an $M_N$ s.t. for any $M_N'\geq M_N$, $\norm{\agen-\ptgen}<\epsilon$ almost surely. Finally, we have
    \begin{align}
        \left\langle\agen f, f\right\rangle_\fss{} &= \left\langle\left(\agen+\ptgen-\ptgen\right) f, f\right\rangle_\fss{} \\
        &=\left\langle\ptgen f, f\right\rangle_\fss{} + \left\langle\left(\agen-\ptgen\right) f, f\right\rangle_\fss{}\\
        &\leq \omega \norm{f}_\fss{}^2 + \norm{\left(\agen-\ptgen\right)f}_\fss{}\norm{f}_\fss{}\\
        &\leq (\omega+\epsilon) \norm{f}_\fss{}^2=\omega' \langle f,f\rangle_\fss{},
    \end{align}
    as claimed.
\end{proof}

We are now ready to show our final convergence result about the generated semigroups and, by extension, the predicted trajectories.

\renewcommand{\theoremtext}{
Let the assumptions be as in \lemref{ass} and assume the true Koopman operators to form a quasi-contractive semigroup, i.e. $\norm{\kot}_{\gB(\fss{})}\leq e^{\omega t}$ for some $\omega_*\in\R$ and all $t\in[0, T]$. Then $\forall  \epsilon>0$ and $t\in[0,T]$, there are constants $N,M_N\in\sN$, a SWIM sampled hidden layer $\dswim{N}$ and outer weights $\vc^i$ such that for any $M_N'\geq M_N$ almost surely
    \begin{equation}
        \norm{
            \exp\left(\tgen t\right)\coord{i} - \sum_{j=1}^N c^i_j \exp\left(\pgen t\right) \oswim{j}
        }_{\fss{}}<\epsilon,\quad \forall i\in\{1,2,\dots,d\}.
    \end{equation}
}
\Trajectories*
\begin{proof}
    The main idea is to apply the First Trotter-Kato Approximation Theorem \tk, which guarantees the convergence of the Koopman Operator Semigroup given Exponential Boundedness and the convergence of the generators. Since we already showed the convergence of the generators in our setting, the main work of the proof is showing boundedness, which will follow from \qdy\ via the Lumer-Phillips Theorem \lp.\\
    Let $g\in\dom{}$ be an observable.\\
    \begin{enumerate}
        \item We start by separating the convergence of the projected observable and the convergence of the projected generator
        \begin{align}
            &\norm{
            e^{\tgen t}g - \sum_{j=1}^N c^i_j e^{\pgen t} \oswim{j}
        }_{\fss{}}\\
        =&\norm{
            e^{\tgen t}g - e^{\pgen t}\projd{N}g
        }_{\fss{}}\\
        \leq&\norm{
            e^{\tgen t}g - e^{\tgen t}\projd{N}g
        }_{\fss{}} + 
        \norm{
            \left(e^{\tgen t} -e^{\pgen t}\right)\projd{N}g
        }_{\fss{}}\\
        \leq& \norm{\kot}_{\gB(\fss{})}\norm{g-\projd{N}g}_{\fss{}}+
        \norm{
            \left(e^{\tgen t} -e^{\pgen t}\right)\projd{N}g
        }_{\fss{}}
        \end{align}
        Since the true operator is bounded by assumption, the first term converges to zero by the universality of the dictionary \uat. The second term is exactly the pointwise convergence of the operator, which we show in the following steps.
        \item Next, we will show that $\pgen$ generates a quasi-contraction (and thus exponentially bounded) semigroup. By \lemref{qd_error}, we can choose $M_N$ so that $\pgen$ is almost surely \qde\ with $\omega'=\omega+\epsilon$. By being finite-dimensional, it is densely defined on $\mathcal{F}_N$ and bounded. Furthermore, with $\lambda>\omega'$
        \begin{align}
            &(\pgen-\lambda\id)f=0 \\
            \iff&(\pgen-\omega'\id)f=(\lambda-\omega')f\\
            \implies&\underbrace{\langle(\pgen-\omega'\id)f,f\rangle_\fss{}}_{\leq0}=\underbrace{(\lambda-\omega')}_{>0}\underbrace{\langle f,f\rangle_\fss{}}_{\geq0}\\
            \implies &f=0.
        \end{align}
         Thus $\pgen$ is injective and, by being finite dimensional, also surjective. With this, we can invoke the Lumer-Phillips Theorem \lp\ for quasicontraction semigroups, yielding that $\norm{e^{\pgen t}}_{\gB(\fss{})}\leq e^{\omega' t}$ for all $N$.  
        \item Finally, since we have exponential stability for the semigroup generated by $\pgen$, and pointwise convergence of $\pgen$ on $\dom{}$ by \lemref{ass}, we are in the setting of the First Trotter-Kato Approximation Theorem \tk. Specifically, we apply the implication $(\text{b})\implies\text{(d)}$ with $D=\dom{}$ and $x_n=\projd{n}f$ in \citet{tk_theorem}, yielding the sought-after convergence of the Koopman Operators, completing the proof. The number of data points $M_N$ is simply chosen as the larger of the requirements imposed by \lemref{ass} and \lemref{qd_error}.
    \end{enumerate}
\end{proof}
\begin{remark}
    It should be emphasised that the proof made no reference to specific properties of sampled networks, making this a general result about g-EDMD rather than our specific dictionary.
\end{remark}
\section{Training NODEs with adjoint sensitivity method, Exploding and Vanishing Gradients Problem}\label{sec:adjoing_ev_proof}

Consider a Neural ODE defined as 
\begin{align}
  \dot{\vz}(t) = \vf(\vz(t), t, \boldsymbol{\theta}),   
\end{align}
where $\vz(t) \in \mathbb{R}^d$ is the state of the network at time $t$ and $\boldsymbol{\theta}$ denote the set of system's parameters. 

The main difficulty in training continuous depth/time networks is reverse-mode differentiation through the ODE solver. Treating the ODE solver as a black box, gradients are computed via the \textit{adjoint sensitivity method}. This approach computes gradients by solving a second (augmented) ODE backwards in time. It works with any ODE solver, scales linearly, has low memory cost, and explicitly controls numerical error \citep{chen2018neural}.\\

 Consider optimizing a scalar loss $\mathcal{J}(\cdot)$ applied to the ODE solver output
\begin{align}\label{eq:loss}
\mathcal{J}(\vz(t_N)) = \mathcal{J} \left( \vz(t_0) + \int_{t_0}^{t_N} f( \vz(t), t, \theta) \, dt \right) = \mathcal{J}(\texttt{ODESolve}( \vz(t_0), f, t_0, t_N, \theta)).
\end{align}
For $\theta \in \boldsymbol{\theta}$, we define an augmented adjoint state $\va_{aug}(t) = [\va(t), \va_\theta(t)]$ with
the adjoint state
\begin{align}
  \va(t) = \dfrac{d \mathcal{J}}{d \vz(t)}, 
\end{align} 
and the parameter sensitivity or parameter adjoint
\begin{align}
\va_\theta(t) := \dfrac{d \mathcal{J}}{d\theta}, 
\end{align} 
which respectively represent the sensitivity of the loss with respect to the state trajectory and the model parameters. Their dynamics are governed by the following ODEs
  \begin{align}\nonumber
    \frac{d \va(t)}{dt} & \, = \, -\va(t)\tran \frac{d \vf(\vz(t), t, \theta)}{d \vz(t)},
    \\[1.5ex]\label{}
    \frac{d \va_\theta(t)}{dt} & \, = \, - \va(t)\tran \frac{d \vf(\vz(t), t, \theta)}{d\theta}.
   \end{align}
These equations are integrated backward in time from the terminal time $t_N$ to $t_0$. Specifically
  \begin{align}\nonumber
    \va(t_0) &\, = \, \va(t_N) + \int_{t_N}^{t_0} \frac{d \va(t)}{dt} dt
    = \va(t_N) - \int_{t_N}^{t_0} \va(t)\tran \frac{d \vf(\vz(t), t, \theta)}{d \vz(t)} dt,
\\[1ex]\label{}   
    \va_{\theta}(t_0) & \, = \, \va_{\theta}(t_N) + \int_{t_N}^{t_0} \frac{d \va(t)}{dt} dt
     = - \int_{t_N}^{t_0} \va(t)\tran \frac{d \vf( \vz(t), t, \theta)}{d \theta} dt. \hspace{1cm} (\va_{\theta}(t_N) =0)
\end{align}  
 For simplicity, we present the results for autonomous Neural ODEs. Similar results can be extended to non-autonomous systems by considering the augmented adjoint state as $\va_{aug}(t) = [\va(t), \va_\theta(t), \va_t(t)]$, replacing the flow map with the corresponding evolution operator (fundamental matrix) and invoking Oseledec's Multiplicative Ergodic Theorem for non-autonomous dynamical systems.
 \begin{definition}
  Consider an ODE system
\begin{align}
  \dot{\vz}(t) = \vf(\vz(t)),  \hspace{.3cm} \vz(t_0)= \vz_0, \hspace{.3cm} \vf \in C^1(E), \hspace{.3cm} E\subseteq \mathbb{R}^d,
\end{align}
with a trajectory $\vz(t)$ as its solution. The time evolution 
of a tangent vector $\delta \vz(t)$ at $\vz(t)$ is given by the linearized variational (tangent) equation
\begin{align}\label{eq:tangent}
 \delta \dot{\vz}(t) = \mA(t)\; \delta \vz(t),   
\end{align}
where 
\begin{align}
 \mA(t) = \frac{\partial\vf}{\partial \vz}(\vz(t)), 
\end{align}
is the Jacobian matrix of $\vf$ along the trajectory $\vz(t)$.
A principal fundamental matrix $\mPhi(t,t_0) \in \mathbb{R}^{d \times d}$ is a matrix-valued solution of
\begin{align}
  \frac{d}{dt} \mPhi(t,t_0) &= \mA(t) \; \mPhi(t,t_0), \quad \mPhi(t_0,t_0) = \mI_d.
\end{align}
Then the general solution satisfies
\begin{align}
  \delta \vz(t) = \mPhi(t,t_0) \, \delta \vz(t_0),
\end{align}
and $\mPhi$ has the semigroup property
\begin{align}
  \mPhi(t,t_1) \, \mPhi(t_1,t_0) = \mPhi(t,t_0) \quad \text{for all } t,t_1,t_0.
\end{align}
Also $\mPhi(t,t_0)$ is invertible and 
\begin{align}
  \mPhi(t,t_0)^{-1} = \mPhi(t_0,t).
\end{align}
 \end{definition}
\begin{definition}\label{def:maxLE}
Let $\mPhi(t,s)$ denote the fundamental matrix of the tangent linear system \eqref{eq:tangent}, propagating tangent vectors along the trajectory from time $s$ to time $t$, where $s<t$. Let
$\; \tau := t-s >0 \;$
denote the propagation time. The largest Finite-Time Lyapunov Exponent (FTLE) over an interval of length $\tau$ is defined by
\begin{align}
  \lambda_{\max}^{FT}(\tau)
  :=
  \frac{1}{\tau}
  \log
  \left\|
  \mPhi(s+\tau,s)
  \right\|,
\end{align}
where $\left\|
  \cdot
  \right\|$ denotes the spectral norm or operator 2-norm. The largest Lyapunov Exponent (LE) along the trajectory is
\begin{align}
  \lambda_{\max}
  :=
  \limsup_{\tau\rightarrow\infty}
  \frac{1}{\tau}
  \log
  \left\|
  \mPhi(s+\tau,s)
  \right\|.
\end{align}
If the limit exists, then
\begin{align}
  \lambda_{\max}
  =
  \lim_{\tau\rightarrow\infty}
  \frac{1}{\tau}
  \log
  \left\|
  \mPhi(s+\tau,s)
  \right\|.
\end{align}
Here $\left\|
  \cdot
  \right\|$ is any matrix norm; the finite-time value depends on the norm, while the asymptotic Lyapunov exponent is norm-independent in finite dimensions. Note that $\lambda_{\max}$ is the supremum over all nonzero initial tangent vectors $\delta\vz_0$ of the vector-specific exponent
\begin{align}
\lambda(\delta\vz_0) = \limsup_{\tau\to\infty} \frac{1}{\tau} \log \frac{\|\mPhi(s+\tau,s)\delta\vz_0\|}{\|\delta\vz_0\|},
\end{align}
i.e. the largest possible asymptotic growth rate of any tangent perturbation, and it coincides with that vector-specific exponent when $\delta\vz_0$ lies in the maximal asymptotic growth direction.
\end{definition} 

We now establish theoretical results that characterize the Exploding and Vanishing Gradient Problem (EVGP) for Neural ODEs trained via the adjoint sensitivity method.
The first theorem gives FTLE bounds for the backward-in-time growth of the state adjoint $\va(t)$ over an interval $[t,T]$.
The second theorem gives the corresponding LE bounds for the state adjoint in the asymptotic regime $T \to \infty$, under Oseledec's multiplicative ergodic theorem.
The third theorem extends both results to the parameter gradient $\nabla_{\boldsymbol{\theta}} \mathcal{J}$, yielding unified FTLE and LE bounds together with the corresponding exploding, vanishing, and bounded regimes.

\begin{theorem}\label{thm:adjoint_state:FTLE}
 Consider the autonomous Neural ODE
\begin{align}
  \dot{\vz}(t) = \vf(\vz(t), \boldsymbol{\theta}),
\end{align}
trained on a time series data $\{\vz(t) : 0 \le t \le T\}$, using the loss function $\mathcal{J}$ defined as \eqref{eq:loss}.
Let $\vf \in C^1$, $\vz(t)$ be a solution on $[0,T]$.
Fix $t \in [0,T)$ and let $\lambda_{\max}^{FT}(t,T)$ denote the largest FTLE over the interval $[t,T]$, i.e. 
\begin{align}
\lambda_{\max}^{FT}(t,T)
  :=
  \frac{1}{T-t}
  \log
  \left\|
  \mPhi(T, t)
  \right\|.
\end{align}
Then the adjoint state $\va(t) = \dfrac{d \mathcal{J}}{d \vz(t)}$ satisfies
\begin{align}
  \| \va(t) \|
  \le
  \| \va_T \| \,
  \exp\!\Big(\lambda_{\max}^{FT}(t,T)\,(T-t)\Big), \qquad \va(T) = \va_T,
\end{align}
for all $t\in[0,T]$. Moreover, assume that $\vu_{\max}(t,T)$ is a unit left singular vector of $\mPhi(T,t)$ corresponding to the largest singular value $\; \sigma_{\max}(\mPhi(T,t)) = \|\mPhi(T,t)\|$. If the terminal adjoint vector is not orthogonal to this maximally expanding direction, i.e.
\begin{align}
\beta(t,T) := \va_T \, \vu_{\max}(t,T) \neq 0,
\end{align}
then the following finite-time lower bound holds:
\begin{align}
  \| \va(t) \|
  \ge
  |\beta(t,T)| \,
  \exp\!\Big(\lambda_{\max}^{FT}(t,T)\,(T-t)\Big).
\end{align}
In particular, if there exists a constant $c_0>0$ such that $\; |\va_T \, \vu_{\max}(t,T)| \ge c_0 \|\va_T\|$,
then
\begin{align}
  \| \va(t) \|
  \ge
  c_0 \|\va_T\| \,
  \exp\!\Big(\lambda_{\max}^{FT}(t,T)\,(T-t)\Big).
\end{align}
Consequently, for the finite interval \([t,T]\):
\begin{enumerate}
    \item If \(\lambda_{\max}^{FT}(t,T) > 0\) and \(\beta(t,T) \neq 0\), then the adjoint norm grows exponentially backward in time, as quantified by the upper and lower bounds above. 
    \item If \(\lambda_{\max}^{FT}(t,T) < 0\), then the upper bound shows that the adjoint norm decays exponentially backward in time for all \(\va_T\).
    \item If \(\lambda_{\max}^{FT}(t,T) = 0\), then the upper bound gives \(\|\va(t)\| \le \|\va_T\|\).
\end{enumerate}
\end{theorem}
\begin{proof}
 Let $\mPhi(t,s)$ be the fundamental matrix for the tangent variational equation
\begin{align}
  \frac{d}{dt} \mPhi(t,s) = \mA(t)\mPhi(t,s),\qquad \mPhi(s,s)=\mI_d,
\end{align}
where $\mA(t) = \dfrac{\partial \vf}{\partial \vz}(\vz(t))$.
For ease of notation, we consider the row-vector adjoint, whose evolution is given by
\begin{align}\label{eq:ad_eq}
  \dot{\va}(t) = -\va(t) \, \mA(t), \qquad \va(T) = \va_T.
\end{align}
Consider the function 
\begin{align}
\va_T \; \mPhi(T,t),
\end{align}
for $t \in [0,T]$, with $T$ fixed. 
By the semigroup property we have 
\begin{align}
\mPhi(T,t) = \mPhi(T,s) \; \mPhi(t,s)^{-1} 
\end{align}
Differentiating with respect to $t$ yields
\begin{align}\label{eq:der}
  \frac{\partial}{\partial t} \mPhi(T,t)
  &= \mPhi(T,s) \, \frac{\partial}{\partial t} \left[ \mPhi(t,s)^{-1} \right]
  \end{align}
Now differentiate the inverse using $\partial_t (M^{-1}) = -M^{-1} (\partial_t M) M^{-1}$ gives
\begin{align}\nonumber
  \frac{\partial}{\partial t} \left[ \mPhi(t,s)^{-1} \right]
  &= -\mPhi(t,s)^{-1} \left( \frac{\partial}{\partial t} \mPhi(t,s) \right) \mPhi(t,s)^{-1} \\\label{eq:inv}
  &= -\mPhi(s,t) \left( \mA(t) \, \mPhi(t,s) \right) \mPhi(t,s)^{-1} = -\mPhi(s,t) \, \mA(t). 
   \end{align}
Hence, from \eqref{eq:der} and \eqref{eq:inv}
\begin{align}
  \frac{\partial}{\partial t} \mPhi(T,t)
  = \mPhi(T,s) \left( -\mPhi(s,t) \, \mA(t) \right) = -\mPhi(T,t) \, \mA(t). 
\end{align}
Therefore
\begin{align}
  \frac{d}{dt} \big( \va_T \, \mPhi(T,t) \big)
  \, = \, \va_T \, \frac{\partial}{\partial t} \mPhi(T,t) 
  \, = \, \va_T \, \big( -\mPhi(T,t) \, \mA(t) \big) 
   \, = \, -\big( \va_T \, \mPhi(T,t) \big) \, \mA(t).
\end{align}
Moreover, evaluating at $t = T$
\begin{align}
  \va_T \; \mPhi(T,T) = \va_T \; \mI_d = \va_T.
\end{align}
Thus, $\; \va_T \; \mPhi(T,t)\;$ satisfies the adjoint differential equation \eqref{eq:ad_eq}.
 By the uniqueness of solutions to linear ODEs, we conclude that
\begin{align}
  \va(t) = \va_T \, \mPhi(T,t) \quad \text{for all } t \in [0,T].
\end{align}
Hence the adjoint state is obtained by propagating the terminal covector backward along the tangent linear flow. Furthermore
\begin{align}
  \| \va(t) \| = \| \va_T \, \mPhi(T,t) \| \leq \| \va_T \| \, \| \mPhi(T,t) \|.
\end{align}
By definition of $\lambda_{\max}^{FT}(t,T)$
\begin{align}
  \|\mPhi(T,t)\|
  =
  \exp\!\Big(\lambda_{\max}^{FT}(t,T)\,(T-t)\Big),
\end{align}
which gives the claimed FTLE-based upper bound:
\begin{align}
  \|\va(t)\|
  \le \|\va_T\|\,
  \exp\!\Big(\lambda_{\max}^{FT}(t,T)\,(T-t)\Big).
\end{align}
For the lower bound, let
\[
\mPhi(T,t) = \mU \mSigma \mV^\top
\]
be a singular value decomposition, and let $\vu_{\max}(t,T)$ be the first column of $\mU$, corresponding to the largest singular value
\[
\sigma_{\max}(\mPhi(T,t)) = \|\mPhi(T,t)\|.
\]
Write the row vector $\va_T$ as
\[
\va_T = \beta(t,T) \, \vu_{\max}(t,T)^\top + \vw^\top,
\]
where $\beta(t,T) = \va_T \cdot \vu_{\max}(t,T)$ and $\vw \perp \vu_{\max}(t,T)$.
Then
\begin{align}
  \va_T \mPhi(T,t)
  &=
  \beta(t,T) \, \vu_{\max}(t,T)^\top \mPhi(T,t)
  +
  \vw^\top \mPhi(T,t).
\end{align}
Since
\[
\vu_{\max}(t,T)^\top \mPhi(T,t)
=
\sigma_{\max}(\mPhi(T,t)) \, \vv_{\max}(t,T)^\top,
\]
for the corresponding right singular vector $\vv_{\max}(t,T)$, and since
\(\vw \perp \vu_{\max}(t,T)\) implies
\[
\vw^\top \mPhi(T,t)
\perp
\vu_{\max}(t,T)^\top \mPhi(T,t),
\]
we obtain, by Pythagorean theorem
\begin{align}
  \|\va_T \mPhi(T,t)\|^2
  &=
  |\beta(t,T)|^2
  \|\vu_{\max}(t,T)^\top \mPhi(T,t)\|^2
  +
  \|\vw^\top \mPhi(T,t)\|^2
  \nonumber\\
  &\ge
  |\beta(t,T)|^2
  \|\vu_{\max}(t,T)^\top \mPhi(T,t)\|^2.
\end{align}
Therefore
\begin{align}
  \|\va_T \mPhi(T,t)\|
  &\ge
  |\beta(t,T)| \,
  \|\vu_{\max}(t,T)^\top \mPhi(T,t)\|
  \nonumber\\
  &=
  |\beta(t,T)| \,
  \sigma_{\max}(\mPhi(T,t))
  \nonumber =
  |\beta(t,T)| \,
  \exp\!\Big(\lambda_{\max}^{FT}(t,T)\,(T-t)\Big).
\end{align}
This proves the lower bound.

If additionally
\begin{align}
|\va_T \, \vu_{\max}(t,T)| \ge c_0 \|\va_T\|,
\end{align}
then substituting $\beta(t,T) = \va_T \vu_{\max}(t,T)$ gives
\begin{align}
\|\va(t)\|
\ge
c_0 \|\va_T\|
\exp\!\Big(\lambda_{\max}^{FT}(t,T)\,(T-t)\Big).
\end{align}
The three consequences stated in the theorem follow immediately from the upper and lower bounds by substituting the sign of \(\lambda_{\max}^{FT}(t,T)\).
\end{proof}
\begin{theorem}\label{thm:adjoint_state}
Consider the autonomous Neural ODE
\begin{align}\label{base:ODE}
\dot{\vz}(t) = \vf(\vz(t), \boldsymbol{\theta}),
\end{align}
trained on a time series data $\{\vz(t) : 0 \le t \le T\}$, using the loss function $\mathcal{J}$ defined as \eqref{eq:loss}.
Let $\vf \in C^1$ and let $\vz(t)$ be a solution that exists for arbitrarily long times.
Fix an initial time $t \in [0,T]$ and let $\lambda_{\max}(t)$ denote the largest LE along the trajectory starting at $\vz(t)$.
Then, as $T \to \infty$ with $t$ fixed, the adjoint state $\va(t) = \dfrac{d \mathcal{J}}{d \vz(t)}$ satisfies
\[
\|\va(t)\| \leq \|\va_T\| \, M(\epsilon) \, e^{(\lambda_{\max}(t) + \epsilon)(T-t)}
\hspace{.5cm} \forall \epsilon > 0.
\]

Moreover, under the assumptions of Oseledec's multiplicative ergodic theorem, for terminal adjoint vectors $\va_T$ that have a nonzero component in the top dual Oseledec subspace associated with $\lambda_{\max}(t)$, there exists a matching lower bound
\[
\|\va(t)\| \ge m(\epsilon, \va_T) \, e^{(\lambda_{\max}(t) - \epsilon)(T-t)}
\hspace{.5cm} \forall \epsilon > 0,
\]
for sufficiently large $T-t$.
Consequently:
\begin{enumerate}
\item If $\lambda_{\max}(t) > 0$, the upper bound allows at most exponential growth. For such adjoint vectors $\va_T$, the lower bound ensures that $\|\va(t)\|$ actually grows exponentially backward in time. For exceptional $\va_T$, e.g., $\va_T=0$ or lying entirely in a lower-growth Oseledec subspace, the growth may be slower.
\item If $\lambda_{\max}(t) < 0$, the upper bound decays exponentially backward. Hence, for all $\va_T$, $\|\va(t)\|$ decays exponentially backward.
\item If $\lambda_{\max}(t) = 0$, the upper bound gives at most subexponential growth for all $\va_T$.
\end{enumerate}
\end{theorem}
%
\begin{proof}
From the proof of Theorem~\ref{thm:adjoint_state:FTLE}, the adjoint state satisfies
\begin{align}
\va(t) = \va_T \, \mPhi(T,t) \quad \text{for all } t \in [0,T].
\end{align}
Hence
\begin{align}
\|\va(t)\| = \|\va_T \, \mPhi(T,t)\| \leq \|\va_T\| \, \|\mPhi(T,t)\|.
\end{align}
Now fix \(t \in [0,T]\) and let \(\tau = T-t\). 
By Definition~\ref{def:maxLE}, the largest Lyapunov exponent along the trajectory starting at \(\vz(t)\) is
\begin{align}
\lambda_{\max}(t) = \limsup_{T \to \infty} \frac{1}{T-t} \log \|\mPhi(T,t)\|.
\end{align}
Thus, for any \(\epsilon > 0\), there exists a constant \(M(\epsilon) > 0\) such that for all sufficiently large \(T-t\)
\begin{align}
\|\mPhi(T,t)\| \leq M(\epsilon) \, e^{(\lambda_{\max}(t) + \epsilon)(T-t)}.
\end{align}
Substituting this into the adjoint norm bound gives
\begin{align}
\|\va(t)\| \leq \|\va_T\| \, M(\epsilon) \, e^{(\lambda_{\max}(t) + \epsilon)(T-t)}.
\end{align}
Thus, for every $\epsilon > 0$,
\begin{align}
\|\va(t)\| = \mathcal{O}\left(e^{(\lambda_{\max}(t) + \epsilon)(T-t)}\right)
\end{align}
as $T - t \to \infty$ with $t$ fixed. This proves the claimed exponential-order upper bound.

To obtain the lower bound, we invoke Oseledec's multiplicative ergodic theorem. Let \(\mu\) be an invariant measure for the flow of the state ODE \eqref{base:ODE}.
For \(\mu\)-almost every initial condition, the tangent space admits an Oseledec decomposition.
The adjoint state \(\va(t)\) evolves backward in time according to
\begin{align}
\va(t) = \va_T \, \mPhi(T,t).
\end{align}
Hence its asymptotic growth rate is determined by the dual Oseledec decomposition.
The top dual Oseledec subspace is the set of adjoint vectors whose backward growth rate equals \(\lambda_{\max}(t)\).
If the terminal adjoint $\va_T$ has a nonzero component in this subspace, then, by Oseledec's theorem, the following limit exists:
\begin{align}
\lim_{\tau \to \infty} \frac{1}{\tau} \log \| \va_T \, \mPhi(T, T-\tau) \| = \lambda_{\max}(t).
\end{align}
This means that the backward growth rate of \(\|\va(t)\|\) equals \(\lambda_{\max}(t)\) for such \(\va_T\).
Therefore, for any \(\epsilon > 0\), there exists a constant \(m(\epsilon, \va_T) > 0\) such that for all sufficiently large \(\tau = T-t\)
\begin{align}
\|\va(t)\| \ge m(\epsilon, \va_T) \, e^{(\lambda_{\max}(t) - \epsilon)(T-t)}.
\end{align}

Finally, the three consequences follow immediately from the upper and lower bounds by substituting the sign of $\lambda_{\max}(t)$: positive yields exponential growth backward in time for terminal adjoint vectors $\va_T$ having a nonzero component in the top dual Oseledec subspace associated with $\lambda_{\max}(t)$ (and possibly slower growth for exceptional $\va_T$), negative yields exponential decay backward in time for all $\va_T$, and zero yields at most subexponential growth backward in time for all $\va_T$.
\end{proof}
  
\begin{theorem}\label{thm:param_grad_unified}
Consider the autonomous Neural ODE
\begin{align}\label{base:ODE_param}
  \dot{\vz}(t) = \vf(\vz(t), \boldsymbol{\theta}),
\end{align}
trained on a time series data $\{\vz(t) : 0 \le t \le T\}$, using the loss function $\mathcal{J}$ defined as \eqref{eq:loss}.
Let $\vf \in C^1$, let $\vz(t)$ be a solution that exists for arbitrarily long times, and assume
\begin{align}
  \left\| \frac{\partial \vf}{\partial \boldsymbol{\theta}} \right\| \le C_\theta,
\end{align}
for some constant $C_\theta > 0$.
Fix $t \in [0,T)$ and let $\lambda_{\max}^{FT}(t,T)$ denote the largest FTLE over the interval $[t,T]$
\begin{align}
  \lambda_{\max}^{FT}(t,T)
  :=
  \frac{1}{T-t}
  \log \left\| \mPhi(T, t) \right\|.
\end{align}
Let $\lambda_{\max}(t)$ denote the largest LE along the trajectory starting at $\vz(t)$.

\smallskip
\noindent \textbf{(i) Finite-time (FTLE) bounds.}
For every $T > 0$, the parameter gradient $\nabla_{\boldsymbol{\theta}} \mathcal{J} = \dfrac{d \mathcal{J}}{d \boldsymbol{\theta}}$ satisfies
\begin{align}
  \|\nabla_{\boldsymbol{\theta}} \mathcal{J}\|
  \le
  C_\theta \, \| \va_T \|
  \int_0^T
  \exp\!\Big(\lambda_{\max}^{FT}(t,T)\,(T-t)\Big)
  \, dt.
\end{align}
Moreover, if the terminal adjoint vector $\va_T$ is not orthogonal to the maximally expanding direction over $[t,T]$, i.e.
\begin{align}
  \beta(t,T) := \va_T \, \vu_{\max}(t,T) \neq 0,
\end{align}
where $\vu_{\max}(t,T)$ is a unit left singular vector of $\mPhi(T,t)$ corresponding to the largest singular value, then
\begin{align}
  \|\nabla_{\boldsymbol{\theta}} \mathcal{J}\|
  \ge
  C_\theta
  \int_0^T
  |\beta(t,T)|
  \exp\!\Big(\lambda_{\max}^{FT}(t,T)\,(T-t)\Big)
  \, dt.
\end{align}

\smallskip
\noindent \textbf{(ii) Asymptotic (LE) bounds.}
As $T \to \infty$ with $t$ fixed, for every $\epsilon > 0$
\begin{align}
  \|\nabla_{\boldsymbol{\theta}} \mathcal{J}\|
  \le
  C_\theta \, \| \va_T \| \, M(\epsilon)
  \int_0^T
  e^{(\lambda_{\max}(t) + \epsilon)(T-t)}
  \, dt.
\end{align}
Moreover, under the assumptions of Oseledec's multiplicative ergodic theorem, for terminal adjoint vectors $\va_T$ that have a nonzero component in the top dual Oseledec subspace associated with $\lambda_{\max}(t)$, there exists a constant $m(\epsilon, \va_T) > 0$ such that for sufficiently large $T$
\begin{align}
  \|\nabla_{\boldsymbol{\theta}} \mathcal{J}\|
  \ge
  C_\theta \, m(\epsilon, \va_T)
  \int_0^T
  e^{(\lambda_{\max}(t) - \epsilon)(T-t)}
  \, dt.
\end{align}
\smallskip
\noindent \textbf{(iii) Finite-time regime consequences (FTLE).}
Let $\lambda_{\max}^{FT} := \lambda_{\max}^{FT}(t,T)$.
\begin{enumerate}
  \item If $\lambda_{\max}^{FT} > 0$, then
    \[
      \|\nabla_{\boldsymbol{\theta}} \mathcal{J}\|
      \le
      C_\theta \, \| \va_T \|
      \frac{e^{\lambda_{\max}^{FT} T} - 1}{\lambda_{\max}^{FT}}
      = \mathcal{O}\!\left(e^{\lambda_{\max}^{FT} T}\right).
    \]
    If in addition $|\beta(t,T)| \ge \beta_{\min} > 0$ for all $t \in [0,T]$, then
    \[
      \|\nabla_{\boldsymbol{\theta}} \mathcal{J}\|
      \ge
      C_\theta \, \beta_{\min}
      \frac{e^{\lambda_{\max}^{FT} T} - 1}{\lambda_{\max}^{FT}}.
    \]
  \item If $\lambda_{\max}^{FT} = 0$, then
    \[
      \|\nabla_{\boldsymbol{\theta}} \mathcal{J}\|
      \le
      C_\theta \, \| \va_T \| \, T
      = \mathcal{O}(T).
    \]
  \item If $\lambda_{\max}^{FT} < 0$, then
    \[
      \|\nabla_{\boldsymbol{\theta}} \mathcal{J}\|
      \le
      C_\theta \, \| \va_T \|
      \frac{1 - e^{\lambda_{\max}^{FT} T}}{|\lambda_{\max}^{FT}|}
      \le
      \frac{C_\theta \, \| \va_T \|}{|\lambda_{\max}^{FT}|},
    \]
    so the parameter gradient is uniformly bounded independently of $T$.
\end{enumerate}

\smallskip
\noindent \textbf{(iv) Asymptotic regime consequences (LE).}
\begin{enumerate}
  \item If $\lambda_{\max}(t) > 0$, then
    \[
      \|\nabla_{\boldsymbol{\theta}} \mathcal{J}\|
      \le
      C_\theta \, \| \va_T \| \, M(\epsilon)
      \frac{e^{(\lambda_{\max}(t) + \epsilon) T} - 1}{\lambda_{\max}(t) + \epsilon}
      = \mathcal{O}\!\left(e^{(\lambda_{\max}(t) + \epsilon) T}\right),
    \]
    and for $\va_T$ satisfying the top dual Oseledec condition
  \begin{align}
   \|\nabla_{\boldsymbol{\theta}} \mathcal{J}\|
      \ge
      C_\theta \, m(\epsilon, \va_T)
      \frac{e^{(\lambda_{\max}(t) - \epsilon) T} - 1}{\lambda_{\max}(t) - \epsilon}
      \ge
      \frac{C_\theta \, m(\epsilon, \va_T)}{\lambda_{\max}(t) - \epsilon}
      \, e^{(\lambda_{\max}(t) - \epsilon) T}.
   \end{align}
    For exceptional $\va_T$ (e.g., $\va_T = 0$ or lying entirely in a lower-growth Oseledec subspace), the growth may be slower.
  \item If $\lambda_{\max}(t) = 0$, then
  \begin{align}
  \|\nabla_{\boldsymbol{\theta}} \mathcal{J}\|
      \le
      C_\theta \, \| \va_T \| \, M(\epsilon)
      \frac{e^{\epsilon T} - 1}{\epsilon}
      = \mathcal{O}(T)
    \end{align}
    for fixed $\epsilon > 0$, so the parameter gradient grows at most linearly.
  \item If $\lambda_{\max}(t) < 0$, then choosing $\epsilon < |\lambda_{\max}(t)|$ gives
    \[
      \|\nabla_{\boldsymbol{\theta}} \mathcal{J}\|
      \le
      C_\theta \, \| \va_T \| \, M(\epsilon)
      \frac{1 - e^{-(|\lambda_{\max}(t)| - \epsilon) T}}{|\lambda_{\max}(t)| - \epsilon}
      \le
      \frac{C_\theta \, \| \va_T \| \, M(\epsilon)}{|\lambda_{\max}(t)| - \epsilon},
    \]
    so the parameter gradient remains uniformly bounded independently of $T$.
\end{enumerate}
\end{theorem}
\begin{proof}
From the adjoint sensitivity equations, the parameter gradient admits the integral representation
\begin{align}\label{eq:param_grad_integral_unified}
  \nabla_{\boldsymbol{\theta}} \mathcal{J}
  =
  - \int_0^T \va(t) \, \frac{\partial \vf}{\partial \boldsymbol{\theta}}(\vz(t)) \, dt,
\end{align}
where $\va_\theta(T) = 0$.
Taking norms and applying the triangle inequality
\begin{align}
  \|\nabla_{\boldsymbol{\theta}} \mathcal{J}\|
  \le
  \int_0^T \| \va(t) \| \,
  \left\| \frac{\partial \vf}{\partial \boldsymbol{\theta}}(\vz(t)) \right\| \, dt
  \le
  C_\theta \int_0^T \| \va(t) \| \, dt.
\end{align}

The finite-time bounds in (i) follow directly from Theorem~\ref{thm:adjoint_state:FTLE}.
Indeed, that theorem gives, for every $t \in [0,T]$
\begin{align}
  \|\va(t)\|
  \le
  \|\va_T\|
  \exp\!\Big(\lambda_{\max}^{FT}(t,T)\,(T-t)\Big),
\end{align}
and, if $\beta(t,T) \neq 0$
\begin{align}
  \|\va(t)\|
  \ge
  |\beta(t,T)|
  \exp\!\Big(\lambda_{\max}^{FT}(t,T)\,(T-t)\Big).
\end{align}
Substituting these pointwise bounds into the integral representation and using $\|\partial \vf / \partial \boldsymbol{\theta}\| \le C_\theta$ yields the FTLE upper and lower bounds in (i).

The asymptotic bounds in (ii) follow directly from Theorem~\ref{thm:adjoint_state}.
Indeed, that theorem gives, for every $\epsilon > 0$ and for sufficiently large $T-t$
\begin{align}
  \|\va(t)\|
  \le
  \|\va_T\| \, M(\epsilon) \, e^{(\lambda_{\max}(t) + \epsilon)(T-t)},
\end{align}
and, under the assumptions of Oseledec's multiplicative ergodic theorem, for terminal adjoint vectors $\va_T$ having a nonzero component in the top dual Oseledec subspace associated with $\lambda_{\max}(t)$
\begin{align}
  \|\va(t)\|
  \ge
  m(\epsilon, \va_T) \, e^{(\lambda_{\max}(t) - \epsilon)(T-t)}.
\end{align}
Substituting these pointwise bounds into the integral representation yields the LE upper and lower bounds in (ii).

The regime consequences in (iii) and (iv) follow by evaluating the elementary exponential integrals in (i) and (ii), respectively, according to the sign of $\lambda_{\max}^{FT}(t,T)$ and $\lambda_{\max}(t)$.
This completes the proof.
\end{proof}
\section{Experiment Details}
\label{app:experiments}

We provide details on the datasets and hyperparameters for all experiments in this chapter. Each problem has a dedicated section.  We also include information on our calculation of spectral residuals in \Cref{sec:gEDMD_spectral_pollution} and an experiment demonstrating the exploding gradient problem for the adjoint sensitivity method in \Cref{app:node_exploding_adjoint}.

Our implementation is based on the existing python library for learning operators from data called \texttt{datafold} \citep{Lehmberg2020}.
Derivatives were computed using a simple finite difference scheme whenever necessary. Further, our implementation appends the raw states to the dictionary in order to facilitate a mapping from the observables space back into the state space, for this reason the size of dictionaries is typically the sum of the number of neurons and the dimension of the state.

The experiment code and datasets are provided as supplementary material.
All computations were performed on a AMD EPYC 7702 64-Core Processor.

\subsection{Van der Pol}
The Van der Pol oscillator is a good example of a nonlinear system with simple dynamics and a well-understood behavior.
The dynamics are given by
\begin{equation}
\begin{split} 
        \dot{z_1} &= z_2 \\
        \dot{z_2} &= \mu (1 - z_1^2)z_2 - z_1,
\end{split}
\end{equation}
and we set \(\mu = 1\). We generate an time-array of irregularly spaced timepoints by drawing each \(\Delta t \sim U(0, 0.05)\). Each trajectory consists of 200 snapshots in time. For training, validation and testing we simulate 50 trajectories for each dataset, with each trajectory starting from a randomly sampled initial condition \((z_1, z_2) \sim U(-3, +3) \times U(-3, +3) \).

\begin{wraptable}{f}{7.5cm}
    \caption{Hyperparameter search scope for DRAGONS on the Van der Pol problem. Best combination shown with bold font.}
    \label{tab:hyperparameters_DRAGONS_VdP}
    \centering
    \small
  \begin{tabular}{ cc }
    \toprule
    sampling & flow-, vector field-, \textbf{generator-driven} \\
    \midrule
         layer width & 
        16,
        32,
        64,
        128,
        256,
        \textbf{512}\\
    \midrule
        activation &
        relu, \textbf{tanh} \\
    \midrule
    rcond &        
        1e-6,
        1e-10,
        1e-14,
        \textbf{1e-18} \\
    \bottomrule
  \end{tabular}
\end{wraptable}

We perform a grid search over possible hyperparameter combinations for the different dictionary choices, these are given in \Cref{tab:hyperparameters_DRAGONS_VdP}, \Cref{tab:hyperparameters_Gaussian_VdP}, \Cref{tab:hyperparameters_gEDMD-DL_VdP}, \Cref{tab:hyperparameters_RBF_VdP}, \Cref{tab:hyperparameters_Poly_VdP} for the DRAGONS approach, Gaussian parameters, learned dictionary with Adam, an RBF dictionary and a polynomial dictionary, respectively. The range of hyperparameters was determined via manual trials in advance of the grid search. We evaluate each model on the validation set using a mean squared error loss and averaging this over the 50 trajectories. All models are initialized 5 times with different random seed, and we report the average, minimum and maximum result across the random seeds in \Cref{tab:main_results}.

\begin{table}[ht]
    \begin{minipage}{.51\linewidth}
    \caption{Hyperparameter search scope for Gaussian random features on the Van der Pol problem. Best combination shown with bold font.}
    \label{tab:hyperparameters_Gaussian_VdP}
    \centering
    \small
  \begin{tabular}{ cc }
    \toprule
         layer width & 
        16,
        32,
        64,
        \textbf{128},
        256,
        512
        \\
    \midrule
        activation &
        relu, \textbf{tanh} \\
    \midrule
    rcond &  
        1e-6,
        1e-10,
        \textbf{1e-14},
        1e-18       \\
    \bottomrule
  \end{tabular}
  \end{minipage}
    \hspace{0.25cm}
    \begin{minipage}{.45\linewidth}
    \caption{Hyperparameter search scope for gEDMD-DL on the Van der Pol problem. Best combination shown with bold font.}
    \label{tab:hyperparameters_gEDMD-DL_VdP}
    \centering
    \small
  \begin{tabular}{ cc }
    \toprule
         layer width & 
        16,
        32,
        64,
        \textbf{128},
        256,
        512\\
    \midrule
        activation &
        relu, \textbf{tanh} \\
    \midrule
    rcond &       
        1e-6,
        1e-10,
        1e-14,
        \textbf{1e-18} \\
    \bottomrule
  \end{tabular}
  \end{minipage}
\end{table}

\begin{table}[ht]
    \begin{minipage}{.45\linewidth}
    \caption{Hyperparameter search scope for RBF-Dict on the Van der Pol problem. \\Best combination shown with bold font.}
    \label{tab:hyperparameters_RBF_VdP}
    \centering
    \small
  \begin{tabular}{ cc }
    \toprule
         num of samples & 100, \textbf{200} \\
    \midrule
        rcond & 1e-8, \textbf{1e-12}, 1e-16 \\
    \bottomrule
  \end{tabular}
  \end{minipage}
  \hspace{1cm}
  \begin{minipage}{.45\linewidth}
    \caption{Hyperparameter search scope for Poly-Dict on the Van der Pol problem. \\Best combination shown with bold font.}
    \label{tab:hyperparameters_Poly_VdP}
    \centering
    \small
  \begin{tabular}{ cc }
    \toprule
         degree & 1, 2, \textbf{3} \\
    \midrule
        rcond &
        1e-8,
        1e-10,
        1e-12,
        1e-14,
        \textbf{1e-16} \\
    \bottomrule
  \end{tabular}
  \end{minipage}
\end{table}

We visualize predictions using the DRAGONS approach on the test set in \Cref{fig:DRAGONS_VdP}. The predicted trajectory matches the trajectory obtained with a numerical ODE solver.
\begin{figure}[h]
    \centering
    \begin{subfigure}{0.265\textwidth}
        \includegraphics[width=\linewidth]{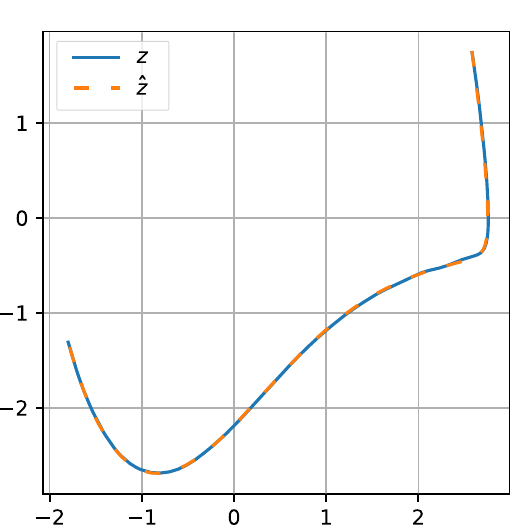}
        \caption{Phase space}
        \label{fig:DRAGONS_VdP_phase}
    \end{subfigure}
    \hspace{10pt}
    \begin{subfigure}{0.63\textwidth}
    \includegraphics[width=\linewidth]{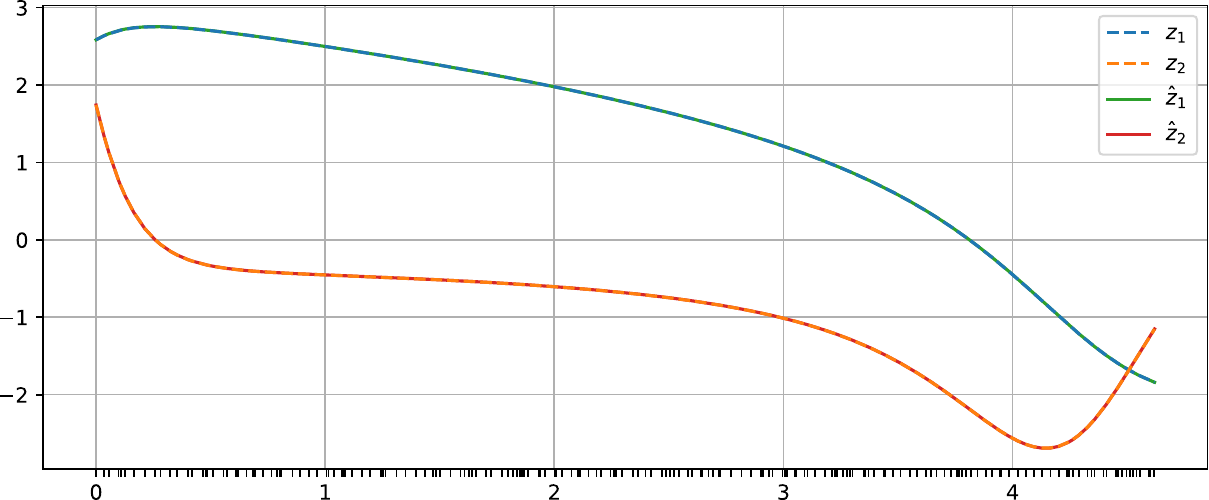}
        \caption{Trajectory}
        \label{fig:DRAGONS_VdP_traj}
    \end{subfigure}
    \caption{Comparison of numerically calculated states $z$ versus DRAGONS predicted trajectories $\hat{z}$ for Van der Pol experiment. In subplot b) the ticks on the x-axis show the evaluation points in time, which are irregularly spaced.}
    \label{fig:DRAGONS_VdP}
\end{figure}

Further, we investigate potential ill-conditioning of the system to be solved using least squares. For this purpose, we compute the condition number of \(\Phi_N(X)^T\Phi_N(X)\) for each dictionary choice and provide these in \Cref{tab:cond_num_VdP}. A higher number of dictionary elements is more likely to lead to a high condition number due to an increased likelihood for similar features. We observe that random features approaches have better conditioning than Adam-trained dictionaries of the same size.

\begin{table}[ht]
    \centering
  \begin{tabular}{ cccccc }
    \toprule
         dictionary & DRAGONS & Gaussian & gEDMD-DL & RBF-Dict & Poly-Dict \\ 
         \midrule
         \#dict & 130 & 130 & 130 & 202 & 9 \\
         \midrule
         cond num & 1.18e+14 & 2.32e+13 & 9.86e+18 & 3.77e+18 & 3.97e+02 \\
    \bottomrule
  \end{tabular}
    \caption{Condition numbers of different dictionaries for Van der Pol problem.}
    \label{tab:cond_num_VdP}
\end{table}

\subsection{Lorenz 63}
Our second experiment is concerned with the Lorenz system \citep{lorenz-1963} in the chaotic regime with $\sigma=10$, $\beta=\frac{8}{3}$, and $\rho=28$ given as
\begin{equation}
\begin{split}
        \dot{z_1} & = \sigma(z_2-z_1)\\ 
        \dot{z_2} & = z_1(\rho-z_3)-z_2\\ 
        \dot{z_3} & = z_1\,z_2-\beta z_3.
\end{split}
\end{equation}

We generate datasets by numerically solving the system using an 8th order explicit Runge Kutta scheme at time values where each consecutive $\Delta t \sim U(0,0.03)$ using 1000 evaluation points for one orbit, and creating 50 orbits for training, validation and for test datasets. The initial conditions for each new orbit are sampled according to \((z_1, z_2, z_3) \sim U(-20, +20) \times U(-20, +20) \). 

\begin{wraptable}{f}{7.5cm}
    \caption{Hyperparameter search scope for DRAGONS on the Lorenz 63 problem.\\Best combination shown with bold font.}
\label{tab:hyperparameters_DRAGONS_L63}
    \centering
    \small
  \begin{tabular}{ cc }
    \toprule
    sampling & flow-, vector field-, \textbf{generator-driven} \\
    \midrule
         layer width & 
        16,
        32,
        64,
        128,
        \textbf{256},
        512\\
    \midrule
        activation &
        \textbf{relu}, tanh \\
    \midrule
    rcond &        
        1e-2,
        1e-6,
        1e-10,
        \textbf{1e-14},
        1e-18 \\
    \bottomrule
  \end{tabular}
\end{wraptable}

We perform a hyperparameter search similarly to the Van der Pol problem, and report the combinations of parameters in \Cref{tab:hyperparameters_DRAGONS_L63}, \Cref{tab:hyperparameters_Gaussian_L63}, \Cref{tab:hyperparameters_gEDMD-DL_L63}, \Cref{tab:hyperparameters_RBF_L63} and \Cref{tab:hyperparameters_Poly_L63} for the DRAGONS, Gaussian parameters, gEDMD-DL RBF dictionary and polynomial dictionary, respectively. Since this is a chaotic system, using the mean squared error metric is not suitable, thus we use an empirical estimate of Kullback-Leibler (KL) divergence between the true and predicted orbits, which should be very similar for very long trajectories due to ergodic properties of the system.

To this end, we treat the trajectories as distributions, denoting the true distribution as \( p(\vz)\) and the predicted distribution as \( q(\vz)\). KL divergence is defined as 
\begin{equation}
    D_{KL} (p(\vx) \lVert q(\vx)) = \int_{\vx \in \mathbb{R}} p(\vx) \log \frac{p(\vx)}{q(\vx)} d\vx .
    \label{eq:kl_div}
\end{equation}.

To empirically estimate this quantity we use an approach similar to \citet{hess_generalized_2023} to use Gaussian Mixture Models (GMM) placed on the true trajectory $\vz$ and predicted $\hat{\vz}$ trajectories
\begin{equation}
    \begin{split}
        \hat{p}(\vz) = \frac{1}{T} \sum_{t=1}^{T} \mathcal{N} (\vz, \vz_t, \Sigma) \qquad \hat{q}(\vz) = \frac{1}{T} \sum_{t=1}^{T} \mathcal{N} (\vz, \hat{\vz}_t, \Sigma)
    \end{split}
\end{equation} for $T$ points in time; we assume unit variance.
These GMMs are then sampled to approximate the integral in \Cref{eq:kl_div} as
\begin{align}\label{eq:D_stsp}
    D_{KL} (\hat{p}(\vz) \lVert \hat{q}(\vz)) \approx \frac{1}{1000} \sum_{i=1}^{1000} \log \frac{\hat{p}(\vz)}{\hat{q}(\vz)} .
\end{align}

In \Cref{fig:DRAGONS_L63} we show predictions on a test trajectory with our DRAGONS approach.

\begin{table}[ht]
    \begin{minipage}{.45\linewidth}
    \caption{Hyperparameter search scope for Gaussian random features on the Lorenz 63 problem.\\Best combination shown with bold font.}
    \label{tab:hyperparameters_Gaussian_L63}
    \centering
    \small
  \begin{tabular}{ cc }
    \toprule
         layer width & 
        16,
        32,
        \textbf{64},
        128,
        256,
        512\\
    \midrule
        activation &
        relu, \textbf{tanh} \\
    \midrule
    rcond &     
        1e-6,
        \textbf{1e-10},
        1e-14,
        1e-18 \\
    \bottomrule
  \end{tabular}
  \end{minipage}
  \hspace{1cm}
    \begin{minipage}{.45\linewidth}
    \caption{Hyperparameter search scope for gEDMD-DL on the Lorenz 63 problem.\\Best combination shown with bold font.}
    \label{tab:hyperparameters_gEDMD-DL_L63}
    \centering
    \small
  \begin{tabular}{ cc }
    \toprule
         layer width & 
        \textbf{16},
        32,
        64,
        128,
        256,
        512\\
    \midrule
        activation &
        relu, \textbf{tanh} \\
    \midrule
    rcond &     
        \textbf{1e-6},
        1e-10,
        1e-14,
        1e-18 \\
    \bottomrule
  \end{tabular}
    \end{minipage}
\end{table}

\begin{table}[!h]
    \begin{minipage}{.45\linewidth}
    \caption{Hyperparameter search scope for RBF-Dict on the Lorenz 63 problem.\\Best combination shown with bold font.}
    \label{tab:hyperparameters_RBF_L63}
    \centering
    \small
  \begin{tabular}{ cc }
    \toprule
         num of samples & 100, 200, \textbf{300} \\
    \midrule
        rcond & 
        1e-4,
        \textbf{1e-8},
        1e-12,
        1e-16 \\
    \bottomrule
  \end{tabular}
  \end{minipage}
  \hspace{1cm}
  \begin{minipage}{.45\linewidth}
    \caption{Hyperparameter search scope for Poly-Dict on the Lorenz 63 problem.\\Best combination shown with bold font.}
    \label{tab:hyperparameters_Poly_L63}
    \centering
    \small
  \begin{tabular}{ cc }
    \toprule
         degree & 1, 2, \textbf{3} \\
    \midrule
        rcond &
        1e-6,
        1e-8,
        1e-10,
        1e-12,
        \textbf{1e-14},
        1e-16\\
    \bottomrule
  \end{tabular}
  \end{minipage}
\end{table}

To investigate the conditioning of the linear system we again report the condition number of \(\Phi_N(X)^T\Phi_N(X)\), for the Lorenz 63 problem the results are given in \Cref{tab:cond_num_L63}. Here we see that even when random feature dictionaries are much larger than the Adam-trained dictionary, they lead to a better conditioned linear system.

\begin{table}[ht]
    \centering
  \begin{tabular}{ cccccc }
    \toprule
         dictionary & DRAGONS & Gaussian & gEDMD-DL & RBF-Dict & Poly-Dict \\ 
         \midrule
         \#dict & 131 & 67 & 19 & 303 & 34 \\
         \midrule
         cond num & 6.85e+09 & 1.07e+08 & 3.24e+18 & 1.60e+12 & 1.72e+10 \\
    \bottomrule
  \end{tabular}
    \caption{Condition numbers of different dictionaries for the Lorenz 63 problem.}
    \label{tab:cond_num_L63}
\end{table}

\subsection{Flow around a cylinder}

The simulation data for the flow past a cylinder at Reynolds number \(Re=160\) are taken from the Computer
Graphics Laboratory at ETH Zurich \citep{gerrisflowsolver, Guenther17}. The corresponding unstable steady-state
solution is computed by enforcing symmetric boundary conditions at \(y=0\). Using these data, we construct a three-dimensional spatial basis consisting of two oscillatory modes and one shift mode, see \Cref{fig:cyl_spatial_modes}. The oscillatory modes $\phi_1$ and $\phi_2$ are obtained as the first two POD modes of the velocity fluctuations around the mean limit-cycle flow \citep{sirovich-1987}. The shift mode $\phi_{\mathrm{shift}}$ represents the change from the unstable steady flow to the mean limit-cycle flow. The corresponding temporal coordinates are modeled using a three-dimensional mean-field model following \citet{noack-2003}:
\begin{equation}
\begin{split}
        \dot{z}_1 & = \mu z_1 - \omega z_2 + A z_1 z_3,\\
        \dot{z}_2 & = \omega z_1 + \mu z_2 + A z_2 z_3,\\
        \dot{z}_3 & = -\lambda (z_3 - z_1^2 - z_2^2).
\end{split}
\end{equation}
Here, $z_1$ and $z_2$ are the normalized oscillatory coordinates and $z_3$ is the shift coordinate. The angular frequency $\omega$ is identified from the developed limit-cycle data by fitting the unwrapped phase of the oscillatory coordinates. The transient growth rate is set to $\mu=0.1|\omega|$ and the shift relaxation rate to $\lambda=5|\omega|$. The nonlinear coupling coefficient is chosen as $A=-\mu$, such that the radial growth rate vanishes at $z_3=1$. With these choices, the system has an unstable equilibrium at the origin and a stable limit cycle defined by
\[
z_1^2+z_2^2=1, \qquad z_3=1.
\]

The dataset is generated by numerically solving the reduced dynamical system. Each trajectory consists of 600 snapshots separated by irregular time increments, independently sampled according to \(\Delta t \sim U(0, 0.01)\). We generate 50 trajectories each for training, validation, and testing. The initial condition for each trajectory is sampled close to the parabolic slow manifold \(z_3=z_1^2+z_2^2\) by drawing
\[
(r,\theta)\sim U(0.05,1.1)\times U(0,2\pi),
\]
and setting
\[
z_1=r\cos\theta,\qquad z_2=r\sin\theta,\qquad z_3=r^2+\delta,
\]
where \(\delta\) is sampled uniformly from the admissible interval within \([-0.3,0.3]\) such that \(0\le z_3\le1.2\).

\begin{wraptable}{f}{7.5cm}
    \caption{Hyperparameter search scope for DRAGONS on the cylinder flow problem.\\Best combination shown with bold font.}
    \label{tab:hyperparameters_DRAGONS_cyl}
    \centering
    \small
  \begin{tabular}{ cc }
    \toprule
    sampling & flow-, vector field-, \textbf{generator-driven} \\
    \midrule
        layer width & \textbf{128}, 256, 512 \\
    \midrule
        activation & \textbf{relu}, tanh \\
    \midrule
        rcond & 1e-8, 1e-10, \textbf{1e-12}, 1e-14 \\
    \bottomrule
  \end{tabular}
\end{wraptable}

We again perform a hyperparameter search following the same procedure as for the Van der Pol and Lorenz 63 problems. The considered hyperparameter combinations are given in \Cref{tab:hyperparameters_DRAGONS_cyl}, \Cref{tab:hyperparameters_Gaussian_cyl}, \Cref{tab:hyperparameters_gEDMD-DL_cyl}, \Cref{tab:hyperparameters_RBF_cyl} and \Cref{tab:hyperparameters_Poly_cyl} for DRAGONS, Gaussian random parameters, gEDMD-DL using an Adam optimizer, RBF-Dict and Poly-Dict, respectively. The search range was determined manually by testing out sensible combinations. We evaluate on the validation set using a mean squared error loss to select the best hyperparameter combination, which then is evaluated on a test set and reported in \Cref{tab:main_results}. Similarly to the other forecasting experiments, a generator-driven sampling yields the best results.

\begin{table}[ht]
    \begin{minipage}{.53\linewidth}
    \caption{Hyperparameter search scope for Gaussian random features on the cylinder flow problem.\\Best combination shown with bold font.}
    \label{tab:hyperparameters_Gaussian_cyl}
    \centering
    \small
  \begin{tabular}{ cc }
    \toprule
        layer width & \textbf{128}, 256, 512 \\
    \midrule
        activation & relu, \textbf{tanh} \\
    \midrule
        rcond & 1e-8, 1e-10, 1e-12, 1e-14, \textbf{1e-18} \\
    \bottomrule
  \end{tabular}
  \end{minipage}
  \hspace{0.15cm}
    \begin{minipage}{.45\linewidth}
    \caption{Hyperparameter search scope for gEDMD-DL on the cylinder flow problem.\\Best combination shown with bold font.}
    \label{tab:hyperparameters_gEDMD-DL_cyl}
    \centering
    \small
  \begin{tabular}{ cc }
    \toprule
        layer width & 
        16,
        \textbf{32},
        64,
        128, 
        256, 
        512 \\
    \midrule
        activation & relu, \textbf{tanh} \\
    \midrule
        rcond & 
        1e-6,
        1e-8,
        1e-10,
        1e-12,
        \textbf{1e-14} \\
    \bottomrule
  \end{tabular}
    \end{minipage}
\end{table}

\begin{table}[ht]
    
    \begin{minipage}{.45\linewidth}
    \caption{Hyperparameter search scope for RBF-Dict on the cylinder flow problem.\\Best combination shown with bold font.}
    \label{tab:hyperparameters_RBF_cyl}
    \centering
    \small
  \begin{tabular}{ cc }
    \toprule
         degree & 
        \textbf{100},
        300,
        500 \\
    \midrule
        rcond & 
        1e-6,
        \textbf{1e-10},
        1e-14 \\
    \bottomrule
  \end{tabular}
  \end{minipage}
  \hspace{1cm}
  \begin{minipage}{.45\linewidth}
    \caption{Hyperparameter search scope for Poly-Dict on the cylinder flow problem.\\Best combination shown with bold font.}
    \label{tab:hyperparameters_Poly_cyl}
    \centering
    \small
  \begin{tabular}{ cc }
    \toprule
         degree & 2, \textbf{3}, 4 \\
    \midrule
        rcond & \textbf{1e-8}, 1e-12, 1e-16 \\
    \bottomrule
  \end{tabular}
  \end{minipage}
\end{table}

In this case again we check the condition number of feature matrices and report them in \Cref{tab:cond_num_cyl}. Once again the worst conditioning is attained by the iteratively trained neural network dictionary.

\begin{table}[ht]
    \centering
  \begin{tabular}{ cccccc }
    \toprule
         dictionary & DRAGONS & Gaussian & gEDMD-DL & RBF-Dict & Poly-Dict \\ 
         \midrule
         \#dict & 515 & 131 & 35 & 103 & 34  \\
         \midrule
         cond num & 5.76e+13  & 1.30e+15 & 3.55e+17 & 1.17e+05 & 7.47e+05 \\
    \bottomrule
  \end{tabular}
    \caption{Condition numbers of different dictionaries for the flow around a cylinder problem.}
    \label{tab:cond_num_cyl}
\end{table}

\Cref{fig:cyl_traj_pred} shows the true and predicted trajectory of one test sample. The full velocity field can be reconstructed by mapping the reduced coordinates to the physical modal coefficients associated with the three spatial modes and adding them to the unstable steady base flow. \Cref{fig:cyl_velocity} shows the true and predicted velocity fields at different time steps and the corresponding absolute error clipped at the 99th percentile. The prediction closely follows the reference trajectory in the reduced phase space and captures the dominant spatial flow structures of the reconstructed velocity field.

\begin{figure}[h]
    \centering
    \begin{subfigure}{0.35\textwidth}
        \includegraphics[width=\linewidth]{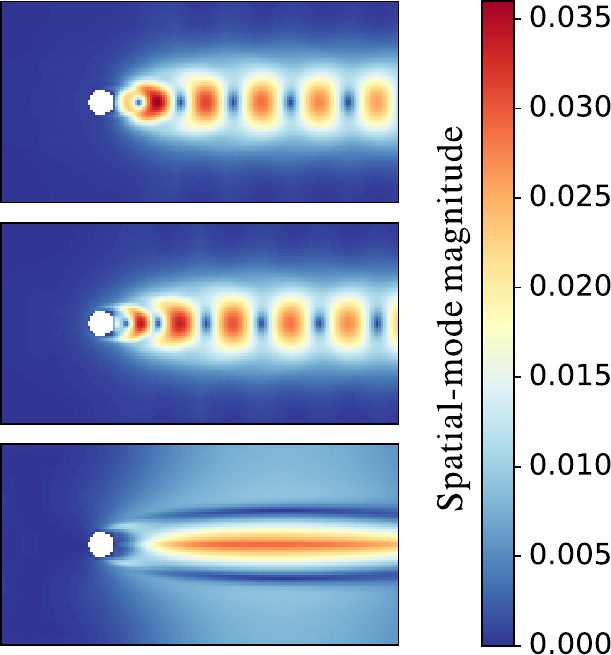}
    \caption{Spatial modes.}
        \label{fig:cyl_spatial_modes}
    \end{subfigure}
    \hspace{20pt}
    \begin{subfigure}{0.5\textwidth}
\includegraphics[width=\linewidth]{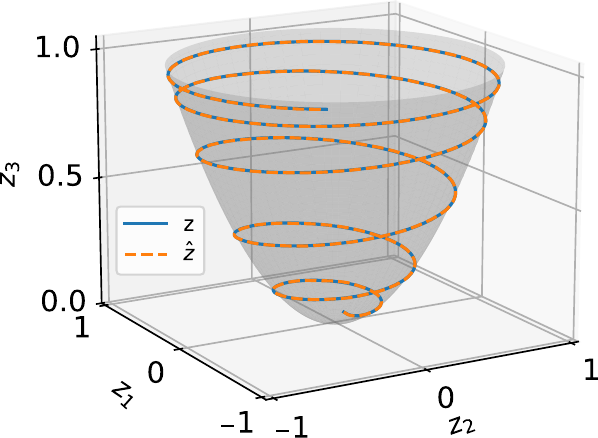}
    \caption{Phase space.}
\label{fig:cyl_traj_pred}
    \end{subfigure}
    \vspace{1cm}
    \begin{subfigure}{0.9\textwidth}
\includegraphics[width=\linewidth]{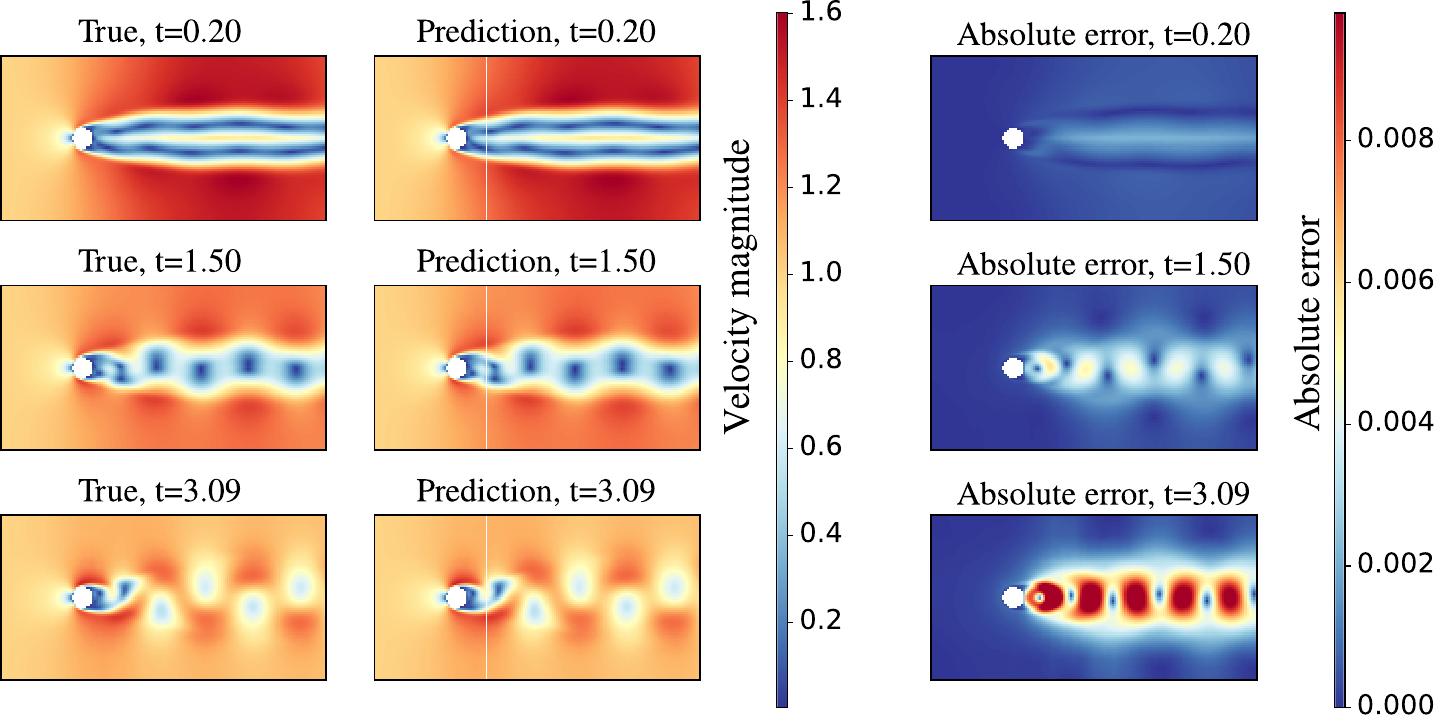}
    \caption{Velocity field reconstruction.}
\label{fig:cyl_velocity}
    \end{subfigure}
\caption{Flow around a cylinder experiment. (a) Two oscillatory modes and one shift mode. (b) Reference and DRAGONS-predicted reduced trajectories in phase space. (c) Reference and predicted velocity fields with the corresponding absolute error, clipped at the 99th percentile.}
\label{fig:cyl_experiment}
\end{figure}

To complement the forecasting results, we examine the spectrum
of the learned Koopman generator and the associated Koopman modes.
In \Cref{fig:cyl_non_oscillatory_modes}, mode 28 corresponds to
an eigenvalue close to zero and represents an almost stationary
contribution, while mode 64 represents a transient decaying
mode. Together with its complex-conjugate counterpart,
mode 31 in \Cref{fig:cyl_oscillatory_mode} represents an oscillatory
mode whose eigenvalue has an imaginary part with magnitude
close to the estimated vortex-shedding angular frequency $|\omega|$.

\subsection{Alanine Dipeptide}

\begin{wrapfigure}{r}{0.26\textwidth}
  \begin{center}
    \includegraphics[width=0.25\textwidth]{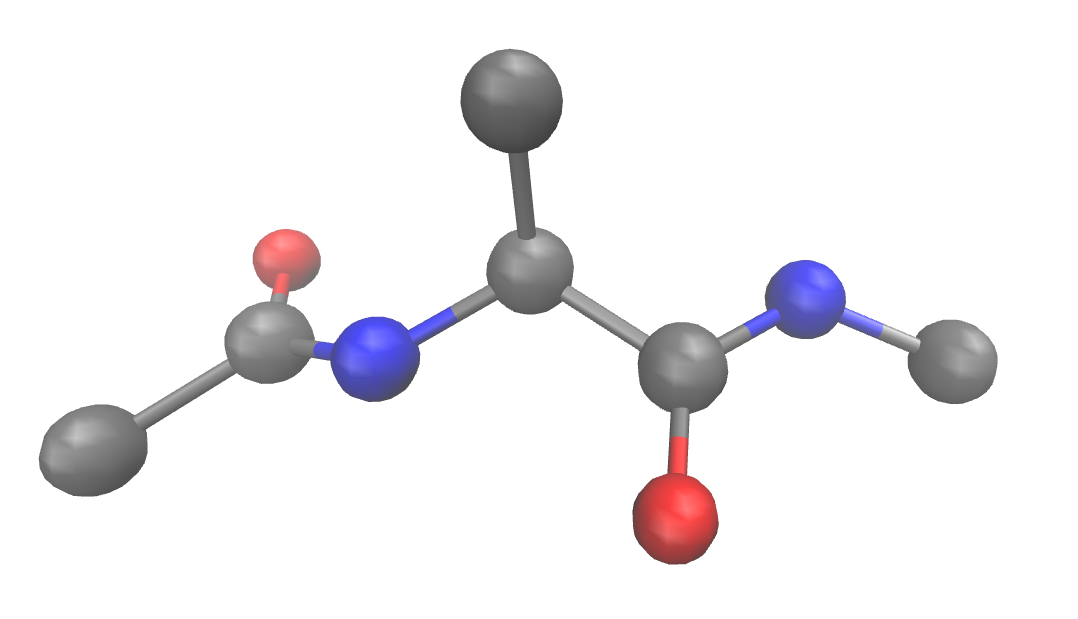}
  \end{center}
  \caption{Molecular representation of alanine dipeptide (only heavy atoms are shown). }
  \label{fig:alanine}
\end{wrapfigure}
We consider a molecular dynamics simulation of the alanine dipeptide molecule in explicit water. 
The alanine dipeptide molecule consists of 22 atoms, however we consider only the 10 heavy atoms, as depicted in \Cref{fig:alanine} and their position in physical space, resulting in a state vector \(\vz \in \mathbb{R}^{30} \). This system has been commonly studied in computational modeling \citep{nuske-2017, klus-2018a, herwig-2025}, thus it represents a robust benchmark on which we test our approach. The datasets are downloaded from \citet{alanine-dipeptide-datasets}, and subsampled to every 100th entry.

In this experiment our aim is not to demonstrate predictive accuracy but rather to perform a conformation analysis. Since our results are not quantifiable we do not directly compare with the existing results from existing approaches in the manuscript; nonetheless the obtained insights align with prior results, the interested reader is invited to make comparison with the cited publications.

We use a DRAGONS model with 58 neurons with a ReLU activation in the hidden layer, where the parameters are sampled using the flow-driven sampling (\Cref{eq:flow-driven sampling}). In the linear system solver we use a truncated SVD with a cut-off ratio of 0.6. The hyperparameters were found using a manual search. 
To perform a clustering we use the k-means algorithm and create three clusters using the five eigenfunctions whose eigenvalues are closest to zero. The clustering is performed in the 30 dimensional state space of the system. However for the purpose of visualization, we map the positions of the molecule to the conformation described by its two dihedral angles as proposed by \citet{ramachandran-1963} to obtain plots in 2D. The model does not have access to the dihedral angles. We visualize two of these eigenfunctions in \Cref{fig:AD_eigenfunctions} and the resulting clustering in \Cref{fig:AD_kmeans}. The results show clear distinctions in the appropriate regions where the metastable states reside, this aligns with previous observations obtained in prior work.

\subsection{Comparison of sampling strategies}

We provide the results across different sampling strategies for the forecasting problems in \Cref{tab:compare_sampling}. The flow-driven sampling already provides a good baseline, while the vector-field-driven sampling is occasionally inaccurate. The generator driven sampling is robust and yields the lowest error on all forecasting problems.

\begin{center}
\label{sec:compare_sampling_strategies}
\begin{table}[h]
\caption{Comparison of different sampling strategies on the forecasting tasks.}
\label{tab:compare_sampling}
\centering
\begin{tabular}{p{2cm} p{3cm}  p{5cm}  }
    \toprule
    Example & sampling distribution & Error: avg (min, max)  \\
    \midrule
    Van der Pol &  flow-driven & 4.22e-6 (2.71e-6, 7.74e-6) \\   
    &  vector-field-driven & 2.22 (1.35e-7, 1.11e+1)  \\   
    &  generator-driven & 6.23e-7 (2.84e-7, 8.08e-7) \\   
     \midrule
     Lorenz 63 &  flow-driven & 5.27e-4 (1.10e-4, 9.53e-4)  \\   
    &  vector-field-driven & 4.20e-4 (1.97e-4, 6.54e-4) \\   
    &  generator-driven & 3.75e-4 (1.87e-4, 5.69e-4) \\ 
      \midrule 
     Cylinder &  flow-driven & 8.68e-4 (4.51e-7, 4.33e-3) \\   
    &  vector-field-driven & 6.09e-7 (3.67e-7, 1.21e-6) \\   
    &  generator-driven & 3.35e-7 (3.28e-7, 3.42e-7) \\ 
  \hline
\end{tabular}
\end{table}
\end{center}

\subsection{gEDMD spectral residuals}
\label{sec:gEDMD_spectral_pollution}
The spectra of the Koopman operator and generator are of interest because the dynamical system's behavior is determined by them. Computing the spectrum of an approximated operator is also subject to inaccuracies arising from a finite-dimensional approximation. Discretizations can lead to spurious eigenvalues that are not related to the underlying operator \citep{colbrook-2024a}. Thus we use an approach to identify the presence of spectral pollution based on an existing approach by \citet{colbrook-2024a} by adjusting their approach for the Koopman generator.

The suitability of an eigenvalue-eigenvector pair \(\lambda, \chi \) of the generator \( \mathcal{L} \) can be evaluated by the following relative residual
\begin{equation}
\begin{split}
    \text{res} &= \frac{\int_{\mathcal{X}}| \mathcal{L}\chi(\vx) - \lambda \chi(\vx) |^2 d\mu(\vx)}{\int_{\mathcal{X}}|\chi(\vx)|^2 d\mu(\vx)} \\
     &= \frac{\langle (\mathcal{L} - \lambda)\chi, (\mathcal{L} - \lambda)\chi \rangle}{\langle \chi, \chi \rangle}\\
     &=\frac{\sum_{i,j=1}^{N} \bar{{\chi}}_i {\chi}_j  \left[ \langle \mathcal{L} \phi_i, \mathcal{L} \phi_j \rangle  - \lambda \langle \phi_i, \mathcal{L} \phi_j \rangle - \bar{\lambda} \langle \mathcal{L} \phi_i, \phi_j \rangle +|\lambda|^2 \langle \phi_i,  \phi_j \rangle \right] }{ \sum_{i,j=1}^{N} \bar{{\chi}}_i {\chi}_j \langle \phi_i,  \phi_j \rangle}
     .
\end{split}
\end{equation}
In the context of finite data and a finite sized dictionary, we approximate this residual by
\begin{equation}
    \hat{\text{res}}=\frac{\sum_{i,j=1}^{N} \bar{\hat{\chi}}_i \hat{\chi}_j \left[ (\dot{\Phi}_X \dot{\Phi}_X^{*})_{i,j} - \lambda({\Phi}_X \dot{\Phi}_X^{*})_{i,j} - \bar{\lambda}(\dot{\Phi}_X {\Phi}_X^{*})_{i,j}  + | \lambda |^2({\Phi}_X {\Phi}_X^{*})_{i,j} \right]}{\sum_{i,j=1}^{N} \bar{\hat{\chi}}_i \hat{\chi}_j ({\Phi}_X {\Phi}_X^{*})_{i,j}}.
\end{equation}
The Galerkin matrices \(\hat{\mathbf{C}}_{NM}=
    \frac{1}{M}\dot{\Phi}_X\Phi_X^{*}\) and \(\hat{\mathbf{G}}_{NM}=\frac{1}{M}\Phi_X\Phi_X^{*}\) from \Cref{eq:appendix_gedmd_galerkin} already provide much of the necessary calculations, and in addition we compute \(\dot{\Phi}_X \dot{\Phi}_X^{*}\).

\begin{figure}[h]
    \centering
    \begin{subfigure}{0.33\textwidth}
        \includegraphics[width=\linewidth]{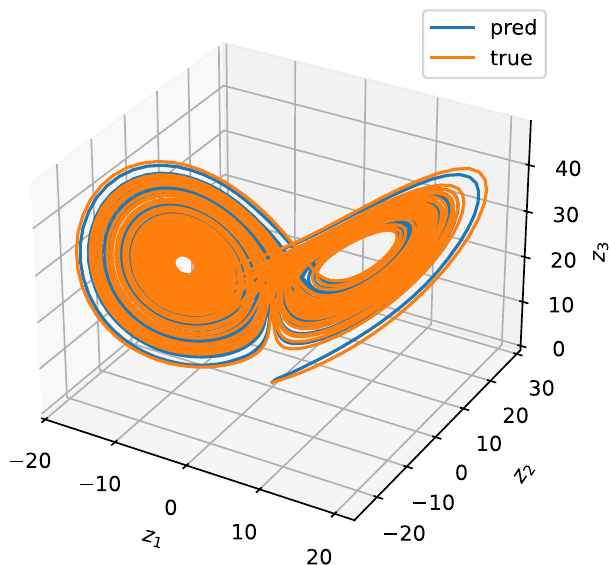}
        \caption{Phase space}
        \label{fig:DRAGONS_L63_long_phase}
    \end{subfigure}
    \hspace{7pt}
    \begin{subfigure}{0.63\textwidth}
        \includegraphics[width=\linewidth]{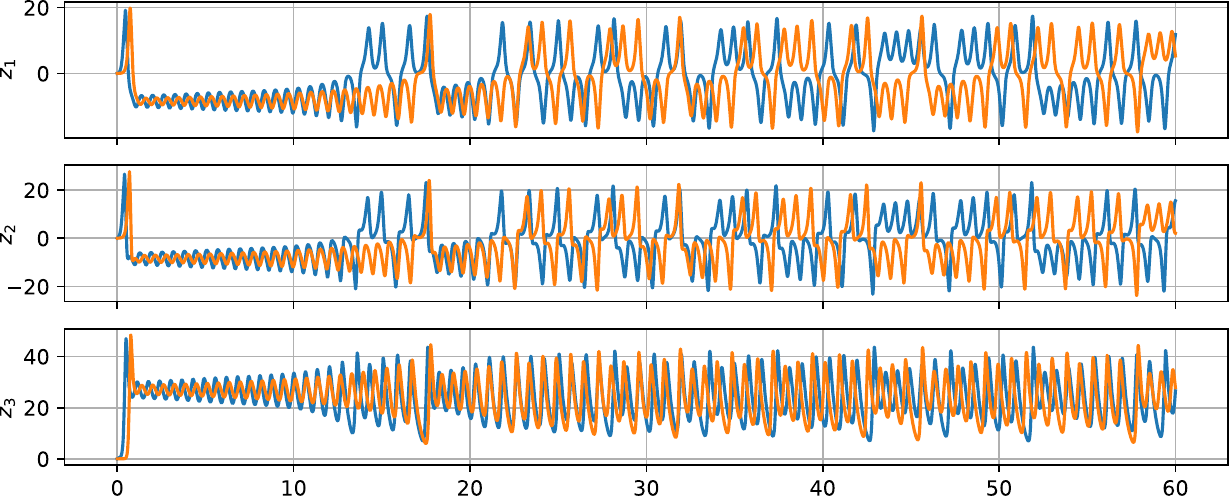}
        \caption{Trajectory}
        \label{fig:DRAGONS_L63_long_traj}
    \end{subfigure}
    \caption{Comparison of numerically calculated versus predicted trajectories for Lorenz experiment. Ticks on the x-axis show the evaluation points in time, which are irregularly spaced.}
    \label{fig:DRAGONS_L63_long}
\end{figure}

\subsection{Exploding gradient problem in adjoint sensitivity method}
\label{app:node_exploding_adjoint}

 \begin{figure}[h]
    \centering
    \begin{subfigure}{0.495\textwidth}
            \includegraphics[width=\linewidth]{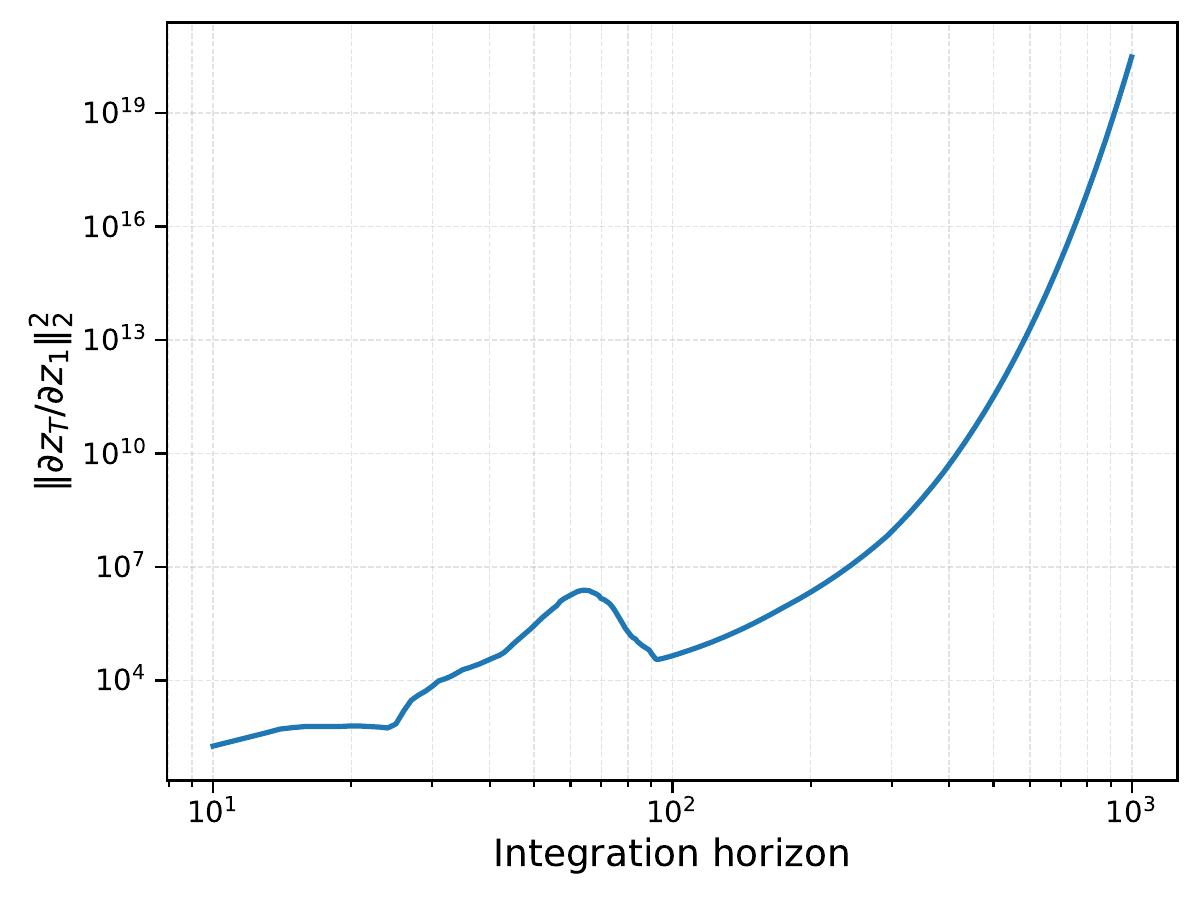}
            \caption{Flow-map Jacobian of the receding state w.r.t. the initial one.}
            \label{fig:lorenz_node_grad_explosion}
    \end{subfigure}
        \begin{subfigure}{0.495\textwidth}
            \includegraphics[width=\linewidth]{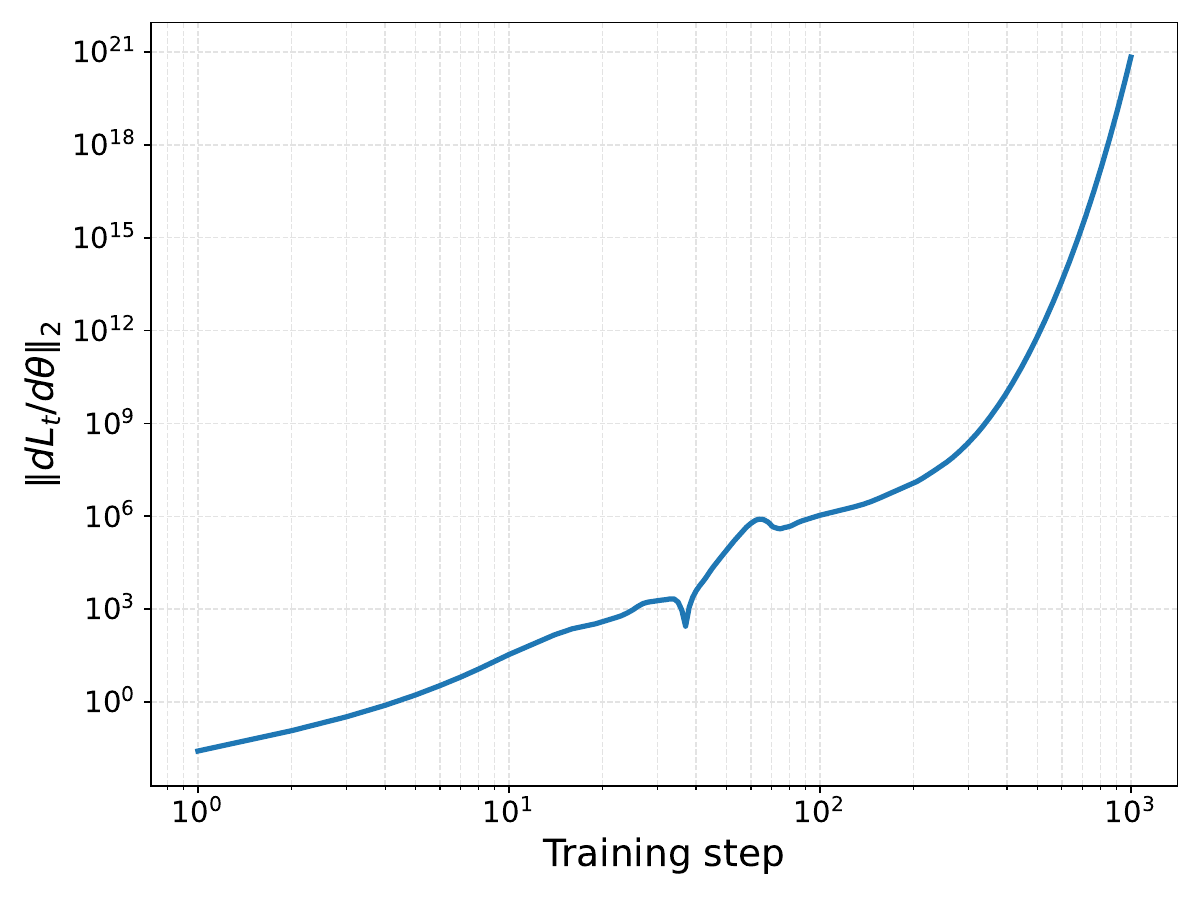}
            \caption{Gradient of the adjoint state w.r.t. the neural network parameters.}
            \label{fig:lorenz_node_adjoint_grad_explosion}
    \end{subfigure}
    \newline
    \begin{subfigure}{0.8\textwidth}
                    \includegraphics[width=\linewidth]{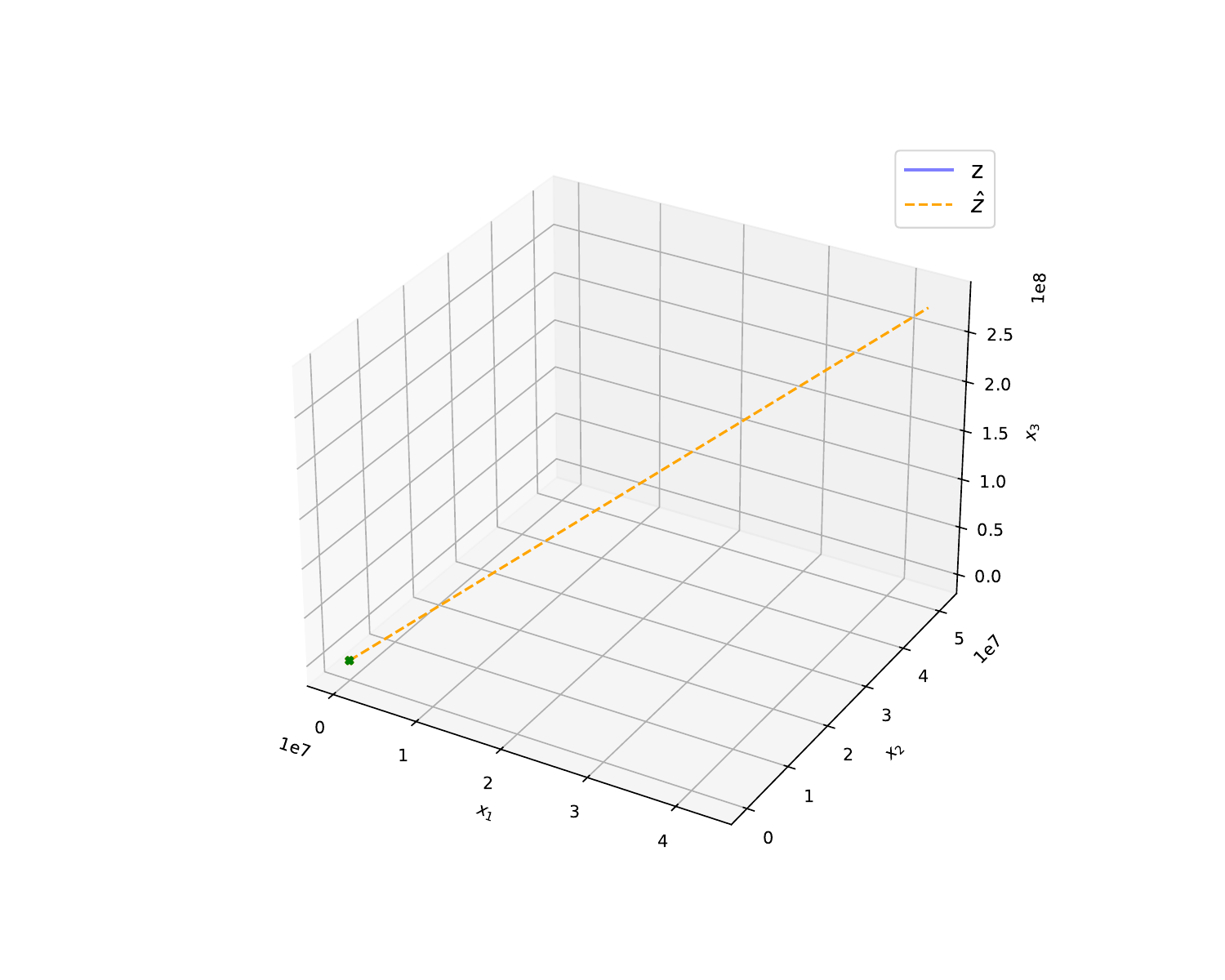}
            \caption{Phase space of neural ODE prediction for Lorenz-63.}
            \label{fig:lorenz_node_prediction}
    \end{subfigure}
    \caption{Explosion of the neural ODE's adjoint sensitivity when trained on the chaotic Lorenz-63 system as the integration horizon increases.}
    \label{fig:node_adjoint_explosion}
\end{figure}

In~\Cref{fig:lorenz_node_grad_explosion}, we empirically demonstrate the explosion of the flow-map Jacobian and adjoint state norms, discussed in~\Cref{sec:limitations_of_neural_ODE} and proved in~\Cref{sec:adjoing_ev_proof} for a neural ODE trained on the Lorenz-63 system. 
This experiment follows~\cite{mikhaeil2022difficulty}, who theoretically established and empirically demonstrated gradient explosion in RNNs trained on chaotic Lorenz dynamics. 
The gradient of the adjoint state with respect to the model parameters exhibits a similar exponential growth, as shown in~\Cref{fig:lorenz_node_adjoint_grad_explosion}, making long-horizon training particularly challenging for chaotic systems. This behavior is demonstrated using a Neural ODE trained on the Lorenz--63 system, with the implementation adapted from~\citet{allauzen2022experimental}.
Their work investigates the use of adaptive numerical solvers within Neural ODE training and shows that a naive application of adaptive time stepping is insufficient, motivating a tighter coupling between the solver and the training procedure. Our results nevertheless show that the observed gradient growth is intrinsic to the dynamics and does not depend on the particular ODE solver.
A forecasting with neural ODE of the Lorenz system becomes impossible due to explosion, as seen on~\Cref{fig:lorenz_node_prediction}. 
We perform forecasting on the same trajectory with our already trained DRAGONS model on Lorenz-63 from our forecasting experiment. The predictions are on the attractor, and are visualized in \Cref{fig:DRAGONS_L63_long}, and the prediction error is 7.45e-2.

\end{document}